\documentclass[twoside]{article}

\usepackage[accepted]{aistats2026}

\usepackage[utf8]{inputenc} 
\usepackage[T1]{fontenc}    
\usepackage{hyperref}       
\usepackage{url}            
\usepackage{booktabs}       
\usepackage{amsfonts}       
\usepackage{nicefrac}       
\usepackage{microtype}      
\usepackage{xcolor}         

\usepackage{multirow,paralist,mathrsfs,amsfonts,dsfont,tikz}

\usepackage{enumitem}

\usepackage{amsmath,amsthm,amssymb,multirow,paralist,mathrsfs,amsfonts,dsfont,tikz,bbm}
\usepackage{amssymb}
\usepackage{mathtools}
\usepackage{amsthm}

\usepackage{amsmath}
\usepackage{amssymb}
\usepackage{mathtools}
\usepackage{amsthm}

\usepackage{soul}
\usepackage{url}
\usepackage[utf8]{inputenc}
\usepackage[small]{caption}
\usepackage{graphicx}
\usepackage{amsmath}
\usepackage{amsthm}
\usepackage{booktabs}
\usepackage{algorithm}
\usepackage{algorithmic}
\usepackage{bbm}
\usepackage{comment}
\usepackage{multirow}
\usepackage{siunitx}
\usepackage{tabularx}
\usepackage{nicematrix}
\usepackage{subfigure}
\usepackage{footnote}

\usepackage{wrapfig}

\usepackage{multirow,paralist,mathrsfs,amsfonts,dsfont,tikz,stmaryrd,wrapfig}

\usepackage{enumitem}

\usepackage[round]{natbib}

\newtheorem{theorem}{Theorem}
\newtheorem{lemma}{Lemma}

\newtheorem{proposition}{Proposition}

\newtheorem{theoreminappendix}{Theorem (Re)}
\newtheorem{lemmainappendix}{Lemma (Re)}

\newtheorem{propositioninappendix}{Proposition (Re)}

\def\cl{\text{cl}}

\def\E{\mathbb{E}}

\def\P{\mathbb P}
\def\R{\mathbb R}

\def\indicator{\mathbbm{1}}

\def\calX{\mathcal X}
\def\calY{\mathcal Y}

\def\calD{\mathcal D}

\def\calP{\mathcal P}

\def\calB{\mathcal B}

\def\calF{\mathcal F}
\def\calL{\mathcal L}

\def\calT{\mathcal T}

\def\E{\mathbb E}
\def\P{\mathbb P}
\def\R{\mathbb R}

\def\min{\text{min}}
\def\sup{\text{sup}}
\def\max{\text{max}}
\def\tr{\text{tr}}

\def\test{\text{test}}

\def\cl{\text{cl}}

\def\noisy{\text{noisy}}
\def\crf{{\text{cr}}}
\def\temp{{\texttt{temp}}}

\begin{document}

%
\runningtitle{Conformal Margin Risk Minimization: An Envelope Framework for Robust Learning under Label Noise}

%

\twocolumn[

\aistatstitle{Conformal Margin Risk Minimization: \\
An Envelope Framework for Robust Learning under Label Noise}

\aistatsauthor{Yuanjie Shi$^{*}$ \And Peihong Li$^{*}$ \And Zijian Zhang \And Jana Doppa \And  Yan Yan }

\aistatsaddress{School of EECS, Washington State University, Pullman, WA, USA} ]

\let\thefootnote\relax\footnotetext{$^{*}$Equal contribution.}

\begin{abstract}
Most methods for learning with noisy labels require 
privileged knowledge such as noise transition matrices, clean subsets or pretrained feature extractors, resources typically unavailable when robustness is most needed.
We propose \emph{{\bf C}onformal {\bf M}argin {\bf R}isk {\bf M}inimization (CMRM)}, a plug-and-play envelope framework that improves \emph{any} classification loss under label noise by adding a single quantile-calibrated regularization term, with no privileged knowledge or training pipeline modification.
CMRM measures the confidence margin between the observed label and competing labels, and thresholds it with a conformal quantile estimated per batch to focus training on high-margin samples while suppressing likely mislabeled ones.
We derive a learning bound for CMRM under arbitrary label noise requiring only mild regularity of the margin distribution. 
Across five base methods and six benchmarks with synthetic and real-world noise, 
CMRM consistently improves accuracy (up to $+3.39\%$), reduces conformal prediction set size (up to $-20.44\%$) and does not hurt under 0\% noise, 
showing that CMRM captures a method-agnostic uncertainty signal that existing mechanisms did not exploit.
\end{abstract}

\section{Introduction}

Deep neural networks have achieved remarkable success across domains such as vision \citep{zhao2024review,zhang2024vision}, language \citep{naveed2025comprehensive,teubner2023welcome}, and healthcare \citep{thirunavukarasu2023large,celard2023survey}, but their performance typically depends on clean labels \citep{arpit2017closer}. 
However, in real-world scenarios, labels are often corrupted by human annotation errors (e.g., crowd-sourcing) or automated data collection pipelines 
\citep{song2022learning,gupta2019dealing}. 
Learning with label noise (LNL) formalizes this setting, where corrupted labels distort empirical risk and lead to overfitting \citep{natarajan2013learning,zhang2021understanding,arpit2017closer}. 
This issue is particularly severe in high-stakes domains such as medical imaging \citep{shi2024survey} and autonomous driving \citep{li2021evaluation}, making robustness to label noise an important problem in modern machine learning \citep{song2022learning,frenay2013classification}.

Learning under label noise has been extensively studied \citep{song2022learning}, with existing approaches broadly falling into the following categories.
Noise modeling approaches assume a specific corruption process (e.g., symmetric, class-conditional, or instance-dependent) and often require estimating or knowing the noise transition matrix \citep{natarajan2013learning,patrini2017making,li2021provably,cheng2020learning}, which is typically unobservable and difficult to model under complex noise \citep{yao2020dual}.
Loss correction methods adjust training objectives to counteract label corruption but typically depend on auxiliary information, such as small trusted clean subsets or accurate noise rate estimates \citep{hendrycks2018using,xia2019anchor}, both of which are often unavailable or unreliable in practice.
Auxiliary methods leverage additional sources, such as peer networks, semi-supervised learning (SSL) pipelines, or external expert models such as large language models (LLMs), which introduce strong data and supervision requirements \citep{han2018co,li2020dividemix,wang2024noisegpt}.
Table \ref{tab:previous_methods} summarizes representative methods and their key assumptions (see Section~\ref{sec:problem_setup} for details).

Despite this progress, these approaches share a common limitation: 
they all require some form of privileged knowledge about the noise process or access to auxiliary resources,
such as a noise transition matrix, a clean data subset, a peer network or a pretrained feature extractor,
which is typically unavailable in the settings where robustness to label noise is most needed,
especially under severe or heterogeneous noise \citep{arazo2019unsupervised,song2022learning}.
This raises the main research question of this paper: 
\emph{Can we design a flexible envelope framework that enhances the robustness of existing methods under arbitrary label noise, requiring only standard mathematical regularity conditions,
instead of any knowledge of the noise process or access to auxiliary supervision?}

\begin{table}[t]
\centering
\resizebox{\columnwidth}{!}{
\begin{tabular}{|c|c|}
\hline
\texttt{Methods} & \texttt{Assumptions} \\ \hline
LNL \citep{natarajan2013learning}
&
Symmetric noise
\\ \hline
Forward \citep{patrini2017making}
& 
Known label noise transition matrix
\\ \hline
GLC \citep{hendrycks2018using} 
&
Small clean data subsets
\\ \hline
Co-teaching \citep{han2018co} 
& 
Two peer nets
\\
\hline 
CORES \citep{cheng2020learning}
&
Instance-dependent noise
\\
\hline
VolMinNet \citep{li2021provably}
&
Class-conditional noise
\\
\hline 
FINE \citep{kim2021fine}
&
Eigen-structure noise
\\
\hline 
ELR+ \citep{liu2020early}
&
Clean samples learned first
\\
\hline 
NCFW \citep{zhang2024multiclass}
&
Known class posteriors
\\
\hline 
CSGN \citep{lin2024learning} 
&
Latent causal graph transition 
\\
\hline 
NI-ERM \citep{zhu2024label}
&
Strong pretrained feature extractor
\\
\hline 
NoiseGPT \citep{wang2024noisegpt}
&
External LLM
\\
\hline 
\textbf{CMRM (ours)} & Smooth CDF + positive density
\\
\hline 
\end{tabular}
}
\caption{
\textbf{Classical and recent methods for learning from noisy labels and their assumptions}. 
CMRM only assumes mild regularity conditions, a smooth CDF with positive density at the target quantile. 
Details in Section~\ref{sec:problem_setup}.
CMRM's assumptions are automatically satisfied in standard training (see Figure~\ref{fig:multi_justification}(d)).
}
\label{tab:previous_methods}
\end{table}

To answer this question, we propose \emph{{\bf C}onformal {\bf M}argin {\bf R}isk {\bf M}inimization (CMRM)},
an envelope framework to improve the robustness and accuracy of prior methods. 
CMRM does not rely on noise models, clean subsets or auxiliary supervision. 
CMRM achieves the goal through an uncertainty-aware training objective that integrates confidence margins and conformal quantiles. 
Confidence margin, defined as the gap between the confidence of the observed label and competing labels, provides a principled signal to distinguish reliable from uncertain training samples \citep{cui2019class,pleiss2020identifying}.
Conformal quantiles offer distribution-free, statistically valid thresholds for these margins \citep{angelopoulos2021gentle,lei2018distribution}. 

This algorithm design is motivated by two prior observations: 
(1) label noise creates a mismatch between corrupted labels and the evolving confidence structure of the model \citep{pleiss2020identifying}, and 
(2) ignoring this uncertainty often leads to overfitting on noisy samples and prevents algorithms from leveraging informative signals \citep{zhang2021understanding,arpit2017closer}.
Building on this motivation, CMRM formulates a conformal margin risk that directly incorporates uncertainty into the training objective.
Specifically, CMRM computes per-example confidence margins, estimates a conformal quantile to set an adaptive threshold, and minimizes a conformal risk defined as the average negative margin above that threshold.
Importantly, CMRM requires no architectural modification, no additional networks and no clean validation data,
as it adds only a single quantile-calibrated regularization term to any existing training objective.
We further establish a learning bound for CMRM under arbitrary noise. 
Our extensive experiments on both binary and multi-class classification benchmarks demonstrate that,
without requiring prior knowledge of noise,
CMRM consistently improves methods with fundamentally different LNL design principles, 
indicating that conformal margin calibration captures a method-agnostic uncertainty signal that existing mechanisms fail to fully exploit.

{\bf Contributions.} Our key contributions include:
\begin{itemize}
    \item 
    We propose CMRM, an envelope framework whose robustness requires only mild regularity of the margin distribution that is automatically satisfied in standard training,
    instead of assumptions on noise models, clean data or auxiliary resources (see Table~\ref{tab:previous_methods}).
    
    \item 
    We derive a learning bound under arbitrary label noise that decomposes the gap between the noisy surrogate and clean conditional margin risk into function-class complexity, quantile estimation error, and distribution shift.

    \item 
    We demonstrate that CMRM, requiring only a single regularization term, improves accuracy across all base method and dataset combinations tested and reduces prediction set size in nearly all cases, across five base methods (CE, Focal, LDAM, GCE, NI-ERM) and six benchmarks with synthetic and real-world noise, and incurs no accuracy penalty when labels are clean. 
    The CMRM code is available at \url{https://github.com/YuanjieSh/CMRM}.
\end{itemize}

\section{Related Work}


{\bf Learning with Label Noise (LNL)} studies how to train models when labels are corrupted by annotation errors, heuristic labeling, or unreliable sources \citep{frenay2013classification,song2022learning,johnson2022survey}.
Many approaches have been proposed, often under specific assumptions on the noise mechanism or requiring auxiliary supervision (see Table~\ref{tab:previous_methods} and Section~\ref{sec:problem_setup}).
These include loss or posterior correction with known noise models \citep{natarajan2013learning,patrini2017making,li2021provably}, methods using clean side information or peer networks \citep{hendrycks2018using,han2018co,li2020dividemix,liu2020early}, approaches for instance-dependent or structured noise \citep{cheng2020learning,kim2021fine,lin2024learning}, and methods leveraging external knowledge such as pretrained features or LLMs \citep{zhang2024multiclass,zhu2024label,wang2024noisegpt}.
In contrast, our envelope approach requires only mild regularity conditions and is compatible with most prior methods.

{\bf Conformal Prediction (CP)} is a distribution-free uncertainty quantification framework that constructs prediction sets with guaranteed finite-sample coverage \citep{vovk2005algorithmic,lei2018distribution,angelopoulos2021gentle,fontana2023conformal,NCP,PRCP,R3CP,LLM-CP}.
Recent conformal training methods incorporate these principles into the training process, focusing on minimizing prediction set size \citep{stutz2021learning,shi2025direct} or refining coverage calibration \citep{einbinder2022training,kiyani2024length}, but assume clean training data.
In contrast, CMRM leverages conformal principles to enhance discriminative ability under noisy supervision.
Importantly, CMRM is not a conformal prediction method in the standard sense; it does not construct prediction sets or provide coverage guarantees at inference time.

\section{Problem Setup and Motivation}
\label{sec:problem_setup}

\textbf{Notations.}
Suppose $X \in \calX$ is an input from $\mathcal{X}$, and $Y \in \calY = \{ 0,1,\cdots, K-1\}$ is the ground-truth label, where $K$ is the number of candidate classes. 
Let the underlying data distribution be $\calP(X,Y)$, which characterizes the relationship between inputs and class labels. 
We denote by $\calP(X)$ the marginal distribution of inputs, and by $\calP(Y|X)$ the conditional label distribution given inputs.  
Let $f: \calX \rightarrow \R^K$ denote the logit vector produced by soft classifier $f \in \calF$, where $\calF$ denotes a hypothesis class.
We also define $P_f(X) = \sigma \circ f(X): \calX \rightarrow \Delta_+^K$ as the corresponding confidence score, where $\Delta_+^K$ is the (K-1)-dimensional probability simplex, 
and $\sigma$ is the Sigmoid function in binary and Softmax function in multi-class setting, respectively.
Let $P_f(X)_y$ denote the confidence score of class $y$.
Define $\indicator[\cdot]$ as an indicator function.
Denote $\calD_\tr$ and $\calD_\test$ as the training and test sets.
Let $\calB$ be a randomly sampled batch of training data of size $s$.
We denote by $W_1(\mathcal{P}, \mathcal{Q}) 
= \inf_{\pi \in \Pi(\mathcal{P}, \mathcal{Q})} 
\int_{\mathcal{X} \times \mathcal{X}} d(x,x')\, \mathrm{d}\pi(x,x')$ as the Wasserstein-1 distance between two distributions $\mathcal{P}$ and $\mathcal{Q}$ on metric space $(\mathcal{X}, d)$, where $\Pi(\mathcal{P}, \mathcal{Q})$ is the set of all joint distributions of $\mathcal{P}$ and $\mathcal{Q}$.

\textbf{Learning with Noisy Labels (LNL).}
In many real-world scenarios, the training data are corrupted by noisy labels.
Instead of the clean label $Y$, we only observe a potentially corrupted label $\widetilde{Y} \in \mathcal{Y}$ generated from $Y$ through an unknown label noise transition matrix $T(\widetilde{Y}\mid X,Y)$ \citep{zhu2024label}.
For example, symmetric noise corresponds to $T(\widetilde{Y}\mid X,Y)$ being uniform over all incorrect labels, while class-conditional noise assumes dependence only on $Y$. 
We highlight that {\it our proposed CMRM framework does not rely on such assumptions and permits $T$ to be arbitrary.}

Such noisy labels setting arises typically due to annotator mistakes, inconsistent labeling criteria, or spurious labels introduced by large-scale data collection pipelines \citep{natarajan2013learning,patrini2017making}.
Formally, we define the noisy training data distribution as:
\begin{align*}
\calP_{\noisy}(X,\widetilde{Y}) \;=\; \sum_{y \in \calY} \calP(X,y)\,T(\widetilde{Y}\mid X,y).
\end{align*}
Accordingly, the training set $\mathcal{D}_{\tr} = \{(X_i,\widetilde{Y}_i)\}_{i=1}^n$ of size $n$ is drawn from $\mathcal{P}_{\noisy}$, while the clean test set $\mathcal{D}_{\test}$ is drawn from $\mathcal{P}$.


\textbf{Limitations of Existing Approaches}.
Despite substantial progress on learning from noisy labels, most existing methods remain fundamentally constrained by strong assumptions. As summarized in Table~\ref{tab:previous_methods}, these assumptions fall into three major categories:

First, many approaches impose explicit noise-model assumptions, from symmetric noise or known transition matrices~\citep{natarajan2013learning,patrini2017making} to class-conditional~\citep{li2021provably}, instance-dependent~\citep{cheng2020learning}, or eigen-structure noise~\citep{kim2021fine}, but these still rarely align with real-world corruption.

Second, a range of methods rely on auxiliary information, such as clean subsets~\citep{hendrycks2018using}, noise transition matrices \citep{patrini2017making}, strong pretrained feature extractors~\citep{zhu2024label}, or external LLMs~\citep{wang2024noisegpt},
which are costly and unavailable at scale.

Third, some approaches exploit architectural assumptions or model-specific heuristics,
such as peer networks~\citep{han2018co,li2020dividemix},
early-learning assumptions~\citep{liu2020early},
or latent causal structures~\citep{lin2024learning}. 

Overall, these assumptions limit existing methods, especially under severe or heterogeneous noise~\citep{arazo2019unsupervised,song2022learning}.
Therefore, an envelope framework that provides robustness guarantees under only mild regularity conditions is needed.

\section{CMRM Framework}
\label{sec:proposed_method}

In this section, we first describe the general {\it Conformal Margin Risk Minimization (CMRM)} envelope framework and its variant for binary classification. 
Next, we develop a practical optimization algorithm for CMRM and theoretically analyze its learning bound. 

\subsection{General Framework}
\label{sec:framework}

\textbf{Confidence Margin.} 
The confidence margin quantifies the separation between the confidence assigned to the observed label and the highest confidence among other candidate labels:  
\begin{align}
\label{eq:framework_margin}
M_f(X_i,\widetilde Y_i) 
= P_f(X_i)_{\widetilde Y_i} - \max_{y \in \calY \setminus \{\widetilde Y_i\}} P_f(X_i)_y .
\end{align}
This notion of confidence margin has been widely used in the machine learning literature \citep{cui2019class,bartlett2002rademacher}.
Large margins indicate a strong preference for the observed label, whereas small or negative margins reflect uncertainty \citep{cui2019class}. 
Margins are widely used to characterize decision boundary distance~\citep{bartlett2002rademacher,neyshabur2017exploring} and to measure predictive uncertainty~\citep{elsayed2018large,jiang2018predictive,lakshminarayanan2017simple}. 

\textbf{Conformal Quantile.}
Conformal prediction (CP)~\citep{vovk2005algorithmic,angelopoulos2021gentle} provides a distribution-free mechanism for defining quantile thresholds. 
Let $V$ be a real-valued random variable, and let $\{V_i\}_{i=1}^m$ denote an i.i.d.\ sample from $V$ with size $m$. 
Given a target level $\alpha \in (0,1)$, the empirical conformal quantile is defined as $\widehat \tau_{\alpha} = 
Q \big (\alpha, \{V_i\}^m_{i=1} \big)$,
where $Q$ selects the $\lceil \alpha(m+1) \rceil$-th largest value in $\{V_i\}^m_{i=1}$.
Hence, at most an $\alpha$-fraction of the sample values lie below $\widehat \tau_{\alpha}$, i.e., the empirical $\alpha$-quantile.
Through CP, this quantile enjoys distribution-free validity guarantee, making it a principled tool for UQ.

\textbf{Conformal Margin Risk.}
While conformal quantiles provide principled thresholds, their effectiveness depends on using a score that reflects label reliability under noise.
Confidence margin serves this role: clean samples typically exhibit large margins, whereas noisy labels tend to have small margins~\citep{zhang2018generalized,liu2020early,zhang2021understanding}.
This is intuitive: correct labels align with dominant class evidence, while corrupted labels conflict with the input features and are outscored by alternative labels.
Figure~\ref{fig:margin_distribution} illustrates this separation on CIFAR-100 under different types of noise (class-conditional and human annotation noise).
However, such observations have mainly been used for diagnostic analysis or heuristic filtering, without formal risk formulations. 

CMRM closes this gap by calibrating confidence margins with conformal quantiles to yield an assumption-light and uncertainty-aware risk.
Specifically, we instantiate $V$ with confidence margin $M_f(X, \widetilde Y)$, and define the conformal quantile threshold on $\{ M_f(X_i,\widetilde Y_i) \}_{i=1}^n$:
\begin{align}
\label{eq:framework_quantile}
\widehat \tau_{\alpha}(f) 
= 
Q \Big (\alpha, \{ M_f(X_i,\widetilde Y_i) \}^n_{i=1} \Big),
\end{align}
so that at most an $\alpha$-fraction of samples fall below this threshold, a property that always holds without assumptions on the noise process.
$\alpha$ is a hyper-parameter tuned on validation set. 
Larger $\alpha$ filters more aggressively, improving noise robustness but risking loss of clean samples, while smaller $\alpha$ retains more signal but with corruption. $\alpha$ is a hyper-parameter selected through validation data (see sensitivity study in Section \ref{subsection:multiclass_classification_experiments}).

Next, we define the conformal margin risk $\widehat \calL_{\text{cr}}(f)$ as:
\begin{align}
\label{eq:CMRM}
&
\widehat \calL_{\text{cr}}(f)
= \frac{1}{n}\sum^n_{i=1} \widehat\ell_{\text{cr}}(f, X_i,\widetilde Y_i), \text{ such that}  
\\
&
\widehat\ell_{\text{cr}}(f, X_i,\widetilde Y_i) = 
- M_f(X_i,\widetilde Y_i) \cdot \widetilde \indicator[M_f(X_i,\widetilde Y_i) \ge \widehat \tau_\alpha(f)]
,
\nonumber
\end{align}
where $\widehat\ell_{\text{cr}}(f, X_i,\widetilde Y_i)$ is the per-sample conformal margin loss of $(X_i,\widetilde Y_i)$, and $\widetilde \indicator[x\geq y] = 1/(1+\exp(-(x-y)/\texttt{temp} ))$ is the smoothed indicator by Sigmoid function with temperature parameter $\texttt{temp}$. 
This formulation assigns a soft weight in $(0,1)$ to each sample rather than hard $0/1$ filtering, smoothly downweighting low-margin samples during training, similar to soft reweighting strategies in \cite{tjandra2023leveraging}.

\begin{figure}[t]
    \centering
    \begin{minipage}[t]{0.48\linewidth}
    \centering
    \textbf{(a)} Class-cond. noise
    \includegraphics[width=\linewidth]{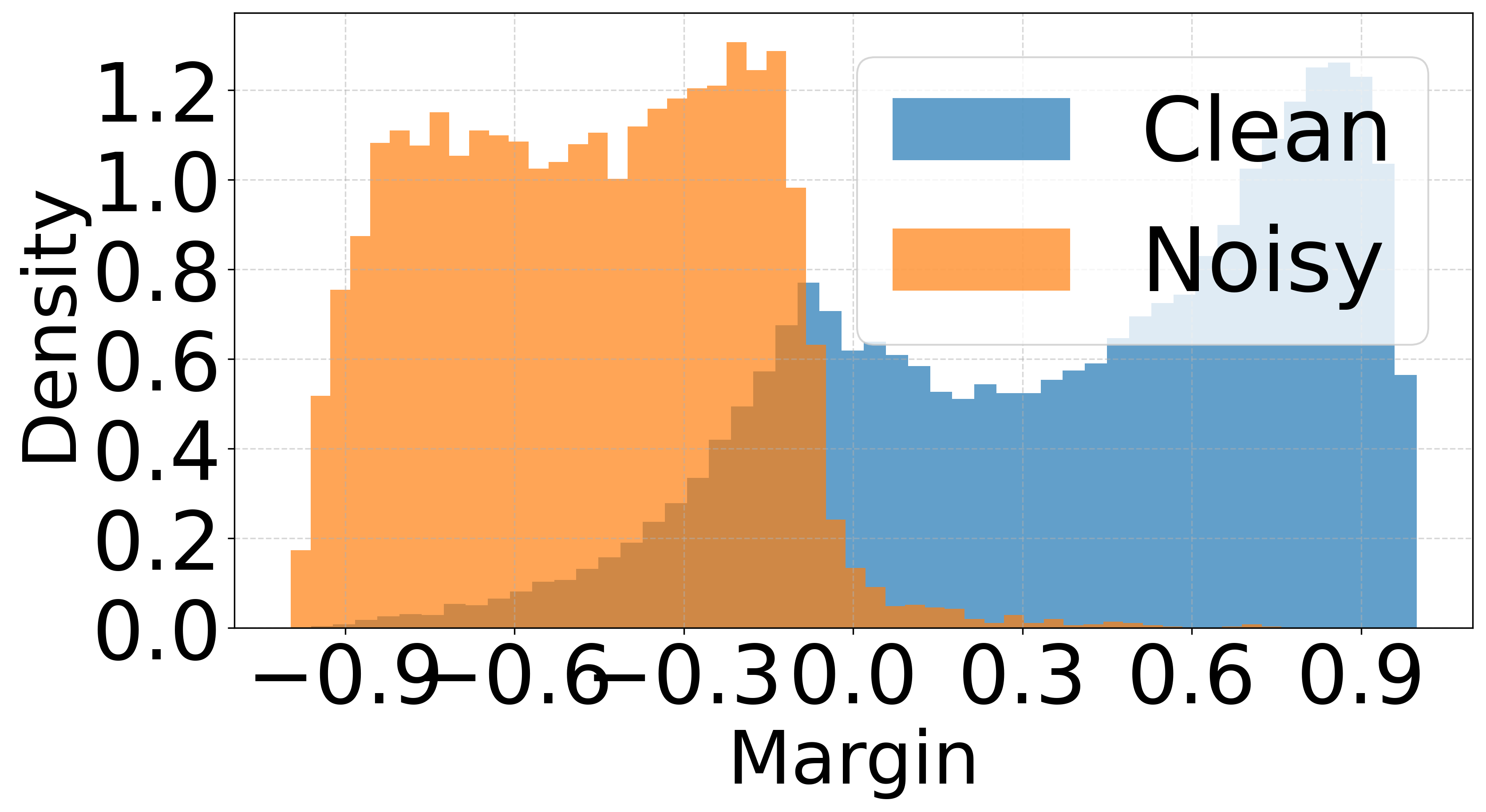}
    \end{minipage}
    \begin{minipage}[t]{0.48\linewidth}
    \centering
    \textbf{(b)} Human annota. noise
    \includegraphics[width=\linewidth]{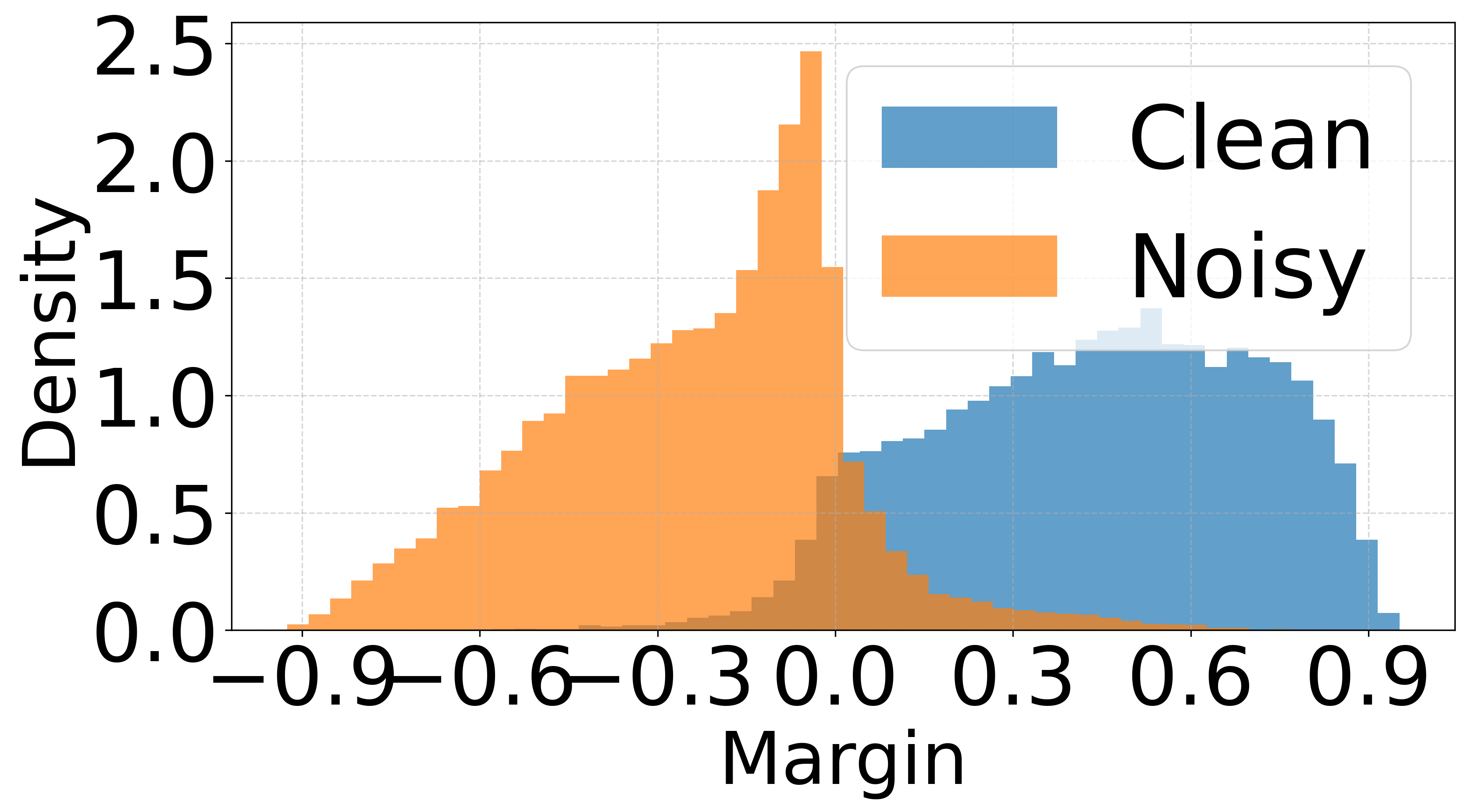}
    \end{minipage}
    \caption{
    \textbf{Confidence margin distributions for clean (blue) and noisy (orange) samples on CIFAR-100.} 
    \textbf{(a)} Class-conditional noise at $20\%$ and \textbf{(b)} human annotation noise at $40\%$. 
    In both cases, clean samples concentrate on positive margins while noisy samples shift to negative, showing that confidence margins can distinguish clean from noisy labels without assumptions on the noise process.
    }
    \label{fig:margin_distribution}
\end{figure}

By construction, $\widehat \calL_{\text{cr}}(f)$ is the average negative margin over the selected samples, i.e., those with margins above the threshold.
Minimizing $\widehat \calL_{\text{cr}}(f)$ increases the separation on selected high-confidence samples while discarding low-confidence samples. 
Thus, this principle aligns training with predictive uncertainty and filters out potentially corrupted labels (empirically verified in Figure~\ref{fig:multi_justification}(b)), without assumptions on the noise model or auxiliary supervision. 

Finally, we incorporate CMRM as a plug-and-play regularizer for any classification loss $\widehat \calL_{\cl}(f)$ (e.g., cross-entropy), including prior LNL loss functions, as demonstrated in our empirical evaluation:
\begin{align}
\label{eq:framework_total}
\widehat \calL(f)
= \widehat \calL_{\cl}(f) + \lambda \cdot \widehat \calL_{\text{cr}}(f),
\end{align}
where $\lambda \geq 0$ controls the strength of the CMRM term relative to $\widehat \calL_{\cl}(f)$. The above formulation is general for multi-class classification problems with standard accuracy metrics. However, binary classification tasks require optimizing more specific metrics. Hence, we provide an instantiation for binary classification below.

\textbf{CMRM for Binary Classification.}
Binary classification often relies on class-conditional performance measures such as false positive and false negative rates (FPR and FNR), which allow users to impose class-conditional tolerances \citep{zhou2006training}.
Since CMRM's quantile-based formulation naturally extends to class-conditional settings, we introduce a binary variant that adapts its margin thresholding to each class separately, enabling asymmetric error control while preserving its core structure.

First, we define class-conditional quantile thresholds for tolerances $\alpha^+$ and $\alpha^-$:
\begin{align*}
\widehat \tau^{-}(f) 
& \! = \! \min\Bigl\{ t \! : \! \,
    \! \frac{1}{n_0} \! \sum_{i: \widetilde Y_i \! = \! 0} 
    \!\indicator \!\bigl[P_f(X_i)_1 \! \ge \! t \bigr] 
    \! \le \! \frac{\lceil \alpha^{-}(n_0 \! + \!1) \rceil}{n_0} \Bigr\}, 
\\
\widehat \tau^{+}(f) 
& \! = \! \max\Bigl\{ t \! : \!\,
    \! \frac{1}{n_1} \! \sum_{i: \widetilde Y_i \! = \! 1} 
    \! \indicator \! \bigl[P_f(X_i)_1 \! \le \! t \bigr] 
    \! \le \! \frac{\lceil \alpha^{+}(n_1 \! + \! 1) \rceil}{n_1} \Bigr\},
\end{align*}
where $n_0$ and $n_1$ are the numbers of observed negative and positive samples, respectively.
These thresholds control the upper tail of the negative class (potential false positives) and the lower tail of the positive class (potential false negatives).

Next, we define a two-sided hinge formulation relative to these thresholds:
\begin{align}
\label{eq:binary-CMRM-empirical}
\widehat \calL^{\text{bin}}_{\text{cr}}(f)
= &
\frac{1}{n}\sum^n_{i=1} \widehat\ell^{\text{bin}}_{\text{cr}}(f, X_i,\widetilde Y_i), \text{ such that} 
\\
\widehat\ell^{\text{bin}}_{\text{cr}}(f, X_i,\widetilde Y_i) 
= &
- \lambda^{-} \widetilde \indicator[\widetilde Y_i=0] \cdot \bigl(P_f(X_i)_1 - \widehat\tau^{-}(f) \bigr)^{+}
\nonumber \\
& - \lambda^{+} \widetilde \indicator[\widetilde Y_i=1] \cdot 
\bigl(\widehat\tau^{+}(f) - P_f(X_i)_1\bigr)^{+},
\nonumber
\end{align}
where $(z)^+ = \max\{0,z\}$ and $\lambda^+,\lambda^- > 0$ control the relative strength of the two penalties.

This binary classification variant preserves the two core ingredients of CMRM: conformal quantile-based filtering and margin maximization on retained samples. 
By separately controlling the class-conditional distribution, it aligns naturally with standard binary performance metrics such as FPR and FNR. 

\subsection{Optimization and Theoretical Analysis}
\label{subsec:learning_theory}

\noindent {\bf Optimization Algorithm.} A key challenge in CMRM is that its objective depends on set-level quantiles of confidence margins over the full training data distribution, requiring $O(n\log n)$ sorting per iteration and thus prohibitive for large-scale training.  
We address this by replacing set-level quantiles with batch-level quantiles $\widehat \tau^s_{\alpha} (f)$: each iteration estimates the quantile from a batch of size $s$, reducing the cost to $O(s\log s)$. 
This yields a tractable surrogate objective $\widehat \calL^s_{\text{cr}}(f)$ by using $\widehat \tau^s_{\alpha} (f)$:
\begin{align}
\label{eq:batch_objective}
\widehat \calL^s_{\text{cr}}(f)
= \mathop{\E}\limits_{\widehat \tau^s_{\alpha} (f) \sim \calT_{\alpha}(f)}\bigg[ \sum^n_{i=1} &  - M_f(X_i,\widetilde Y_i) \cdot 
\\
& \widetilde \indicator[M_f(X_i,\widetilde Y_i) \ge \widehat \tau^s_\alpha(f)] \bigg],
\nonumber 
\end{align}
where $\calT_{\alpha}(f)$ is the underlying distribution of $\widehat \tau^s_{\alpha} (f)$.
Although this introduces an approximation gap relative to population quantiles, our theoretical analysis shows that the gap is bounded.

Algorithm~\ref{alg:CMRM} summarizes the CMRM training procedure.
The algorithm follows a standard stochastic optimization loop. 
For each iteration $t$, we first compute the confidence margins on batch $\calB_t$ (Line~\ref{alg:line:margin}) and then estimate the batch-wise conformal quantile threshold (Line~\ref{alg:line:threshold}). 
Next, we compute both the classification loss (Line~\ref{alg:line:cl_loss}) and the conformal margin risk (Line~\ref{alg:line:CMRM_loss}). 
The final update (Line~\ref{alg:line:update_f}) jointly minimizes both terms. 
Importantly, CMRM requires no change to model architecture or optimization, introducing only a quantile-calibrated regularizer that is compatible with arbitrary classifiers and loss functions.

\begin{algorithm}[t]
\caption{Conformal Margin Risk Minimization}
\label{alg:CMRM}
\begin{algorithmic}[1]

\STATE \textbf{Input}: training dataset $\calD_\tr$, regularization parameter $\lambda$, batch size $s$, learning-rate $\eta >0$, exclusion rate $\alpha$

\STATE Randomly initialize the deep neural network $f_0$ 

    \FOR{$t \leftarrow 0 : T-1$}

        \STATE Randomly sample batch $\calB_t \subset \calD_\tr$
        
        \STATE Compute 
        $M_{f_t}(X_i,\widetilde Y_i)$ on $\calB_t$
        \label{alg:line:margin}

        \STATE Compute batch-wise quantile $\widehat \tau^s_{\alpha}(f_t)$ on $\calB_t$
        \label{alg:line:threshold}

        \STATE Compute classification loss $\widehat \calL_{\cl}(f_t)$ on $\calB_t$
        \label{alg:line:cl_loss}

        \STATE Compute conformal margin risk $\widehat \calL_{\text{cr}}(f_t)$ on $\calB_t$
        \label{alg:line:CMRM_loss}

        \STATE $f_{t+1} 
        \leftarrow 
        f_t - \eta \nabla_f \big(\widehat \calL_{\cl}(f_t) + \lambda \widehat \calL_{\text{cr}}(f_t) \big)$
        \label{alg:line:update_f}
        
    \ENDFOR
    \label{alg:line:reinitialize}
\STATE \textbf{Output}: the trained model $f_T$
\end{algorithmic}
\end{algorithm}

\noindent {\bf Theoretical Analysis.} Building on the optimization procedure, we analyze the effect of using batch-level estimates and the corresponding learning bound. 
Proposition~\ref{proposition:quantile_gap} shows that the batch-level quantile in $\widehat \calL^s_{\text{cr}}(f)$ concentrates around the population quantile at rate $\tilde O(1/\sqrt{s})$ where $s$ is the batch size. 
Theorem~\ref{theorem:learning_bound} further bounds the gap between empirical margin risk on noisy data and population risk on clean data. 
Together, these results characterize the robustness of CMRM under arbitrary label noise with minimal assumptions.

We start with the definition of population quantile $\tau_{\alpha}(f)$ on the noisy data distribution:
\begin{align*}
\tau_{\alpha}(f) = \min \Big \{t: \P_{(X, \widetilde Y) \sim \calP_{\noisy}} [M_f(X,\widetilde Y) \leq t \big ] \geq \alpha \Big \}.
\end{align*}
Next, we analyze the gap between the batch-level quantile and population one.

\begin{proposition}[Gap between $\tau_\alpha$ and $\widehat \tau_\alpha^s$] 
\label{proposition:quantile_gap}
Denote by $G(t)$ the cumulative distribution function (CDF) of $M_f(X,\widetilde Y)$ under the noisy distribution $\mathcal{P}_{\noisy}$.
Assume that $G(t)$ is continuously differentiable in a neighborhood of $\tau_\alpha(f)$ with density $g(t)$ and $g(\tau_\alpha(f)) > 0$. 
Then, for any $\delta \in (0,1)$, we have:
\begin{align*}
\P \bigg ( |\tau_{\alpha}(f) - \widehat\tau^s_{\alpha}(f) | \leq \tilde O \Big (\frac{1}{\sqrt{s}} \Big ) \bigg ) \geq 1-\delta,
\end{align*}
where $\tilde O$ hides the logarithmic factors.
\end{proposition}

{\bf Remark 1.} Proposition~\ref{proposition:quantile_gap} shows that the batch-level quantile $\widehat{\tau}^s_\alpha(f)$ closely approximates the population quantile $\tau_\alpha(f)$, with an error of order $\tilde O(1/\sqrt{s})$, which quantifies the statistical accuracy of the threshold estimation step in CMRM. 
This relies on a mild regularity assumption on the CDF (smoothness and positive density), standard in quantile estimation theory (e.g.,~\cite{serfling2009approximation,van2000asymptotic}). 
We empirically verify these assumptions in Figure~\ref{fig:multi_justification}(c).

Next, we analyze the learning bound.
Instead of evaluating the unconditional classification risk, CMRM focuses on the retained high-margin region where the model is confident.
This reflects the mechanism of the method: low-margin samples—more likely to be corrupted under label noise—are suppressed during training, while high-margin samples dominate the learning signal.
Accordingly, the appropriate clean-distribution objective is the conditional margin risk restricted to this retained region, analogous in spirit to Conditional Value at Risk (CVaR)-style tail-risk analyses in the robustness literature.

To this end, 
we first define the conformal margin risk on the clean distribution as:
\begin{align*}
\mathcal{L}_{\text{cr}}(f)
= - \E_{(X,Y) \sim \calP} \left[ M_f(X,Y) \mid M_f(X,Y) \ge \tau_\alpha(f) \right].
\end{align*}

We define its empirical Rademacher complexity as:
\[
\widehat{\mathfrak{R}}_n(\mathcal{F})
:= \mathbb{E}_{\sigma}\left[
\sup_{f \in \mathcal{F}} 
\frac{1}{n}\sum_{i=1}^n \sigma_i \, f(X_i,\widetilde Y_i)
\right],
\]
where $\{\sigma_i\}_{i=1}^n$ are Rademacher variables.
The following theorem analyzes the learning bound of CMRM.

\begin{theorem}[Learning bound]
\label{theorem:learning_bound}
Suppose the assumptions in Proposition~\ref{proposition:quantile_gap} hold.
Define $\delta_w$ as the average Wasserstein-1 distance between the noisy and clean label distributions conditional on $X$, where $\delta_w = \mathbb{E}_{X \sim \mathcal{P}(X)} \Big[ W_1\big( \mathcal{P}_{\noisy}(\cdot \mid X),\; \mathcal{P}(\cdot \mid X) \big) \Big]$.
Then the learning bound of CMRM is:
\begin{align*}
\mathcal{L}_{\text{cr}}(f) - \widehat \calL^s_{\text{cr}}(f)
\leq
\tilde O \Big ( \widehat {\mathfrak{R}}_n(\mathcal{F}) + \frac{1}{\sqrt{s}} + \delta_w + \alpha + \texttt{temp} \Big ).
\end{align*}
\end{theorem}
{\bf Remark 2.} Theorem~\ref{theorem:learning_bound} characterizes how the surrogate objective optimized by CMRM on noisy data relates to the clean retained-region margin risk, and clearly decomposes the generalization gap of CMRM into three terms reflecting its core components: quantile estimation error $\tilde O(1/\sqrt{s})$, function class complexity $\widehat {\mathfrak{R}}_n(\mathcal{F})$, and distribution shift $\delta_w$. 
These correspond to the use of batchwise quantiles, empirical-to-population approximation on noisy data, and the mismatch between noisy and clean label distributions, respectively. 
Together, they provide an assumption-light characterization of CMRM’s robustness under arbitrary label noise.

\section{Experiments and Results}
\label{sec:experiment_results}

\begin{table*}[!t]
  \centering
  \resizebox{\textwidth}{!}{%
  \begin{NiceTabular}{@{}ll*{8}{c}@{}}
    \toprule
    \multirow{2}{*}{\textbf{Dataset}} & \multirow{2}{*}{\textbf{Metric}} &
    \multicolumn{2}{c}{\textbf{CE}} &
    \multicolumn{2}{c}{\textbf{Focal}} &
    \multicolumn{2}{c}{\textbf{LDAM}} &
    \multicolumn{2}{c}{\textbf{GCE}} \\
    \cmidrule(lr){3-4}\cmidrule(lr){5-6}\cmidrule(lr){7-8}\cmidrule(lr){9-10}
    \multicolumn{2}{c}{} & Base & +CMRM & Base & +CMRM & Base & +CMRM & Base & +CMRM \\
    \midrule
    \multirow{2}{*}{CIFAR-100}
      & ACC (\%) $\uparrow$ & 65.16 & \textbf{66.32 ($+1.16$)} & 64.42 & \textbf{65.39 ($+0.97$)} & 59.63 & \textbf{61.12 ($+1.49$)} & 62.17 & \textbf{63.65 ($+1.48$)} \\
      & M.APSS $\downarrow$ & 6.67 & \textbf{6.52 ($-2.25\%$)} & 6.89 & \textbf{6.61 ($-4.06\%$)} & 17.85 & \textbf{17.67 ($-0.78\%$)} & 7.70 & \textbf{6.28 ($-18.44\%$)} \\
    \midrule
    \multirow{2}{*}{mini-ImageNet}
      & ACC (\%) $\uparrow$ & 57.42 & \textbf{59.41 ($+1.99$)} & 55.54 & \textbf{57.93 ($+2.39$)} & 56.60 & \textbf{56.62 ($+0.02$)} & 55.16 & \textbf{55.51 ($+0.35$)} \\
      & M.APSS $\downarrow$ & 7.40  & \textbf{7.04 ($-4.86\%$)} & 7.67 & \textbf{7.28 ($-5.08\%$)} & 13.07 & \textbf{12.88 ($-1.45\%$)} & 9.90 & \textbf{9.59 ($-3.13\%$)} \\
    \midrule
    \multirow{2}{*}{Food-101}
      & ACC (\%) $\uparrow$ & 56.21 & \textbf{58.48 ($+2.27$)} & 56.35 & \textbf{58.92 ($+2.57$)} & 55.49 & \textbf{56.42 ($+0.93$)} & 55.66 & \textbf{59.05 ($+3.39$)} \\
      & M.APSS $\downarrow$ & 7.93 & \textbf{6.96 ($-12.23\%$)} & 7.76  & \textbf{6.89 ($-11.21\%$)} & 11.94 & \textbf{11.53 ($-3.43\%$)} & 8.41 & \textbf{6.93 ($-20.44\%$)} \\
    \bottomrule
  \end{NiceTabular}%
  }
  \caption{\textbf{Top-1 accuracy (\%) and marginal average prediction set size (M.APSS $\downarrow$) on multi-class datasets corrupted by synthetic noise with noise rate $20\%$.}
  Each Base objective is paired with its +CMRM counterpart; the better value within each pair is in \textbf{bold}.
  Numbers in parentheses indicate the relative change (\%): $+$ denotes accuracy improvement, and $-$ denotes M.APSS reduction compared to the corresponding Base objective. 
  On average across all datasets and objectives, 
  CMRM improves accuracy by $1.58$ and reduces M.APSS by $7.28\%$.}
  \label{tab:results_multi}
\end{table*}

\begin{figure*}[!t]
    \centering
    \begin{minipage}[t]{0.24\linewidth}
    \centering
    \textbf{(a)} Loss
    \includegraphics[width=\linewidth]{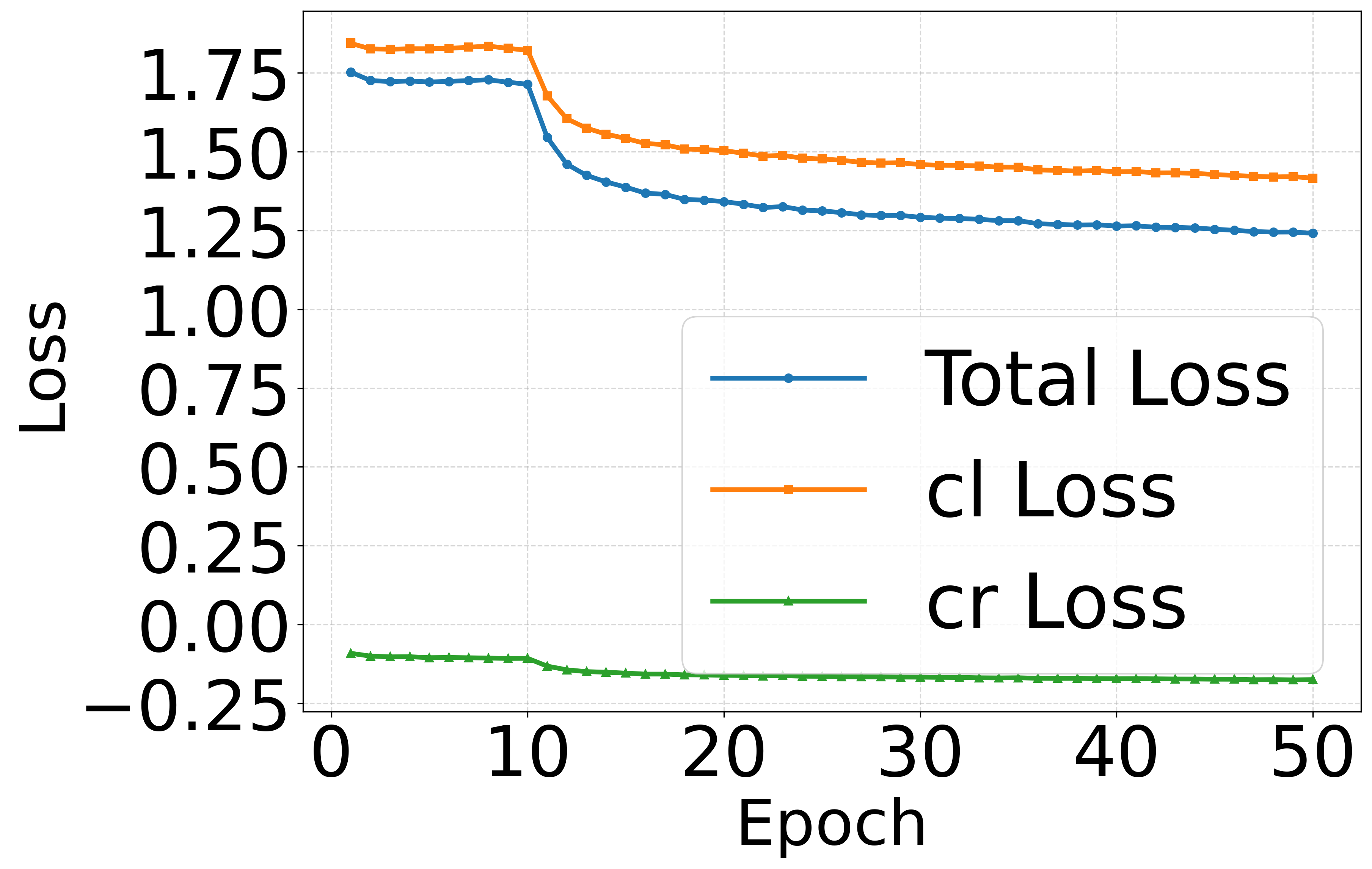}
    \end{minipage}
    \begin{minipage}[t]{0.24\linewidth}
    \centering
    \textbf{(b)} Filtered noise ratio
    \includegraphics[width=\linewidth]{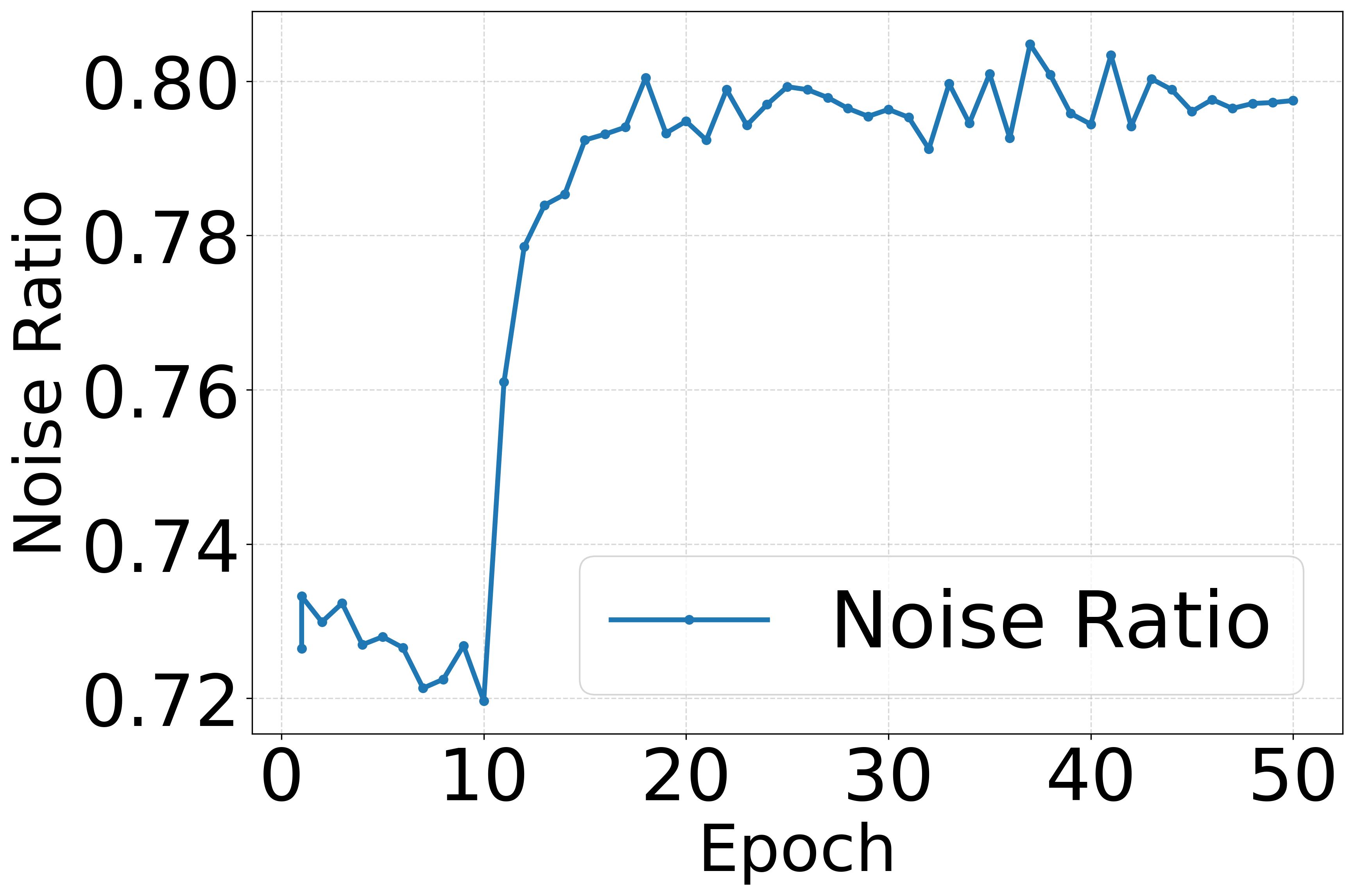}
    \end{minipage}
    \begin{minipage}[t]{0.24\linewidth}
    \centering
    \textbf{(c)} Sensitivity of $\alpha$
    \includegraphics[width=\linewidth]{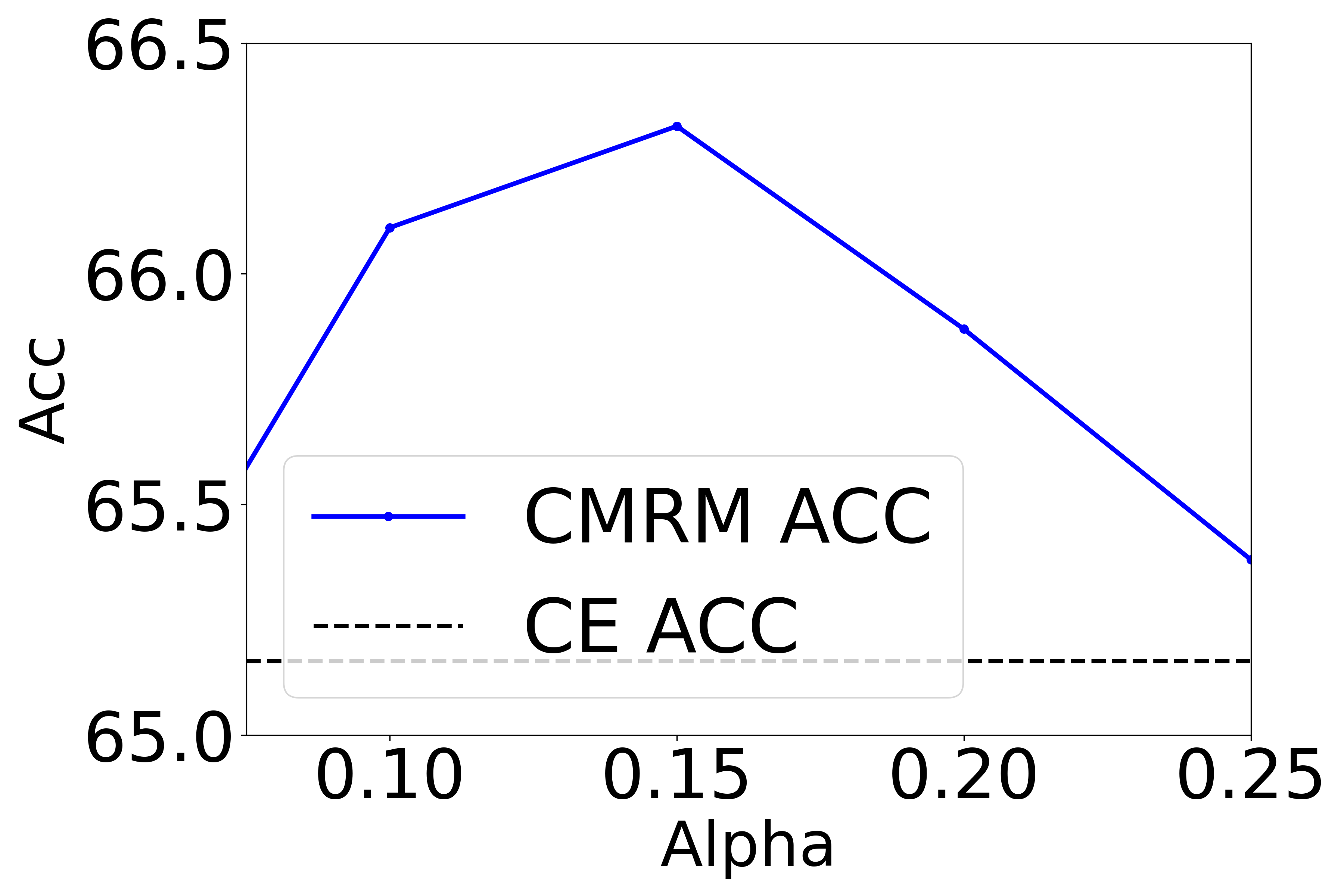}
    \end{minipage}
    \begin{minipage}[t]{0.24\linewidth}
    \centering
    \textbf{(d)} KDE of margin 
    \includegraphics[width=\linewidth]{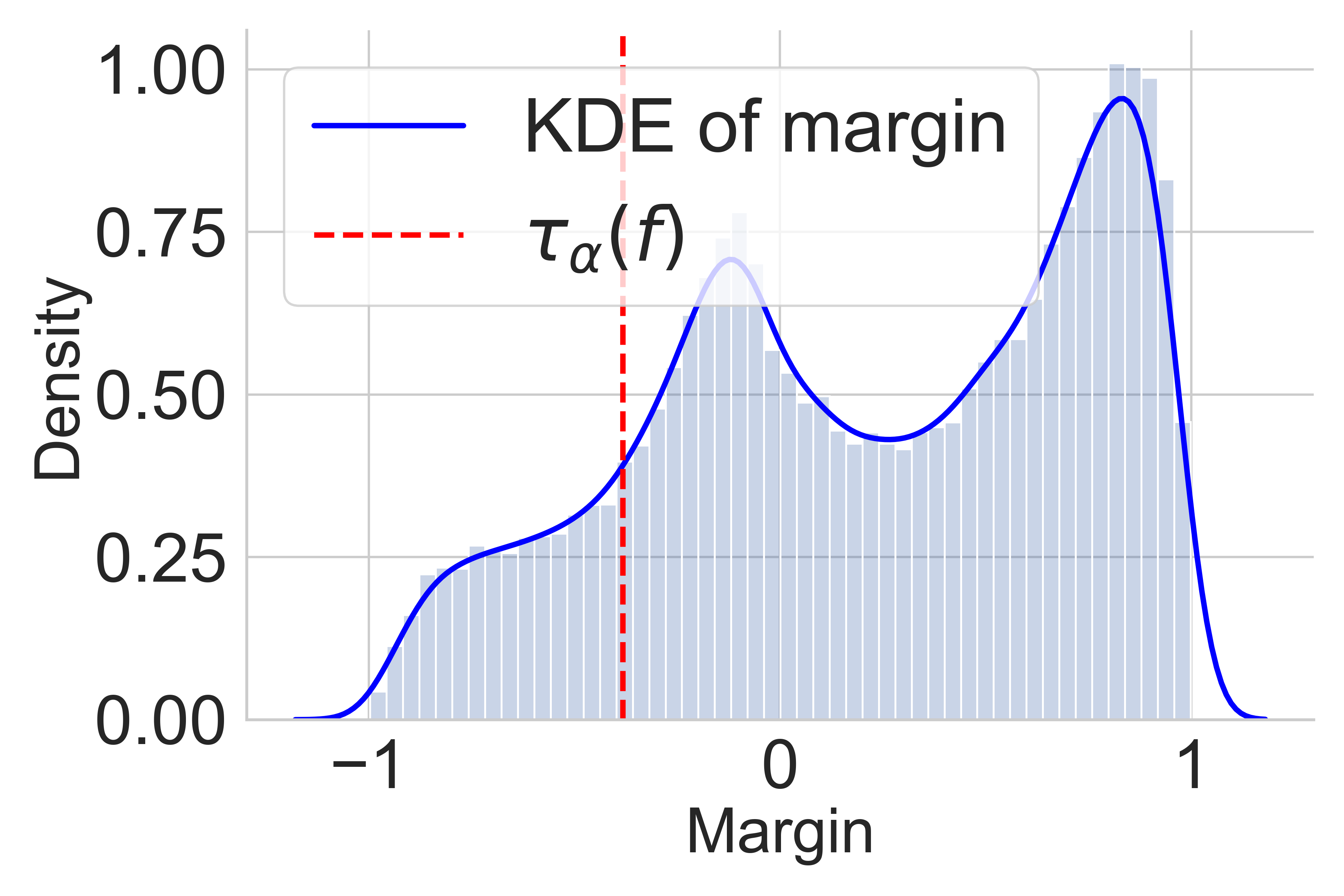}
    \end{minipage}
    \caption{
    \textbf{Justification experiments for multi-class classification } 
    on CIFAR-100 with $20\%$ synthetic label noise. 
    Subfigure \textbf{(a)} shows the training dynamics of total loss (Total), classification loss (cl), and CMRM loss (cr) over epochs. 
    CMRM exhibits stable and monotonic convergence alongside standard loss components.
    Subfigure \textbf{(b)} reports the ratio of noisy samples among those filtered (with soft weight $< 0.5$) out by CMRM ($\alpha = 0.15$) at each epoch, demonstrating that CMRM consistently suppresses noisy examples by excluding low-margin samples during training.
    Subfigure \textbf{(c)} examines the sensitivity of $\alpha$, showing that CMRM maintains higher accuracy than CE across a range of $\alpha$ values, indicating robustness to hyperparameter $\alpha$.
    Subfigure \textbf{(d)} depicts the kernel density estimate (KDE) of the margin distribution, with the vertical dashed line indicating the estimated $\tau_\alpha(f)$ with $\alpha = 0.15$. 
    The density curve is smooth and strictly positive around $\tau_\alpha(f)$, supporting the differentiability and positive-density assumption in Proposition~\ref{proposition:quantile_gap}.
    }
    \label{fig:multi_justification}
\end{figure*}

\subsection{Multi-class Classification Experiments}
\label{subsection:multiclass_classification_experiments}

\textbf{Datasets.}
We evaluate CMRM under both synthetic and real-world noisy supervision. 
For \emph{synthetic noise}, we conduct experiments on CIFAR-100~\citep{krizhevsky2009learning}, mini-ImageNet~\citep{vinyals2016matching}, and Food-101~\citep{bossard2014food}. 
CIFAR-100 is corrupted with class-conditional label noise, where each label is randomly replaced with another label within the same coarse superclass, following the protocol of~\citep{yao2023latent}.
For mini-ImageNet and Food-101, we apply symmetric label flips following~\citep{jiang2020beyond}.
We vary the noise rate from $\{0\%, 5\%, 10\%, 20\%, 30\%, 40\%\}$, where $0\%$ corresponds to the clean-label setting; additional implementation details are provided in Appendix~\ref{ssup:sec:multi_experiment_setup}.
For \emph{real-world noise}, we evaluate on CIFAR-10N and its four variants, as well as CIFAR-100N~\citep{wei2022learning}, 
whose labels were collected through human annotation. 
The noise rates of these datasets range from $9.03\%$ to $40.20\%$, covering a broad spectrum of labeling quality\footnote{Detailed statistics and descriptions of CIFAR-10N and CIFAR-100N are summarized at \url{http://noisylabels.com}.}.

\begin{table}[!t]
  \centering
  \resizebox{\columnwidth}{!}{%
    \begin{NiceTabular}{@{}l|cccc@{}}
      \toprule
      \multirow{2}{*}{\textbf{Method}} & \multicolumn{2}{c}{\textbf{CIFAR-10N (worst)}} & \multicolumn{2}{c}{\textbf{CIFAR-100N}}  \\
      \cmidrule(lr){2-3} \cmidrule(lr){4-5} 
      & ACC(\%) & M.APSS  & ACC(\%)  & M.APSS \\
      \midrule
      NI-ERM
      & 95.71 & 0.93
      & 83.17 & 1.49 \\
      NI-ERM+CMRM 
      & \textbf{97.19} & \textbf{0.91}
      & \textbf{83.95} & \textbf{1.29} \\
      & ($+1.48$) & ($-2.15\%$) 
      & ($+0.78$) & ($-13.42\%$) \\
      \bottomrule
    \end{NiceTabular}%
  }
  \caption{\textbf{Top-1 accuracy (\%) and marginal average prediction set size (M.APSS $\downarrow$) on CIFAR-10N and CIFAR-100N corrupted by human annotation noise.} 
  Numbers in parentheses indicate the relative change: $+$ denotes accuracy improvement and $-$ denotes M.APSS reduction. 
  CMRM consistently improves accuracy and reduces uncertainty across CIFAR-N variants, with the largest gains observed on CIFAR-10N Worst and CIFAR-100N. 
  Full results on all CIFAR-10N variants (Aggre, Rand1–3, Worst) and CIFAR-100N are provided in Appendix~\ref{ssup:sec:experiment}.
  }
  \label{tab:results_multi_cifarn}
\end{table}

\textbf{Baselines and ML Models.}
For synthetic noise experiments, we compare CMRM against four representative families of methods: 
\textbf{(i)} cross-entropy (CE), the standard objective; 
\textbf{(ii)} Focal loss~\citep{lin2017focal}, a standard robustness-oriented variant; 
\textbf{(iii)} LDAM~\citep{cao2019learning}, a margin-based objective; and 
\textbf{(iv)} GCE~\citep{zhang2018generalized}, a general-purpose noise-robust loss that remains a common baseline in the recent learning from noisy labels literature \citep{englesson2024robust,nguyen2024noisy}. 
All methods use a ResNet-20~\citep{he2016deep} backbone with standard data augmentation.  
For real-world noise experiments, we benchmark CMRM against NI-ERM~\citep{zhu2024label}, a recent state-of-the-art approach that achieves strong performance on CIFAR-10N and CIFAR-100N by training a linear classifier on frozen DINOv2 \citep{oquab2023dinov2} features.
All methods share the same protocol.
For CMRM, We set $\texttt{temp}=1.0$ in all experiments and select $\lambda \in \{0.05, \cdots,0.25\}$ and $\alpha \in \{0.05, \cdots,0.25\}$ via grid search; 
See Appendix \ref{ssup:sec:multi_experiment_setup} for full details.

\textbf{Evaluation Metrics.}
We evaluate models using marginal Top-1 accuracy and marginal average prediction set size (M.APSS).
To assess predictive uncertainty, we adopt CP and define prediction sets and average prediction set size (APSS)~\citep{romano2020classification,angelopoulos2022image}.
For each input, CP outputs a prediction set containing the true label with high probability, controlled by a target coverage rate of $0.9$.
For multi-class settings, M.APSS denotes the marginal variant of APSS, i.e., averaged uniformly over all test examples. It also complements top-1 accuracy by capturing how sharply the model distinguishes plausible labels under noisy supervision. 
We also report class-conditional variants (PC APSS, NC APSS) for binary classification.


\begin{table*}[!t]
\centering
\setlength{\tabcolsep}{4pt}
\resizebox{\textwidth}{!}{%
\begin{NiceTabular}{@{} l *{8}{c}@{}}
\toprule
\textbf{Method} & \multicolumn{8}{c}{\textbf{Evaluation Metric}} \\
\cmidrule(lr){2-9}
& AUROC ($\uparrow$) & AUPRC ($\uparrow$) & FNR ($\downarrow$) & FPR ($\downarrow$)
  & ACC ($\uparrow$) & M.APSS ($\downarrow$) & PC APSS ($\downarrow$) & NC APSS ($\downarrow$) \\
\midrule
LR              & 0.784 & 0.885 & \textbf{0.073} & 0.571 & 0.802 & 1.223 & 1.154 & 1.432 \\
LR + CMRM        & \textbf{0.852 ($+0.068$)} & \textbf{0.925 ($+0.04$)} & 0.082 ($+0.009$) & \textbf{0.422 ($-0.149$)} & \textbf{0.833 ($0.031$)} & \textbf{1.209 ($-1.15\%$)} & \textbf{1.109 ($-3.9\%$)} & \textbf{1.308 ($-8.66\%$)} \\
Focal           & 0.809 & 0.890 & 0.136 & 0.388 & 0.801 & 1.257 & 1.224 & 1.356 \\
Focal + CMRM     & \textbf{0.872 ($+0.063$)} & \textbf{0.942 ($+0.052$)} & \textbf{0.128 ($-0.008$)} & \textbf{0.324 ($-0.064$)} & \textbf{0.823 ($0.022$)} & \textbf{1.221 $-2.87\%$} & \textbf{1.148 $-6.21\%$} & \textbf{1.295 $-4.5\%$} \\
SVM             & 0.808 & 0.925 & \textbf{0.029} & 0.807 & 0.776 & 1.276 & 1.370 & 1.512 \\
SVM + CMRM       & \textbf{0.847 ($+3.9\%$)} & \textbf{0.937 ($+1.2\%$)} & 0.048 ($+1.9\%$) & \textbf{0.585 ($-22.2\%$)} & \textbf{0.817 ($+4.1\%$)} & \textbf{1.199 ($-6.03\%$)} & \textbf{1.322 ($-3.5\%$)} & \textbf{1.343 ($-11.17\%$)} \\
GCE             & 0.819 & 0.904 & \textbf{0.119} & 0.424 & \textbf{0.804} & 1.286 & \textbf{1.176} & 1.396 \\
GCE + CMRM      & \textbf{0.846 ($+0.027$)} & \textbf{0.928 ($+0.024$)} & 0.172 ($+0.053$) & \textbf{0.286 ($-0.138$)} & 0.800 ($-0.004$) & \textbf{1.273 ($-1.01\%$)} & 1.207 ($+2.63\%$) & \textbf{1.340 ($-4.01\%$)} \\
\bottomrule
\end{NiceTabular}
}%
\caption{\textbf{Results for binary classification on Adult} with $20\%$ label noise. 
We report ranking (AUROC, AUPRC), error rates (FNR, FPR), accuracy (ACC), and uncertainty (M.APSS, PC APSS, NC APSS). 
$\uparrow$ indicates higher is better; $\downarrow$ lower is better. 
Best results for each base method (LR, Focal, SVM, GCE) are in \textbf{bold}. 
Numbers in parentheses indicate absolute and relative changes (\%) of CMRM, where $+$ denotes increase and $-$ denotes decrease. 
On average, CMRM improves AUROC by $0.049$, AUPRC by $0.032$, and ACC by $0.023$, while reducing FPR by $0.143$ and slightly increasing FNR by $0.018$. 
For uncertainty-aware metrics, CMRM reduces M.APSS, PC APSS, and NC APSS by $2.77\%$, $2.75\%$, and $7.09\%$ on average, respectively. 
Results on Email and Credit show similar trends and are provided in Appendix~\ref{supp:subsec:binary_results}.
}
\label{tab:binary_main}
\end{table*}

\textbf{CMRM improves accuracy and reduces uncertainty under different types of noise.}
Table \ref{tab:results_multi} reports results on CIFAR-100, mini-ImageNet, and Food-101 corrupted by synthetic label noise. 
For each objective (CE, Focal, LDAM, GCE), we compare the Base model with its +CMRM counterpart. 
On average across all datasets and objectives, CMRM improves accuracy by $1.58$ and reduces M.APSS by $7.28\%$. 
The improvements are most pronounced for CE, Focal, and GCE, with accuracy gains of up to $3.39$ and substantial uncertainty reduction. 
Even when combined with LDAM, which already encourages margin separation, CMRM consistently yields additional accuracy improvements without increasing uncertainty.
Table \ref{tab:results_multi_cifarn} summarizes results on CIFAR-10N and CIFAR-100N with human-annotated label noise. 
CMRM consistently outperforms NI-ERM across all CIFAR-N variants, achieving the largest gains on the most challenging settings (CIFAR-10N Worst and CIFAR-100N). 
Complete results on all CIFAR-10N variants (Aggre, Rand1–3, Worst) are provided in Appendix~\ref{ssup:sec:experiment}. 
These findings clearly demonstrate that CMRM improves accuracy and reduces uncertainty under both synthetic and real-world noisy supervision.

\textbf{CMRM loss convergence results.}
Figure~\ref{fig:multi_justification}(a) shows the training dynamics of the classification loss and the CMRM regularization loss. 
Both components decrease steadily and stabilize as training progresses, indicating smooth joint optimization. 
The CMRM term integrates with standard objectives and does not introduce instability or slow down of convergence, demonstrating that CMRM can be efficiently optimized.

\textbf{CMRM filters out noisy samples during training.}
Figure~\ref{fig:multi_justification}(b) shows the fraction of noisy samples among those excluded by CMRM (with soft weight $< 0.5$) at each epoch. 
This proportion rapidly increases during the early training phase and stabilizes above $78\%$, indicating that CMRM consistently identifies and filters out mislabeled examples via its margin-based thresholding mechanism.

\textbf{CMRM is robust to the choice of hyperparameter $\alpha$.}
Figure~\ref{fig:multi_justification}(c) examines the sensitivity of CMRM to the hyperparameter $\alpha$.
Across a range of $\alpha$ values, CMRM consistently achieves higher accuracy than CE, indicating that its performance is robust to the choice of $\alpha$ and does not rely on careful hyperparameter tuning.
Notably, CMRM performs well even when $\alpha$ does not match the true noise rate, suggesting that exact noise rate knowledge is unnecessary.
In practice, $\alpha \in [0.1, 0.2]$ consistently yields strong results across all settings tested.

\textbf{Assumptions in Proposition~\ref{proposition:quantile_gap} are empirically valid.}
Figure~\ref{fig:multi_justification}(d) presents the kernel density estimate (KDE) of the margin distribution, with the vertical dashed line indicating the estimated $\tau_{\alpha}(f)$.
The density curve is smooth and strictly positive in the neighborhood of $\tau_{\alpha}(f)$, supporting the differentiability and positive-density assumption in Proposition~\ref{proposition:quantile_gap}.

\textbf{CMRM incurs no penalty on clean labels.}
Notably, CMRM also improves or matches accuracy at 0\% noise across all objectives (Table~\ref{tab:results_multi_noise_rates} in Appendix~\ref{ssup:sec:bianry_experiment_setup}), confirming that the regularizer incurs no penalty even when labels are clean.

\subsection{Binary Classification Experiments}
\label{subsec:binary_results}

\textbf{Experiment Setup.}
We also evaluate the binary variant of CMRM on three datasets: 
Email~\citep{spambase_94}, Credit~\citep{credit_approval_27}, and Adult~\citep{adult_2}. 
To simulate label noise, we randomly flip $20\%$ of the training labels while keeping the test labels clean.
Additional implementation details are provided in Appendix~\ref{supp:subsec:binary_setup}.
Baselines include logistic regression (LR), focal loss, support vector machines (SVM) with hinge loss as a margin-based method, and GCE. 
LR, Focal, and GCE models use a two-layer MLP, 
while SVM employs a linear kernel with default regularization.
We report AUC-ROC and PR-AUC to assess ranking quality, FPR and FNR to capture class-specific error tendencies, and Accuracy as a general indicator. 
We also measure the predictive uncertainty using APSS, including marginal (M.APSS), positive-class (PC APSS), and negative-class (NC APSS) variants.
Larger values of AUC-ROC, PR-AUC, and Accuracy indicate better performance ($\uparrow$), whereas smaller values of FPR, FNR, and APSS metrics (M.APSS, PC APSS, NC APSS) are preferred ($\downarrow$).


\textbf{CMRM improves robustness for binary classification.}
Table~\ref{tab:binary_main} reports results on the Adult dataset with $20\%$ label noise.
CMRM consistently improves ranking performance, with AUROC and AUPRC increasing by $0.049$ and $0.032$ on average, respectively, indicating that it enables models to better separate classes under noisy supervision.
These gains are accompanied by improvements in classification metrics, with Accuracy increasing by $0.023$ on average and FPR decreasing by $0.143$, at the cost of a modest increase in FNR ($+0.018$).
Finally, CMRM reduces predictive uncertainty, as reflected by lower M.APSS, PC APSS, and NC APSS (average reduction of $2.77\%$, $2.75\%$, and $7.1\%$, respectively), indicating sharper and more discriminative predictions.
Overall, these improvements demonstrate that CMRM enhances robustness in binary classification under label noise.
Similar trends are observed on the Email and Credit datasets (see Appendix~\ref{supp:subsec:binary_results} for complete results).

\begin{figure}[!t]
    \centering
    \begin{minipage}[t]{0.48\linewidth}
    \centering
    \textbf{(a)} Loss
    \includegraphics[width=\linewidth]{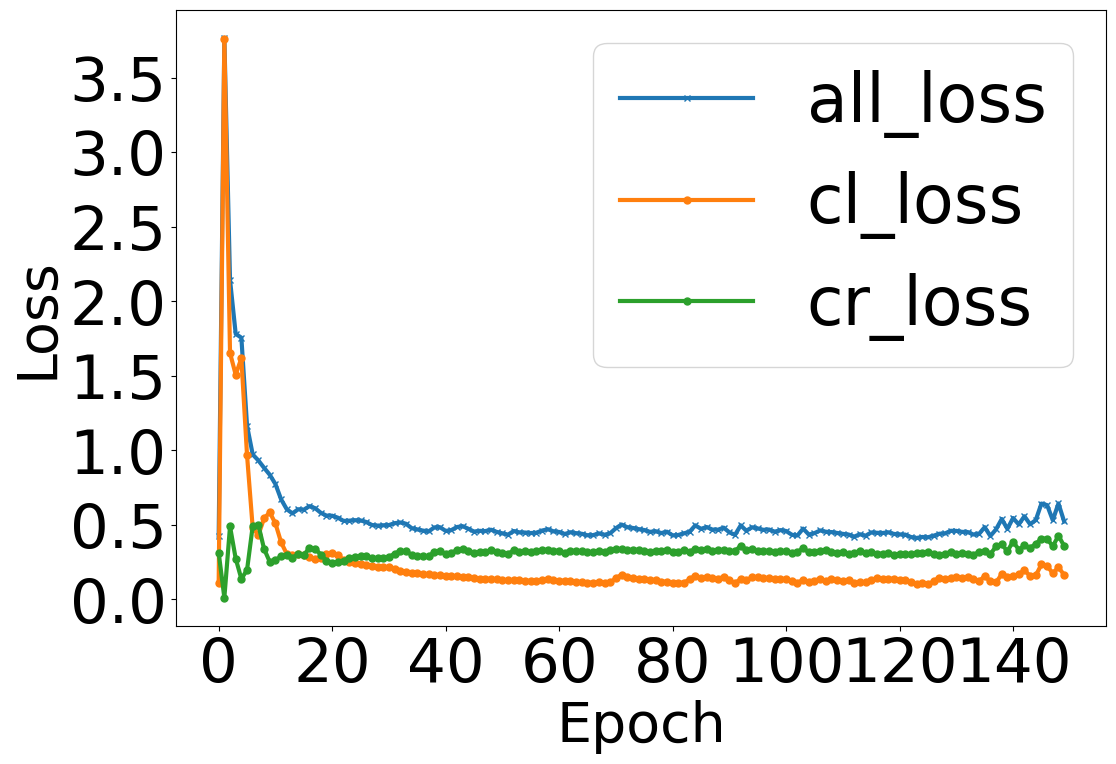}
    \end{minipage}
    \begin{minipage}[t]{0.48\linewidth}
    \centering
    \textbf{(b)} Two thresholds
    \includegraphics[width=\linewidth]{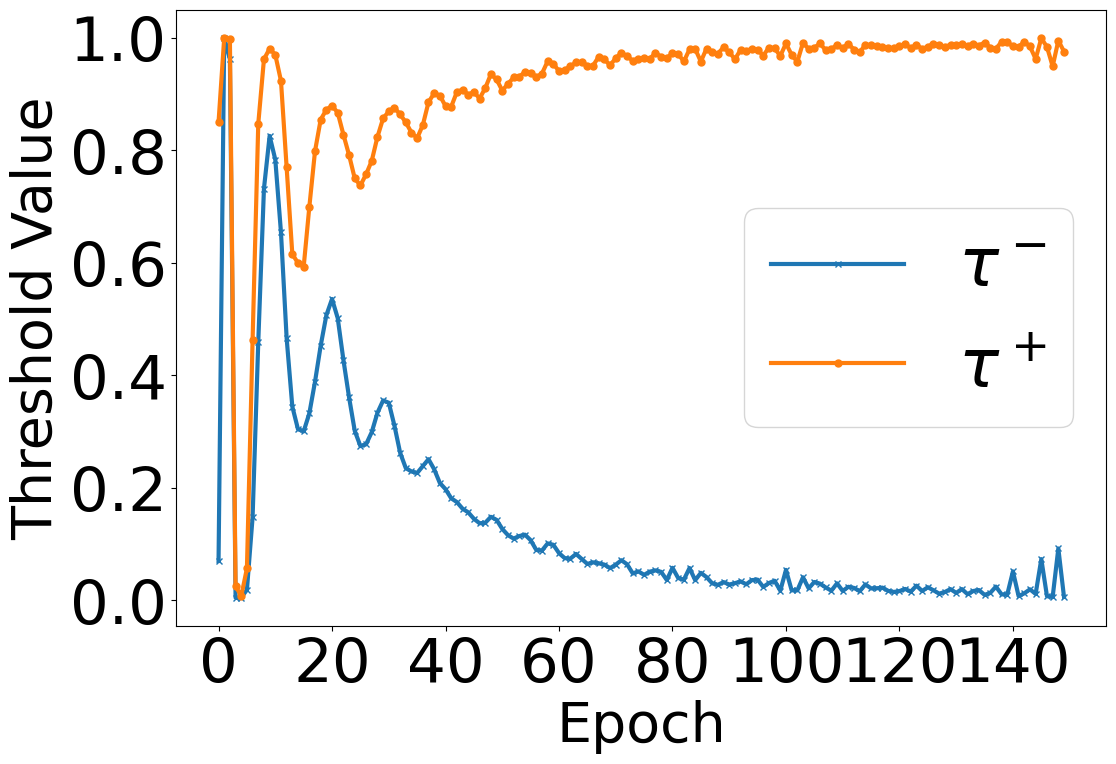}
    \end{minipage}
    \caption{
    \textbf{
    Training dynamics of LR+CMRM for binary classification  
    } on the Email dataset.
    Subfigure \textbf{(a)} Training dynamics of total loss (all loss), classification loss (cl loss), and CMRM loss (cr loss) over epochs. 
    CMRM exhibits stable and monotonic convergence alongside standard loss components.
    Subfigure \textbf{(b)} $\tau^-$ (negative class threshold) and $\tau^+$ (positive class threshold) of LR+CMRM during training.
    The separation between the thresholds increases, indicating that CMRM actively maximizes the margin between positive and negative classes.
    }
    \label{fig:binary_justification}
\end{figure}

\begin{figure}[!t]
    \centering
    \begin{minipage}[t]{0.48\linewidth}
    \centering
    \textbf{(a)} LR
    \includegraphics[width=\linewidth]{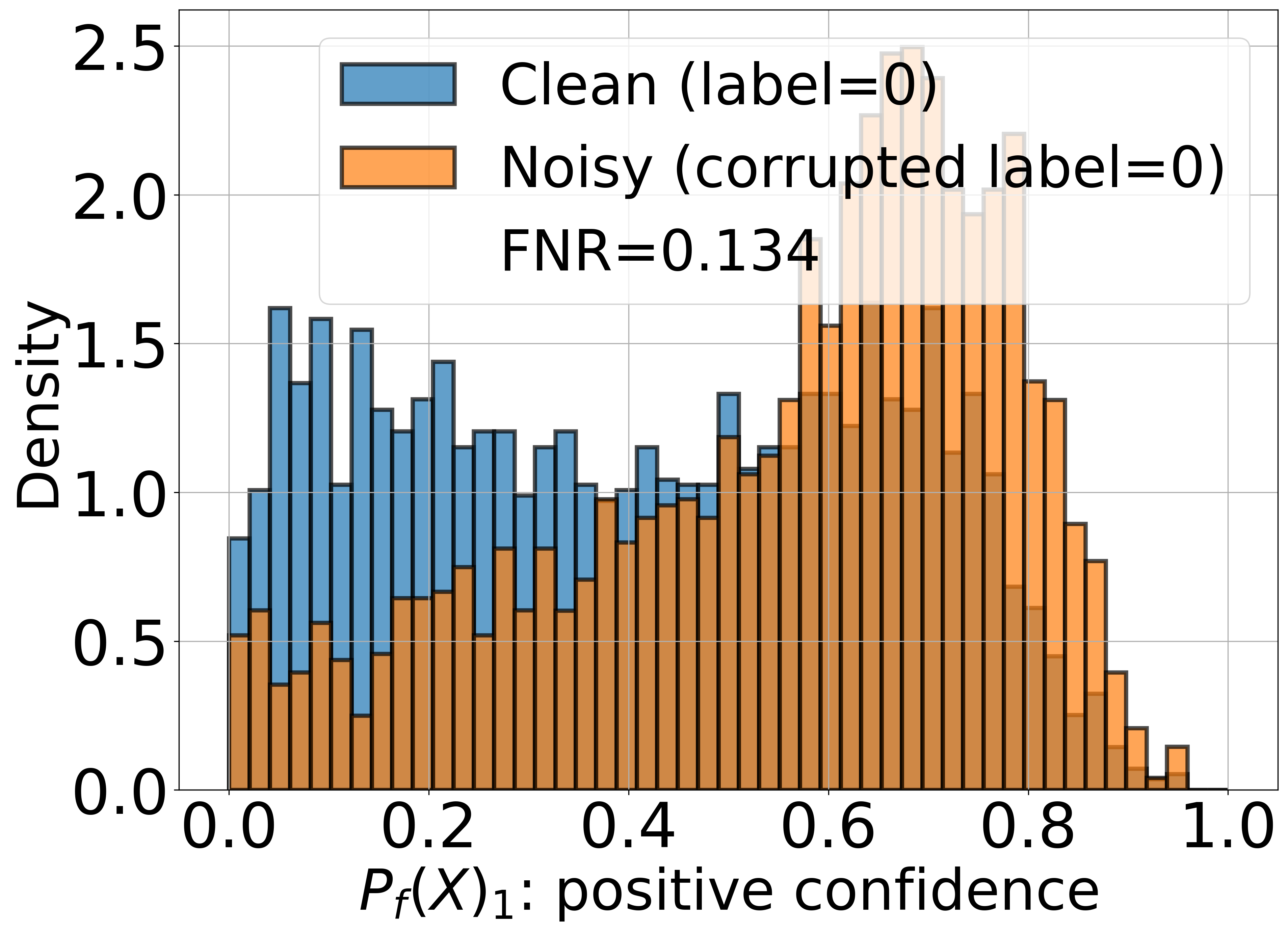}
    \\
    \includegraphics[width=\linewidth]{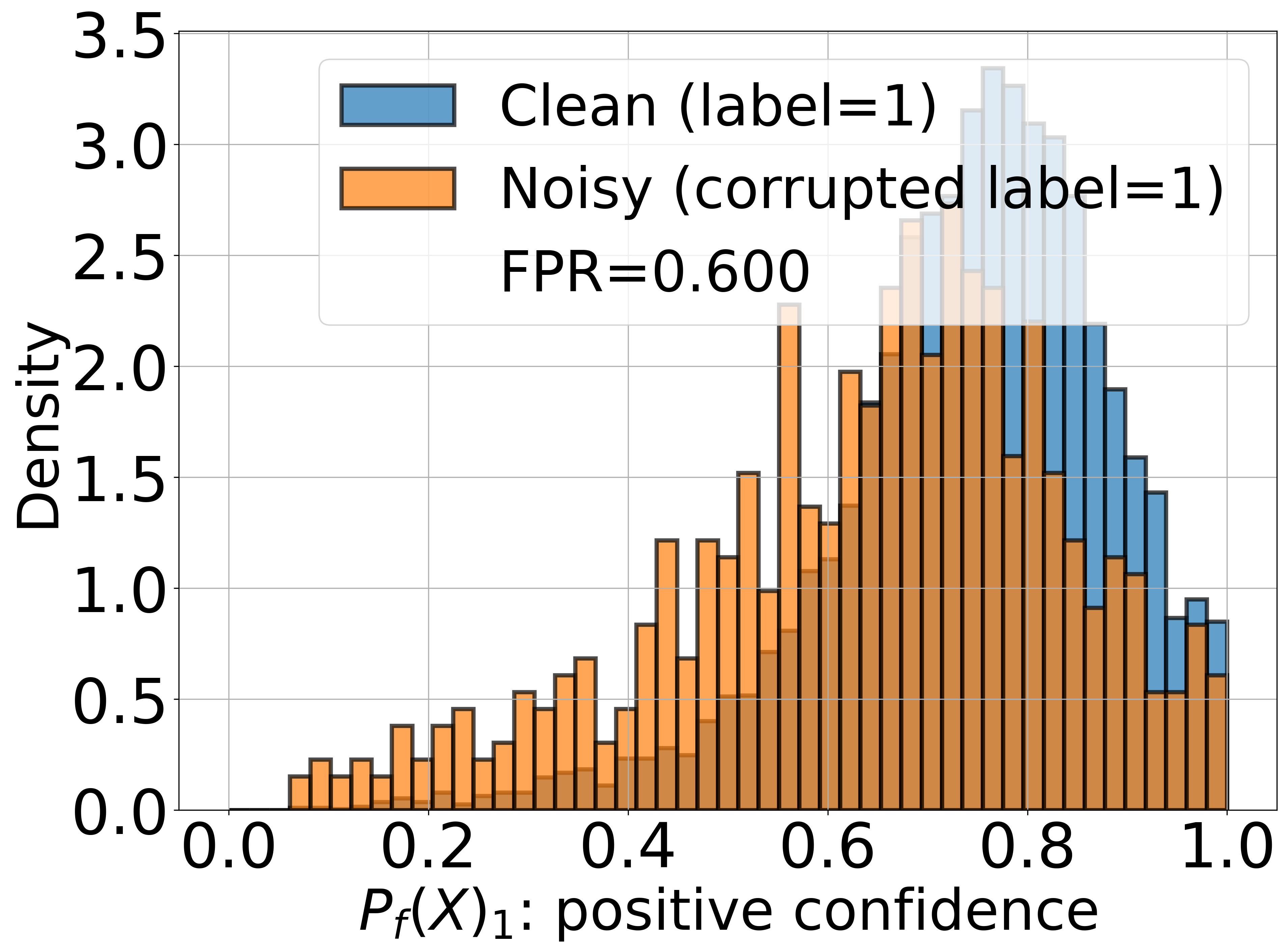}
    \end{minipage}
    \begin{minipage}[t]{0.48\linewidth}
    \centering
    \textbf{(b)} LR+CMRM
    \includegraphics[width=\linewidth]{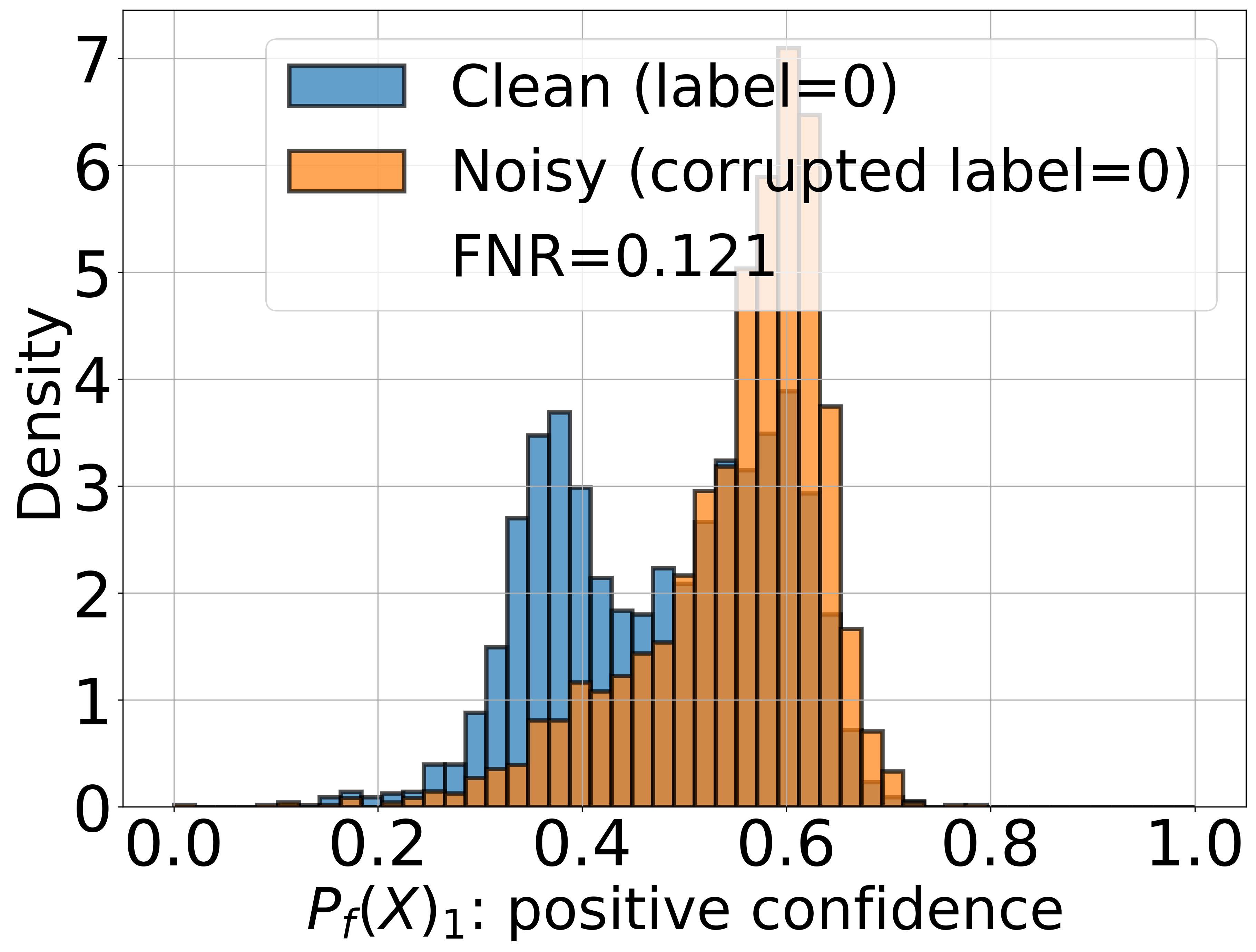}
    \\
    \includegraphics[width=\linewidth]{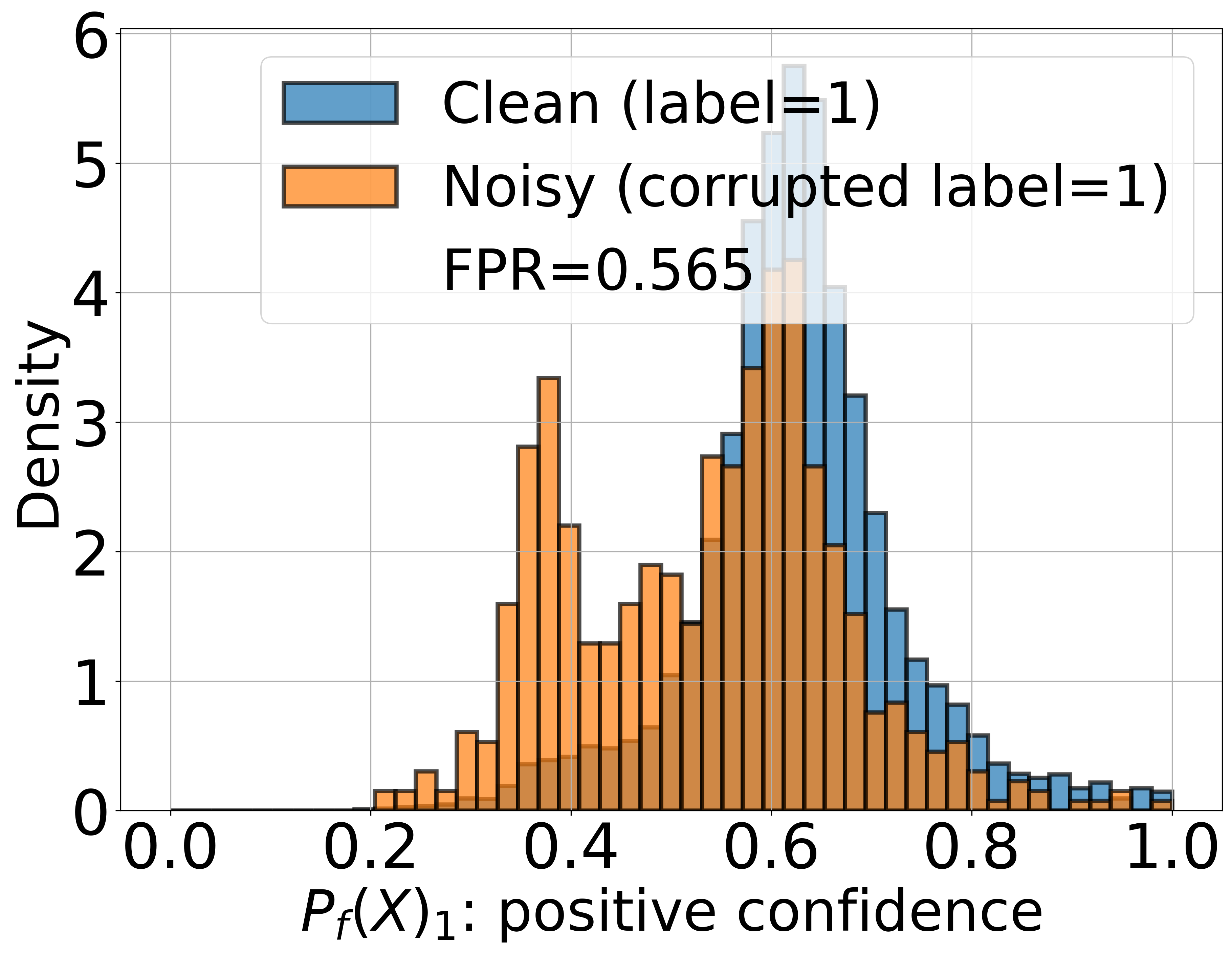}
    \end{minipage}
    \caption{
    \textbf{
    Histograms of positive confidence distributions for clean (blue) and noisy (orange) samples} on the Credit dataset with $20\%$ label noise. 
    The top and bottom rows correspond to samples with observed labels $\widetilde Y=0$ (negative) and $\widetilde Y=1$ (positive), respectively. Distributions are obtained using LR (left) and LR+CMRM (right). 
    CMRM induces a clearer separation between clean and noisy confidence distributions for both classes.
    }
    \label{fig:binary_score}
\end{figure}

\textbf{CMRM optimization dynamics in binary classification.}
Figure \ref{fig:binary_justification} examines the training behavior of LR + CMRM on the Email dataset with $20\%$ label noise.
Subfigure (a) shows that both the classification and CMRM regularization losses decrease steadily and stabilize, indicating that the joint objective can be optimized smoothly on binary classification data. 
Subfigure (b) tracks the evolution of the class-conditional thresholds $\tau^+$ and $\tau^-$ during training. 
$\tau^-$ steadily increases while $\tau^+$ decreases, leading to a widening gap between them. 
This growing separation reflects that CMRM could actively enlarge the margin between positive and negative classes.

\textbf{CMRM separates clean and noisy supervision.}
Figure~\ref{fig:binary_score} shows the effect of CMRM on positive-class confidence distributions under label noise.
With LR (Figure~\ref{fig:binary_score}a), the distributions of clean and noisy samples overlap substantially for both observed positive and negative label groups, indicating limited ability to distinguish reliable labels from corrupted labels.
In contrast, LR+CMRM (Figure~\ref{fig:binary_score}b) yields a clearer separation: clean samples concentrate in high-confidence regions for their true class, while noisy samples are pushed toward lower-confidence regions, improving their distinguishability in the confidence space.

\section{Conclusion}

We propose CMRM, a simple envelope framework for robust learning under label noise, 
requiring only mild regularity of the margin distribution rather than noise models, clean data or auxiliary resources. 
CMRM uses a batch-wise conformal quantile on confidence margins to focus training on reliable samples while suppressing likely corrupted ones. 
Theoretically, we establish a learning bound under arbitrary label noise. 
Empirically, CMRM integrates smoothly into standard pipelines and consistently improves accuracy and robustness across binary and multi-class benchmarks. 
As a single regularization term with no architectural cost, CMRM is a natural default when training with noisy labels.

\section*{Acknowledgements}

The authors gratefully acknowledge the in part support by the USDA-NIFA funded AgAID Institute
award 2021-67021-35344, and NSF grants CNS-2312125, IIS-2443828, DUE-2519063. 
The views expressed are those of
the authors and do not reflect the official policy or position of the USDA-NIFA and NSF.

\bibliographystyle{plainnat}
\bibliography{ref}

\begin{thebibliography}{74}
\providecommand{\natexlab}[1]{#1}
\providecommand{\url}[1]{\texttt{#1}}
\expandafter\ifx\csname urlstyle\endcsname\relax
  \providecommand{\doi}[1]{doi: #1}\else
  \providecommand{\doi}{doi: \begingroup \urlstyle{rm}\Url}\fi

\bibitem[Angelopoulos and Bates(2021)]{angelopoulos2021gentle}
Anastasios~N Angelopoulos and Stephen Bates.
\newblock A gentle introduction to conformal prediction and distribution-free uncertainty quantification.
\newblock \emph{arXiv preprint arXiv:2107.07511}, 2021.

\bibitem[Angelopoulos et~al.(2022)Angelopoulos, Kohli, Bates, Jordan, Malik, Alshaabi, Upadhyayula, and Romano]{angelopoulos2022image}
Anastasios~N Angelopoulos, Amit~Pal Kohli, Stephen Bates, Michael Jordan, Jitendra Malik, Thayer Alshaabi, Srigokul Upadhyayula, and Yaniv Romano.
\newblock Image-to-image regression with distribution-free uncertainty quantification and applications in imaging.
\newblock In \emph{International Conference on Machine Learning}, pages 717--730. PMLR, 2022.

\bibitem[Arazo et~al.(2019)Arazo, Ortego, Albert, O’Connor, and McGuinness]{arazo2019unsupervised}
Eric Arazo, Diego Ortego, Paul Albert, Noel O’Connor, and Kevin McGuinness.
\newblock Unsupervised label noise modeling and loss correction.
\newblock In \emph{International conference on machine learning}, pages 312--321. PMLR, 2019.

\bibitem[Arjovsky et~al.(2017)Arjovsky, Chintala, and Bottou]{arjovsky2017wasserstein}
Martin Arjovsky, Soumith Chintala, and L{\'e}on Bottou.
\newblock Wasserstein generative adversarial networks.
\newblock In \emph{International conference on machine learning}, pages 214--223. PMLR, 2017.

\bibitem[Arpit et~al.(2017)Arpit, Jastrzebski, Ballas, Krueger, Bengio, Kanwal, Maharaj, Fischer, Courville, Bengio, et~al.]{arpit2017closer}
Devansh Arpit, Stanislaw Jastrzebski, Nicolas Ballas, David Krueger, Emmanuel Bengio, Maxinder~S Kanwal, Tegan Maharaj, Asja Fischer, Aaron Courville, Yoshua Bengio, et~al.
\newblock A closer look at memorization in deep networks.
\newblock In \emph{International Conference on Machine Learning}, pages 233--242. PMLR, 2017.

\bibitem[Bartlett and Mendelson(2002)]{bartlett2002rademacher}
Peter~L Bartlett and Shahar Mendelson.
\newblock Rademacher and gaussian complexities: Risk bounds and structural results.
\newblock \emph{Journal of Machine Learning Research}, 3:\penalty0 463--482, 2002.

\bibitem[Becker and Kohavi(1996)]{adult_2}
Barry Becker and Ronny Kohavi.
\newblock {Adult}.
\newblock UCI Machine Learning Repository, 1996.
\newblock {DOI}: https://doi.org/10.24432/C5XW20.

\bibitem[Bossard et~al.(2014)Bossard, Guillaumin, and Van~Gool]{bossard2014food}
Lukas Bossard, Matthieu Guillaumin, and Luc Van~Gool.
\newblock Food-101--mining discriminative components with random forests.
\newblock In \emph{Computer Vision--ECCV 2014: 13th European Conference, Zurich, Switzerland, September 6-12, 2014, Proceedings, Part VI 13}, pages 446--461. Springer, 2014.

\bibitem[Cao et~al.(2019)Cao, Wei, Gaidon, Arechiga, and Ma]{cao2019learning}
Kaidi Cao, Colin Wei, Adrien Gaidon, Nikos Arechiga, and Tengyu Ma.
\newblock Learning imbalanced datasets with label-distribution-aware margin loss.
\newblock \emph{Advances in neural information processing systems}, 32, 2019.

\bibitem[Celard et~al.(2023)Celard, Iglesias, Sorribes-Fdez, Romero, Vieira, and Borrajo]{celard2023survey}
Pedro Celard, Eva~Lorenzo Iglesias, Jos{\'e}~Manuel Sorribes-Fdez, Rub{\'e}n Romero, A~Seara Vieira, and Lourdes Borrajo.
\newblock A survey on deep learning applied to medical images: from simple artificial neural networks to generative models.
\newblock \emph{Neural Computing and Applications}, 35\penalty0 (3):\penalty0 2291--2323, 2023.

\bibitem[Cheng et~al.(2020)Cheng, Zhu, Li, Gong, Sun, and Liu]{cheng2020learning}
Hao Cheng, Zhaowei Zhu, Xingyu Li, Yifei Gong, Xing Sun, and Yang Liu.
\newblock Learning with instance-dependent label noise: A sample sieve approach.
\newblock \emph{arXiv preprint arXiv:2010.02347}, 2020.

\bibitem[Cui et~al.(2019)Cui, Jia, Lin, Song, and Belongie]{cui2019class}
Yin Cui, Menglin Jia, Tsung-Yi Lin, Yang Song, and Serge Belongie.
\newblock Class-balanced loss based on effective number of samples.
\newblock In \emph{Proceedings of the IEEE/CVF conference on computer vision and pattern recognition}, pages 9268--9277, 2019.

\bibitem[Dudley(2018)]{dudley2018real}
Richard~M Dudley.
\newblock \emph{Real analysis and probability}.
\newblock Chapman and Hall/CRC, 2018.

\bibitem[Einbinder et~al.(2022)Einbinder, Romano, Sesia, and Zhou]{einbinder2022training}
Bat-Sheva Einbinder, Yaniv Romano, Matteo Sesia, and Yanfei Zhou.
\newblock Training uncertainty-aware classifiers with conformalized deep learning.
\newblock \emph{Advances in neural information processing systems}, 35:\penalty0 22380--22395, 2022.

\bibitem[Elsayed et~al.(2018)Elsayed, Krishnan, Mobahi, Regan, and Bengio]{elsayed2018large}
Gamaleldin~Fathy Elsayed, Dilip Krishnan, Hossein Mobahi, Kevin Regan, and Samy Bengio.
\newblock Large margin deep networks for classification.
\newblock In \emph{Advances in Neural Information Processing Systems (NeurIPS)}, volume~31, 2018.

\bibitem[Englesson and Azizpour(2024)]{englesson2024robust}
Erik Englesson and Hossein Azizpour.
\newblock Robust classification via regression for learning with noisy labels.
\newblock In \emph{ICLR 2024-The Twelfth International Conference on Learning Representations, Messe Wien Exhibition and Congress Center, Vienna, Austria, May 7-11t, 2024}, 2024.

\bibitem[Fontana et~al.(2023)Fontana, Zeni, and Vantini]{fontana2023conformal}
Matteo Fontana, Gianluca Zeni, and Simone Vantini.
\newblock Conformal prediction: a unified review of theory and new challenges.
\newblock \emph{Bernoulli}, 29\penalty0 (1):\penalty0 1--23, 2023.

\bibitem[Fr{\'e}nay and Verleysen(2013)]{frenay2013classification}
Beno{\^\i}t Fr{\'e}nay and Michel Verleysen.
\newblock Classification in the presence of label noise: a survey.
\newblock \emph{IEEE transactions on neural networks and learning systems}, 25\penalty0 (5):\penalty0 845--869, 2013.

\bibitem[Ghosh et~al.(2023{\natexlab{a}})Ghosh, Belkhouja, Yan, and Doppa]{NCP}
Subhankar Ghosh, Taha Belkhouja, Yan Yan, and Janardhan~Rao Doppa.
\newblock {Improving Uncertainty Quantification of Deep Classifiers via Neighborhood Conformal Prediction: Novel Algorithm and Theoretical Analysis}.
\newblock In \emph{Proc. of {AAAI} Conf.}, pages 7722--7730, 2023{\natexlab{a}}.

\bibitem[Ghosh et~al.(2023{\natexlab{b}})Ghosh, Shi, Belkhouja, Yan, Doppa, and Jones]{PRCP}
Subhankar Ghosh, Yuanjie Shi, Taha Belkhouja, Yan Yan, Jana Doppa, and Brian Jones.
\newblock {Probabilistically Robust Conformal Prediction}.
\newblock In \emph{{UAI} Conf.}, volume 216 of \emph{Proc. of Machine Learning Research}, pages 681--690. {PMLR}, 2023{\natexlab{b}}.

\bibitem[Gupta and Gupta(2019)]{gupta2019dealing}
Shivani Gupta and Atul Gupta.
\newblock Dealing with noise problem in machine learning data-sets: A systematic review.
\newblock \emph{Procedia Computer Science}, 161:\penalty0 466--474, 2019.

\bibitem[Han et~al.(2018)Han, Yao, Yu, Niu, Xu, Hu, Tsang, and Sugiyama]{han2018co}
Bo~Han, Quanming Yao, Xingrui Yu, Gang Niu, Miao Xu, Weihua Hu, Ivor Tsang, and Masashi Sugiyama.
\newblock Co-teaching: Robust training of deep neural networks with extremely noisy labels.
\newblock \emph{Advances in neural information processing systems}, 31, 2018.

\bibitem[He et~al.(2016)He, Zhang, Ren, and Sun]{he2016deep}
Kaiming He, Xiangyu Zhang, Shaoqing Ren, and Jian Sun.
\newblock Deep residual learning for image recognition.
\newblock In \emph{Proceedings of the IEEE conference on computer vision and pattern recognition}, pages 770--778, 2016.

\bibitem[Hendrycks et~al.(2018)Hendrycks, Mazeika, Wilson, and Gimpel]{hendrycks2018using}
Dan Hendrycks, Mantas Mazeika, Duncan Wilson, and Kevin Gimpel.
\newblock Using trusted data to train deep networks on labels corrupted by severe noise.
\newblock \emph{Advances in neural information processing systems}, 31, 2018.

\bibitem[Hopkins et~al.(1999)Hopkins, Reeber, Forman, and Suermondt]{spambase_94}
Mark Hopkins, Erik Reeber, George Forman, and Jaap Suermondt.
\newblock Spambase.
\newblock UCI Machine Learning Repository, 1999.
\newblock DOI: https://doi.org/10.24432/C53G6X.

\bibitem[Jiang et~al.(2018)Jiang, Kim, Guan, and Gupta]{jiang2018predictive}
Heinrich Jiang, Been Kim, Melody~Y. Guan, and Maya~R. Gupta.
\newblock To trust or not to trust a classifier.
\newblock In \emph{Advances in Neural Information Processing Systems (NeurIPS)}, volume~31, 2018.

\bibitem[Jiang et~al.(2020)Jiang, Huang, Liu, and Yang]{jiang2020beyond}
Lu~Jiang, Di~Huang, Mason Liu, and Weilong Yang.
\newblock Beyond synthetic noise: Deep learning on controlled noisy labels.
\newblock In \emph{International conference on machine learning}, pages 4804--4815. PMLR, 2020.

\bibitem[Johnson and Khoshgoftaar(2022)]{johnson2022survey}
Justin~M Johnson and Taghi~M Khoshgoftaar.
\newblock A survey on classifying big data with label noise.
\newblock \emph{ACM Journal of Data and Information Quality}, 14\penalty0 (4):\penalty0 1--43, 2022.

\bibitem[Kim et~al.(2021)Kim, Ko, Choi, Yun, et~al.]{kim2021fine}
Taehyeon Kim, Jongwoo Ko, JinHwan Choi, Se-Young Yun, et~al.
\newblock Fine samples for learning with noisy labels.
\newblock \emph{Advances in Neural Information Processing Systems}, 34:\penalty0 24137--24149, 2021.

\bibitem[Kiyani et~al.(2024)Kiyani, Pappas, and Hassani]{kiyani2024length}
Shayan Kiyani, George~J Pappas, and Hamed Hassani.
\newblock Length optimization in conformal prediction.
\newblock \emph{Advances in Neural Information Processing Systems}, 37:\penalty0 99519--99563, 2024.

\bibitem[Krizhevsky et~al.(2009)Krizhevsky, Hinton, et~al.]{krizhevsky2009learning}
Alex Krizhevsky, Geoffrey Hinton, et~al.
\newblock Learning multiple layers of features from tiny images.
\newblock \emph{Toronto, ON, Canada}, 2009.

\bibitem[Lakshminarayanan et~al.(2017)Lakshminarayanan, Pritzel, and Blundell]{lakshminarayanan2017simple}
Balaji Lakshminarayanan, Alexander Pritzel, and Charles Blundell.
\newblock Simple and scalable predictive uncertainty estimation using deep ensembles.
\newblock In \emph{Advances in Neural Information Processing Systems (NeurIPS)}, volume~30, 2017.

\bibitem[Lei et~al.(2018)Lei, G’Sell, Rinaldo, Tibshirani, and Wasserman]{lei2018distribution}
Jing Lei, Max G’Sell, Alessandro Rinaldo, Ryan~J Tibshirani, and Larry Wasserman.
\newblock Distribution-free predictive inference for regression.
\newblock \emph{Journal of the American Statistical Association}, 113\penalty0 (523):\penalty0 1094--1111, 2018.

\bibitem[Li et~al.(2020)Li, Socher, and Hoi]{li2020dividemix}
Junnan Li, Richard Socher, and Steven~CH Hoi.
\newblock Dividemix: Learning with noisy labels as semi-supervised learning.
\newblock \emph{arXiv preprint arXiv:2002.07394}, 2020.

\bibitem[Li et~al.(2021{\natexlab{a}})Li, Fan, Ma, and Pan]{li2021evaluation}
Sijia Li, Yong Fan, Yue Ma, and Ya~Pan.
\newblock Evaluation of dataset distribution and label quality for autonomous driving system.
\newblock In \emph{2021 IEEE 21st International Conference on Software Quality, Reliability and Security Companion (QRS-C)}, pages 196--200. IEEE, 2021{\natexlab{a}}.

\bibitem[Li et~al.(2021{\natexlab{b}})Li, Liu, Han, Niu, and Sugiyama]{li2021provably}
Xuefeng Li, Tongliang Liu, Bo~Han, Gang Niu, and Masashi Sugiyama.
\newblock Provably end-to-end label-noise learning without anchor points.
\newblock In \emph{International conference on machine learning}, pages 6403--6413. PMLR, 2021{\natexlab{b}}.

\bibitem[Lin et~al.(2017)Lin, Goyal, Girshick, He, and Doll{\'a}r]{lin2017focal}
Tsung-Yi Lin, Priya Goyal, Ross Girshick, Kaiming He, and Piotr Doll{\'a}r.
\newblock Focal loss for dense object detection.
\newblock In \emph{Proceedings of the IEEE international conference on computer vision}, pages 2980--2988, 2017.

\bibitem[Lin et~al.(2024)Lin, Yao, and Liu]{lin2024learning}
Yexiong Lin, Yu~Yao, and Tongliang Liu.
\newblock Learning the latent causal structure for modeling label noise.
\newblock \emph{Advances in Neural Information Processing Systems}, 37:\penalty0 120549--120577, 2024.

\bibitem[Liu et~al.(2020)Liu, Niles-Weed, Razavian, and Fernandez-Granda]{liu2020early}
Sheng Liu, Jonathan Niles-Weed, Narges Razavian, and Carlos Fernandez-Granda.
\newblock Early-learning regularization prevents memorization of noisy labels.
\newblock \emph{Advances in neural information processing systems}, 33:\penalty0 20331--20342, 2020.

\bibitem[Mohri et~al.(2018)Mohri, Rostamizadeh, and Talwalkar]{mohri2018foundations}
Mehryar Mohri, Afshin Rostamizadeh, and Ameet Talwalkar.
\newblock \emph{Foundations of machine learning}.
\newblock MIT press, 2018.

\bibitem[Natarajan et~al.(2013)Natarajan, Dhillon, Ravikumar, and Tewari]{natarajan2013learning}
Nagarajan Natarajan, Inderjit~S Dhillon, Pradeep~K Ravikumar, and Ambuj Tewari.
\newblock Learning with noisy labels.
\newblock \emph{Advances in neural information processing systems}, 26, 2013.

\bibitem[Naveed et~al.(2025)Naveed, Khan, Qiu, Saqib, Anwar, Usman, Akhtar, Barnes, and Mian]{naveed2025comprehensive}
Humza Naveed, Asad~Ullah Khan, Shi Qiu, Muhammad Saqib, Saeed Anwar, Muhammad Usman, Naveed Akhtar, Nick Barnes, and Ajmal Mian.
\newblock A comprehensive overview of large language models.
\newblock \emph{ACM Transactions on Intelligent Systems and Technology}, 16\penalty0 (5):\penalty0 1--72, 2025.

\bibitem[Neyshabur et~al.(2017)Neyshabur, Bhojanapalli, McAllester, and Srebro]{neyshabur2017exploring}
Behnam Neyshabur, Srinadh Bhojanapalli, David McAllester, and Nati Srebro.
\newblock Exploring generalization in deep learning.
\newblock \emph{Advances in neural information processing systems}, 30, 2017.

\bibitem[Nguyen et~al.(2024)Nguyen, Ibrahim, and Fu]{nguyen2024noisy}
Tri Nguyen, Shahana Ibrahim, and Xiao Fu.
\newblock Noisy label learning with instance-dependent outliers: Identifiability via crowd wisdom.
\newblock \emph{Advances in Neural Information Processing Systems}, 37:\penalty0 97261--97298, 2024.

\bibitem[Oquab et~al.(2023)Oquab, Darcet, Moutakanni, Vo, Szafraniec, Khalidov, Fernandez, Haziza, Massa, El-Nouby, et~al.]{oquab2023dinov2}
Maxime Oquab, Timoth{\'e}e Darcet, Th{\'e}o Moutakanni, Huy Vo, Marc Szafraniec, Vasil Khalidov, Pierre Fernandez, Daniel Haziza, Francisco Massa, Alaaeldin El-Nouby, et~al.
\newblock Dinov2: Learning robust visual features without supervision.
\newblock \emph{arXiv preprint arXiv:2304.07193}, 2023.

\bibitem[Patrini et~al.(2017)Patrini, Rozza, Krishna~Menon, Nock, and Qu]{patrini2017making}
Giorgio Patrini, Alessandro Rozza, Aditya Krishna~Menon, Richard Nock, and Lizhen Qu.
\newblock Making deep neural networks robust to label noise: A loss correction approach.
\newblock In \emph{Proceedings of the IEEE conference on computer vision and pattern recognition}, pages 1944--1952, 2017.

\bibitem[Pleiss et~al.(2020)Pleiss, Zhang, Elenberg, and Weinberger]{pleiss2020identifying}
Geoff Pleiss, Tianyi Zhang, Ethan Elenberg, and Kilian~Q Weinberger.
\newblock Identifying mislabeled data using the area under the margin ranking.
\newblock \emph{Advances in Neural Information Processing Systems}, 33:\penalty0 17044--17056, 2020.

\bibitem[Quinlan(1987)]{credit_approval_27}
J.~R. Quinlan.
\newblock {Credit Approval}.
\newblock UCI Machine Learning Repository, 1987.
\newblock {DOI}: https://doi.org/10.24432/C5FS30.

\bibitem[Romano et~al.(2020)Romano, Sesia, and Candes]{romano2020classification}
Yaniv Romano, Matteo Sesia, and Emmanuel Candes.
\newblock Classification with valid and adaptive coverage.
\newblock \emph{Advances in Neural Information Processing Systems}, 33:\penalty0 3581--3591, 2020.

\bibitem[Serfling(2009)]{serfling2009approximation}
Robert~J Serfling.
\newblock \emph{Approximation theorems of mathematical statistics}.
\newblock John Wiley \& Sons, 2009.

\bibitem[Shahrokhi et~al.(2025)Shahrokhi, Roy, Yan, Arnaoudova, and Doppa]{LLM-CP}
Hooman Shahrokhi, Devjeet~Raj Roy, Yan Yan, Venera Arnaoudova, and Janaradhan~Rao Doppa.
\newblock {Conformal Prediction Sets for Deep Generative Models via Reduction to Conformal Regression}, 2025.
\newblock URL \url{https://arxiv.org/abs/2503.10512}.

\bibitem[Shi et~al.(2024{\natexlab{a}})Shi, Zhang, Guo, Yang, Xu, and Wu]{shi2024survey}
Jialin Shi, Kailai Zhang, Chenyi Guo, Youquan Yang, Yali Xu, and Ji~Wu.
\newblock A survey of label-noise deep learning for medical image analysis.
\newblock \emph{Medical image analysis}, 95:\penalty0 103166, 2024{\natexlab{a}}.

\bibitem[Shi et~al.(2024{\natexlab{b}})Shi, Ghosh, Belkhouja, Doppa, and Yan]{R3CP}
Yuanjie Shi, Subhankar Ghosh, Taha Belkhouja, Jana Doppa, and Yan Yan.
\newblock {Conformal Prediction for Class-wise Coverage via Augmented Label Rank Calibration}.
\newblock In \emph{Advances in Neural Information Processing Sys. ({NeurIPS})}, 2024{\natexlab{b}}.

\bibitem[Shi et~al.(2025)Shi, Shahrokhi, Jia, Chen, Doppa, and Yan]{shi2025direct}
Yuanjie Shi, Hooman Shahrokhi, Xuesong Jia, Xiongzhi Chen, Janardhan~Rao Doppa, and Yan Yan.
\newblock Direct prediction set minimization via bilevel conformal classifier training.
\newblock \emph{arXiv preprint arXiv:2506.06599}, 2025.

\bibitem[Song et~al.(2022)Song, Kim, Park, Shin, and Lee]{song2022learning}
Hwanjun Song, Minseok Kim, Dongmin Park, Yooju Shin, and Jae-Gil Lee.
\newblock Learning from noisy labels with deep neural networks: A survey.
\newblock \emph{IEEE transactions on neural networks and learning systems}, 34\penalty0 (11):\penalty0 8135--8153, 2022.

\bibitem[Stutz et~al.(2021)Stutz, Cemgil, Doucet, et~al.]{stutz2021learning}
David Stutz, Ali~Taylan Cemgil, Arnaud Doucet, et~al.
\newblock Learning optimal conformal classifiers.
\newblock \emph{arXiv preprint arXiv:2110.09192}, 2021.

\bibitem[Teubner et~al.(2023)Teubner, Flath, Weinhardt, Van Der~Aalst, and Hinz]{teubner2023welcome}
Timm Teubner, Christoph~M Flath, Christof Weinhardt, Wil Van Der~Aalst, and Oliver Hinz.
\newblock Welcome to the era of chatgpt et al. the prospects of large language models.
\newblock \emph{Business \& Information Systems Engineering}, 65\penalty0 (2):\penalty0 95--101, 2023.

\bibitem[Thirunavukarasu et~al.(2023)Thirunavukarasu, Ting, Elangovan, Gutierrez, Tan, and Ting]{thirunavukarasu2023large}
Arun~James Thirunavukarasu, Darren Shu~Jeng Ting, Kabilan Elangovan, Laura Gutierrez, Ting~Fang Tan, and Daniel Shu~Wei Ting.
\newblock Large language models in medicine.
\newblock \emph{Nature medicine}, 29\penalty0 (8):\penalty0 1930--1940, 2023.

\bibitem[Tjandra and Wiens(2023)]{tjandra2023leveraging}
Donna Tjandra and Jenna Wiens.
\newblock Leveraging an alignment set in tackling instance-dependent label noise.
\newblock In \emph{Conference on Health, Inference, and Learning}, pages 477--497. PMLR, 2023.

\bibitem[Van~der Vaart(2000)]{van2000asymptotic}
Aad~W Van~der Vaart.
\newblock \emph{Asymptotic statistics}, volume~3.
\newblock Cambridge university press, 2000.

\bibitem[Vinyals et~al.(2016)Vinyals, Blundell, Lillicrap, Wierstra, et~al.]{vinyals2016matching}
Oriol Vinyals, Charles Blundell, Timothy Lillicrap, Daan Wierstra, et~al.
\newblock Matching networks for one shot learning.
\newblock \emph{Advances in neural information processing systems}, 29, 2016.

\bibitem[Vovk et~al.(2005)Vovk, Gammerman, and Shafer]{vovk2005algorithmic}
Vladimir Vovk, Alexander Gammerman, and Glenn Shafer.
\newblock \emph{Algorithmic learning in a random world}.
\newblock Springer Science \& Business Media, 2005.

\bibitem[Wang et~al.(2024)Wang, Huang, Lin, and Liu]{wang2024noisegpt}
Haoyu Wang, Zhuo Huang, Zhiwei Lin, and Tongliang Liu.
\newblock Noisegpt: Label noise detection and rectification through probability curvature.
\newblock \emph{Advances in Neural Information Processing Systems}, 37:\penalty0 120159--120183, 2024.

\bibitem[Wei et~al.(2022)Wei, Zhu, Cheng, Liu, Niu, and Liu]{wei2022learning}
Jiaheng Wei, Zhaowei Zhu, Hao Cheng, Tongliang Liu, Gang Niu, and Yang Liu.
\newblock Learning with noisy labels revisited: A study using real-world human annotations.
\newblock In \emph{International Conference on Learning Representations}, 2022.
\newblock URL \url{https://openreview.net/forum?id=TBWA6PLJZQm}.

\bibitem[Xia et~al.(2019)Xia, Liu, Wang, Han, Gong, Niu, and Sugiyama]{xia2019anchor}
Xiaobo Xia, Tongliang Liu, Nannan Wang, Bo~Han, Chen Gong, Gang Niu, and Masashi Sugiyama.
\newblock Are anchor points really indispensable in label-noise learning?
\newblock \emph{Advances in neural information processing systems}, 32, 2019.

\bibitem[Yao et~al.(2023)Yao, Han, Zhou, Zhang, and Tsang]{yao2023latent}
Jiangchao Yao, Bo~Han, Zhihan Zhou, Ya~Zhang, and Ivor~W Tsang.
\newblock Latent class-conditional noise model.
\newblock \emph{IEEE Transactions on Pattern Analysis and Machine Intelligence}, 45\penalty0 (8):\penalty0 9964--9980, 2023.

\bibitem[Yao et~al.(2020)Yao, Liu, Han, Gong, Deng, Niu, and Sugiyama]{yao2020dual}
Yu~Yao, Tongliang Liu, Bo~Han, Mingming Gong, Jiankang Deng, Gang Niu, and Masashi Sugiyama.
\newblock Dual t: Reducing estimation error for transition matrix in label-noise learning.
\newblock \emph{Advances in neural information processing systems}, 33:\penalty0 7260--7271, 2020.

\bibitem[Zhang et~al.(2021)Zhang, Bengio, Hardt, Recht, and Vinyals]{zhang2021understanding}
Chiyuan Zhang, Samy Bengio, Moritz Hardt, Benjamin Recht, and Oriol Vinyals.
\newblock Understanding deep learning requires rethinking generalization.
\newblock \emph{Proceedings of the National Academy of Sciences}, 118\penalty0 (3), 2021.

\bibitem[Zhang et~al.(2024)Zhang, Huang, Jin, and Lu]{zhang2024vision}
Jingyi Zhang, Jiaxing Huang, Sheng Jin, and Shijian Lu.
\newblock Vision-language models for vision tasks: A survey.
\newblock \emph{IEEE transactions on pattern analysis and machine intelligence}, 46\penalty0 (8):\penalty0 5625--5644, 2024.

\bibitem[Zhang and Agarwal(2024)]{zhang2024multiclass}
Mingyuan Zhang and Shivani Agarwal.
\newblock Multiclass learning from noisy labels for non-decomposable performance measures.
\newblock In \emph{International Conference on Artificial Intelligence and Statistics}, pages 2170--2178. PMLR, 2024.

\bibitem[Zhang and Sabuncu(2018)]{zhang2018generalized}
Zhilu Zhang and Mert Sabuncu.
\newblock Generalized cross entropy loss for training deep neural networks with noisy labels.
\newblock \emph{Advances in neural information processing systems}, 31, 2018.

\bibitem[Zhao et~al.(2024)Zhao, Wang, Zhang, Han, Deveci, and Parmar]{zhao2024review}
Xia Zhao, Limin Wang, Yufei Zhang, Xuming Han, Muhammet Deveci, and Milan Parmar.
\newblock A review of convolutional neural networks in computer vision.
\newblock \emph{Artificial Intelligence Review}, 57\penalty0 (4):\penalty0 99, 2024.

\bibitem[Zhou and Liu(2006)]{zhou2006training}
Zhi-Hua Zhou and Xu-Ying Liu.
\newblock Training cost-sensitive neural networks with methods addressing the class imbalance problem.
\newblock \emph{IEEE Transactions on knowledge and data engineering}, 18\penalty0 (1):\penalty0 63--77, 2006.

\bibitem[Zhu et~al.(2024)Zhu, Zhang, Gangrade, and Scott]{zhu2024label}
Yilun Zhu, Jianxin Zhang, Aditya Gangrade, and Clay Scott.
\newblock Label noise: Ignorance is bliss.
\newblock \emph{Advances in Neural Information Processing Systems}, 37:\penalty0 116575--116616, 2024.

\end{thebibliography}



\section*{Checklist}

 \begin{enumerate}

 \item For all models and algorithms presented, check if you include:
 \begin{enumerate}
   \item A clear description of the mathematical setting, assumptions, algorithm, and/or model. 
   [Yes] We provide it in Section \ref{subsec:learning_theory}.
   \item An analysis of the properties and complexity (time, space, sample size) of any algorithm. [Yes]
   We provide it in Section \ref{subsec:learning_theory}.
   \item (Optional) Anonymized source code, with specification of all dependencies, including external libraries. [Yes/No/Not Applicable]
 \end{enumerate}

 \item For any theoretical claim, check if you include:
 \begin{enumerate}
   \item Statements of the full set of assumptions of all theoretical results. [Yes] We provide it in Section \ref{subsec:learning_theory}.
   \item Complete proofs of all theoretical results. [Yes] We provide it in Appendix \ref{sec:appendix:proofs}.
   \item Clear explanations of any assumptions. 
   [Yes] We provide it in Section \ref{subsec:learning_theory}.
 \end{enumerate}

 \item For all figures and tables that present empirical results, check if you include:
 \begin{enumerate}
   \item The code, data, and instructions needed to reproduce the main experimental results (either in the supplemental material or as a URL). [Yes] we provide it in the supplemental material.
   \item All the training details (e.g., data splits, hyperparameters, how they were chosen). [Yes]
   we provide it in Appendix \ref{ssup:sec:experiment}  and \ref{ssup:sec:experiment_binary}.
         \item A clear definition of the specific measure or statistics and error bars (e.g., with respect to the random seed after running experiments multiple times). 
         [Yes] we provide it in Appendix \ref{ssup:sec:experiment}  and \ref{ssup:sec:experiment_binary}.
         \item A description of the computing infrastructure used. (e.g., type of GPUs, internal cluster, or cloud provider). [Not Applicable] 
 \end{enumerate}

 \item If you are using existing assets (e.g., code, data, models) or curating/releasing new assets, check if you include:
 \begin{enumerate}
   \item Citations of the creator If your work uses existing assets. [Yes] we cite it in \ref{sec:experiment_results}. 
   \item The license information of the assets, if applicable. [Not Applicable]
   \item New assets either in the supplemental material or as a URL, if applicable. [Yes] we provide it in the supplemental material.
   \item Information about consent from data providers/curators. [Not Applicable]
   \item Discussion of sensible content if applicable, e.g., personally identifiable information or offensive content. [Not Applicable]
 \end{enumerate}

 \item If you used crowdsourcing or conducted research with human subjects, check if you include:
 \begin{enumerate}
   \item The full text of instructions given to participants and screenshots. [Yes/No/Not Applicable]
   \item Descriptions of potential participant risks, with links to Institutional Review Board (IRB) approvals if applicable. [Yes/No/Not Applicable]
   \item The estimated hourly wage paid to participants and the total amount spent on participant compensation. [Yes/No/Not Applicable]
 \end{enumerate}

 \end{enumerate}

\clearpage
\newpage

\onecolumn

\appendix


\section{ Technical Proofs}
\label{sec:appendix:proofs}

\subsection{ Technical Proofs for Proposition \ref{proposition:quantile_gap}}
\label{sec:appendix:proofs_section3}

\begin{propositioninappendix}[Gap between $\tau_\alpha$ and $\widehat \tau_\alpha^s$] 
\label{proposition:quantile_gap_appendix}
Denote by $G(t)$ the cumulative distribution function (CDF) of $M_f(X,\widetilde Y)$ under the noisy distribution $\mathcal{P}_{\noisy}$.
Assume that $G(t)$ is continuously differentiable in a neighborhood of $\tau_\alpha(f)$ with density $g(t)$ and $g(\tau_\alpha(f)) > 0$. 
Then, for any $\delta \in (0,1)$, we have:
\begin{align*}
\P \bigg ( |\tau_{\alpha}(f) - \widehat\tau^s_{\alpha}(f) | \leq \tilde O \Big (\frac{1}{\sqrt{s}} \Big ) \bigg ) \geq 1-\delta,
\end{align*}
where $\tilde O$ hides the logarithmic factors.
\end{propositioninappendix}

\begin{proof}
of Proposition \ref{proposition:quantile_gap}.

Before proving Proposition \ref{proposition:quantile_gap}, we first define $\widehat G^s(t) = \frac{1}{s}\sum \indicator [M_f(X_i,\widetilde Y_i) \leq t]$ as the empirical CDF of $M_f(X,\widetilde Y)$ with $s$ samples.

By the definition, $\tau_\alpha(f)=G^{-1}(\alpha)$, where $\widehat\tau_\alpha^s(f):=\inf\{t:\widehat G^s(t)\ge \alpha\}$.

\textbf{Step 1.}
By the definition of $\widehat\tau_\alpha^s$, we have $\widehat G^s(\widehat\tau_\alpha^s)\ge \alpha$.
Moreover, since $\widehat G^s$ is a step function with jumps of size $1/s$,
the value at the first crossing satisfies
\begin{equation}
\label{eq:empirical_cdf_crossing}
\alpha \le \widehat G^s(\widehat\tau_\alpha^s) \le \alpha+\frac{1}{s}.
\end{equation}
Equivalently, $\widehat\tau_\alpha^s$ is an order statistic and the ECDF at an order statistic lies in an interval of width $1/s$.

\textbf{Step 2.}
By the DKW inequality, for any $\delta\in(0,1)$, with probability at least $1-\delta$,
\begin{equation}
\label{eq:dkw}
\sup_{t\in\mathbb R}\big|\widehat G^s(t)-G(t)\big|
\le
\varepsilon_s(\delta)
:=
\sqrt{\frac{\log(2/\delta)}{2s}}.
\end{equation}
On the event \eqref{eq:dkw}, evaluating at the \emph{same} point $t=\widehat\tau_\alpha^s$ gives
\[
\big|G(\widehat\tau_\alpha^s)-\widehat G^s(\widehat\tau_\alpha^s)\big|\le \varepsilon_s(\delta).
\]
Combining with \eqref{eq:empirical_cdf_crossing}, we obtain the sandwich
\begin{equation}
\label{eq:G_hat_tau_sandwich}
\alpha-\varepsilon_s(\delta)
\le
G(\widehat\tau_\alpha^s)
\le
\alpha+\frac{1}{s}+\varepsilon_s(\delta).
\end{equation}

\textbf{Step 3.}
Let $U=(\tau_\alpha-\eta,\tau_\alpha+\eta)$ and assume $\inf_{t\in U} g(t)\ge g_0>0$.
Then $G$ is strictly increasing on $U$, hence invertible on $G(U)$, and its inverse is $(1/g_0)$-Lipschitz on $G(U)$:
for any $u_1,u_2\in G(U)$,
\begin{equation}
\label{eq:inv_lip}
|G^{-1}(u_1)-G^{-1}(u_2)|\le \frac{1}{g_0}|u_1-u_2|.
\end{equation}

Now we show that on the event \eqref{eq:dkw} and for sufficiently large $s$ (or more explicitly,
whenever $\varepsilon_s(\delta)+1/s \le g_0\eta$), the probability interval in \eqref{eq:G_hat_tau_sandwich}
lies inside $G(U)$, which implies $\widehat\tau_\alpha^s\in U$ and legitimizes applying \eqref{eq:inv_lip}.
Indeed, since $G(\tau_\alpha)=\alpha$ and $g\ge g_0$ on $U$, we have
\[
G(\tau_\alpha+\eta)-\alpha \ge g_0\eta,
\qquad
\alpha-G(\tau_\alpha-\eta) \ge g_0\eta.
\]
Thus if $\varepsilon_s(\delta)+1/s \le g_0\eta$, then
\[
\alpha-\varepsilon_s(\delta)\ge G(\tau_\alpha-\eta),
\qquad
\alpha+\frac1s+\varepsilon_s(\delta)\le G(\tau_\alpha+\eta),
\]
so \eqref{eq:G_hat_tau_sandwich} implies $G(\widehat\tau_\alpha^s)\in G(U)$ and hence $\widehat\tau_\alpha^s\in U$.

\textbf{Step 4.}
On the event \eqref{eq:dkw} and under $\varepsilon_s(\delta)+1/s \le g_0\eta$, we can apply \eqref{eq:inv_lip} with
$u_1=G(\widehat\tau_\alpha^s)$ and $u_2=\alpha=G(\tau_\alpha)$:
\[
|\widehat\tau_\alpha^s-\tau_\alpha|
=
\big|G^{-1}(G(\widehat\tau_\alpha^s)) - G^{-1}(\alpha)\big|
\le
\frac{1}{g_0}\,|G(\widehat\tau_\alpha^s)-\alpha|.
\]
Using \eqref{eq:G_hat_tau_sandwich}, we obtain
\[
|G(\widehat\tau_\alpha^s)-\alpha|\le \varepsilon_s(\delta)+\frac{1}{s},
\]
hence
\[
|\widehat\tau_\alpha^s-\tau_\alpha|
\le
\frac{1}{g_0}\Big(\varepsilon_s(\delta)+\frac{1}{s}\Big)
=
\frac{1}{g_0}\Big(\sqrt{\frac{\log(2/\delta)}{2s}}+\frac{1}{s}\Big).
\]

Finally, since \eqref{eq:dkw} holds with probability at least $1-\delta$, the desired high-probability bound follows.
\end{proof}

\subsection{ Technical Proofs for Theorem \ref{theorem:learning_bound}}
\label{sec:appendix:proofs_section4}

\begin{theoreminappendix}[Learning bound]
\label{theorem:learning_bound_appendix}
(Theorem \ref{theorem:learning_bound} restated.)
Suppose the assumptions in Proposition~\ref{proposition:quantile_gap} hold.
Define $\delta_w$ as the average Wasserstein-1 distance between the noisy and clean label distributions conditional on $X$, where $\delta_w = \mathbb{E}_{X \sim \mathcal{P}(X)} \Big[ W_1\big( \mathcal{P}_{\noisy}(\cdot \mid X),\; \mathcal{P}(\cdot \mid X) \big) \Big]$.
Then the learning bound of CMRM is:
\begin{align*}
\mathcal{L}_{\text{cr}}(f) - \widehat \calL^s_{\text{cr}}(f)
\leq
\tilde O \Big ( \widehat {\mathfrak{R}}_n(\mathcal{F}) + \frac{1}{\sqrt{s}} + \delta_w + \alpha + \texttt{temp} \Big ).
\end{align*}
\end{theoreminappendix}

\begin{proof}
of Theorem \ref{theorem:learning_bound}.

Before proving Theorem \ref{theorem:learning_bound}, we first define conformal margin risk on the noisy distribution as:
\begin{align*}
\calL^{\noisy}_{\text{cr}}(f)
= - \E_{(X,\widetilde Y) \sim \calP_{\noisy}} \Big [ M_f(X,\widetilde Y) \cdot \indicator[M_f(X,\widetilde Y) \ge \tau_\alpha(f)] \Big ].
\end{align*}
and its surrogate risk on the noisy distribution, where the indicator function $\indicator[x\geq y]$ is replaced by the Sigmoid function $\widetilde \indicator[x\geq y] = 1/(1+\exp(-(x-y)/\texttt{temp} ))$, as:
\begin{align*}
\widetilde \calL^{\noisy}_{\text{cr}}(f)
= - \E_{(X,\widetilde Y) \sim \calP_{\noisy}} \Big [ M_f(X,\widetilde Y) \cdot \widetilde \indicator[M_f(X,\widetilde Y) \ge \tau_\alpha(f)] \Big ].
\end{align*}

Recall that:
\begin{align*}
\widehat \calL_{\text{cr}}(f)
= \frac{1}{n}\sum^n_{i=1} 
- M_f(X_i,\widetilde Y_i) \cdot \widetilde \indicator[M_f(X_i,\widetilde Y_i) \ge \widehat \tau_\alpha(f)]
,
\end{align*}

\begin{align*}
\widehat \calL^s_{\text{cr}}(f)
= \mathop{\E}\limits_{\widehat \tau^s_{\alpha} (f) \sim \calT_{\alpha}(f)}\bigg[ \sum^n_{i=1} &  - M_f(X_i,\widetilde Y_i) \cdot 
\widetilde \indicator[M_f(X_i,\widetilde Y_i) \ge \widehat \tau^s_\alpha(f)] \bigg],
\end{align*}

and
\begin{align*}
\mathcal{L}_{\text{cr}}(f)
= - \E_{(X,Y) \sim \calP} \left[ M_f(X, Y) \mid M_f(X,Y) \ge \tau_\alpha(f) \right].
\end{align*}

Then we show the following technical lemmas:
\begin{lemma}
\label{lemma:CMRM_population_gap}
(Immediate results from Theorem 3.3 in \citep{mohri2018foundations})
For any $f \in \calF$, the following inequality holds with high probability:
\begin{align*}
\mathop{\E}\limits_{\widehat \tau^s_{\alpha} (f) \sim \calT_{\alpha}(f)}\bigg[ 
- \E_{(X,\widetilde Y) \sim \calP_{\noisy}} \left[ M_f(X,\widetilde Y) \cdot
\widetilde \indicator \big [M_f(X,\widetilde Y) \ge \widehat \tau^s_\alpha(f) \big ] \right] \bigg]
\leq 
\widetilde \calL^s_{\text{cr}}(f) + 
\tilde O \Big ( \widehat {\mathfrak{R}}_n(\mathcal{F}) + \frac{1}{\sqrt{n}}\Big ).
\end{align*}
\end{lemma}

\begin{lemma}
\label{lemma:CMRM_batch_gap}
Suppose the assumptions in Proposition~\ref{proposition:quantile_gap} hold.
Then, the following inequality holds:
\begin{align*}
\Bigg | \widetilde \calL^{\noisy}_{\text{cr}}(f) - 
\mathop{\E}\limits_{\widehat \tau^s_{\alpha} (f) \sim \calT_{\alpha}(f)}\bigg[ 
- \E_{(X,\widetilde Y) \sim \calP_{\noisy}} \left[ M_f(X,\widetilde Y) \cdot
\widetilde \indicator \big [M_f(X,\widetilde Y) \ge \widehat \tau^s_\alpha(f) \big ] \right] \bigg] \Bigg | 
\leq \tilde O(1 / \sqrt{s} ).
\end{align*}
\end{lemma}

\begin{lemma}
\label{lemma:CMRM_surrogate_gap}
Suppose the assumptions in Proposition~\ref{proposition:quantile_gap} hold.
Then, the following inequality holds:
\begin{align*}
| \calL^{\noisy}_{\text{cr}}(f) - \widetilde \calL^{\noisy}_{\text{cr}}(f) | \leq O(\texttt{temp}).
\end{align*}
\end{lemma}

\begin{lemma}
\label{lemma:CMRM_distribution_gap}
Define $\delta_w$ as the average Wasserstein-1 distance between posteriors under the training and test domains, where $\delta_w = \mathbb{E}_{X \sim \mathcal{P}(X)} \Big[ W_1\big( \P_{\calP_{\noisy}}(\cdot \mid X),\; \P_{\calP}(\cdot \mid X) \big) \Big]$.
Then the following inequality holds:
\[
\big|\calL_{\text{cr}}(f) - \calL^{\noisy}_{\text{cr}}(f) \big| 
\;\leq\;
2\delta_w + \alpha.
\]
\end{lemma}

Now we begin to prove Theorem \ref{theorem:learning_bound}:
\begin{align*}
&
\mathcal{L}_{\text{cr}}(f) - \widehat \calL^s_{\text{cr}}(f)
\\
= &
\underbrace{\mathcal{L}_{\text{cr}}(f) - \calL^{\noisy}_{\text{cr}}(f)}_{2\delta_w + \alpha, \text{ Lemma 4} } 
+
\underbrace{\calL^{\noisy}_{\text{cr}}(f) - \widetilde \calL^{\noisy}_{\text{cr}}(f)}_{O(\texttt{temp}), \text{ Lemma 3} } 
\\
& + \underbrace{\widetilde \calL^{\noisy}_{\text{cr}}(f) 
- \mathop{\E}\limits_{\widehat \tau^s_{\alpha} (f) \sim \calT_{\alpha}(f)}\bigg[ 
- \E_{(X,\widetilde Y) \sim \calP_{\noisy}} \left[ M_f(X,\widetilde Y) \cdot
\widetilde \indicator \big [M_f(X,\widetilde Y) \ge \widehat \tau^s_\alpha(f) \big ] \right] \bigg]}_{\tilde O(1 / \sqrt{s}, \text{ Lemma 2} }
\\
& + \underbrace{\mathop{\E}\limits_{\widehat \tau^s_{\alpha} (f) \sim \calT_{\alpha}(f)}\bigg[ 
- \E_{(X,\widetilde Y) \sim \calP_{\noisy}} \left[ M_f(X,\widetilde Y) \cdot
\widetilde \indicator \big [M_f(X,\widetilde Y) \ge \widehat \tau^s_\alpha(f) \big ] \right] \bigg]
- \widehat \calL^s_{\text{cr}}(f)}_{\tilde O \Big ( \widehat {\mathfrak{R}}_n(\mathcal{F}) + \frac{1}{\sqrt{n}}\Big ), \text{ Lemma 1} }
\\
\leq &
2\delta_w + \alpha + O(\texttt{temp})+ \tilde O(1 / \sqrt{s} ) 
+ \tilde O \Big ( \widehat {\mathfrak{R}}_n(\mathcal{F}) + \frac{1}{\sqrt{n}}\Big )
\\
\leq &
\tilde O \Big ( \widehat {\mathfrak{R}}_n(\mathcal{F}) + \frac{1}{\sqrt{s}}+ \delta_w + \alpha + \texttt{temp}\Big ),
\end{align*}
where the first inequality is due to Lemma \ref{lemma:CMRM_population_gap}, \ref{lemma:CMRM_batch_gap}, 
\ref{lemma:CMRM_surrogate_gap}, \ref{lemma:CMRM_distribution_gap}, and the second inequality is due to $s \leq n$.

Therefore, we have:
\begin{align*}
\mathcal{L}_{\text{cr}}(f) - \widehat \calL^s_{\text{cr}}(f)
\leq
\tilde O \Big ( \widehat {\mathfrak{R}}_n(\mathcal{F}) + \frac{1}{\sqrt{s}}+ \delta_w + \alpha + \texttt{temp}\Big ).
\end{align*}
\end{proof}

\subsection{ Proofs for Technical Lemmas }
\label{sec:appendix:proofs_tech_lemmas}

\subsubsection{ Proofs for Lemma \ref{lemma:CMRM_batch_gap} }

\begin{lemmainappendix}
\label{lemma:CMRM_batch_gap_appendix}
(Lemma \ref{lemma:CMRM_batch_gap} restated.)
Suppose the assumptions in Proposition~\ref{proposition:quantile_gap} hold.
Then, the following inequality holds:
\begin{align*}
\Bigg | \widetilde \calL^{\noisy}_{\text{cr}}(f) - 
\mathop{\E}\limits_{\widehat \tau^s_{\alpha} (f) \sim \calT_{\alpha}(f)}\bigg[ 
- \E_{(X,\widetilde Y) \sim \calP_{\noisy}} \left[ M_f(X,\widetilde Y) \cdot
\widetilde \indicator \big [M_f(X,\widetilde Y) \ge \widehat \tau^s_\alpha(f) \big ] \right] \bigg] \Bigg | 
\leq \tilde O(1 / \sqrt{s} ).
\end{align*}
\end{lemmainappendix}

\begin{proof}
of Lemma \ref{lemma:CMRM_batch_gap}.

Before proving Lemma \ref{lemma:CMRM_batch_gap}, we first show the following technical lemma:

\begin{lemma}
\label{lemma:CMRM_batch_lipchitz}
(Lipschitz continuity of $\widetilde \calL_{\crf}^{\noisy}$ in $\tau$)
Define the smoothed conformal margin risk on the noisy distribution as a function of the threshold $\tau$:
\begin{align*}
\widetilde \calL_{\crf}^{\noisy}(f; \tau)
:= 
- \E_{(X, \widetilde Y) \sim \calP_\noisy}
[ 
M_f(X, \widetilde Y) \cdot \tilde \indicator [ M_f(X, \widetilde Y) \geq \tau ]
]
,
\end{align*}
where the softened indicator is the Sigmoid function
$\tilde \indicator(M_f \geq \tau) := \sigma( \frac{M_f(X, \widetilde Y) - \tau}{\temp} )$,
and 
$\sigma(u) = \frac{1}{1+\exp(-u)}$.
Assume $| M_f(X, \widetilde Y) | \leq 1$ almost surely and $\temp > 0$.
Then, for any fixed $f$, the function $\tau \rightarrow \widetilde \calL_\crf^{\noisy}(f; \tau)$ is $L$-Lipscthiz continuous with $L = 1/(4 \temp)$.
\end{lemma}

The proof of Lemma \ref{lemma:CMRM_batch_lipchitz} is deferred to the end of this proof.

Now we begin to prove Lemma \ref{lemma:CMRM_batch_gap}.
\begin{align*}
&
\Bigg | \widetilde \calL^{\noisy}_{\text{cr}}(f) - 
\mathop{\E}\limits_{\widehat \tau^s_{\alpha} (f) \sim \calT_{\alpha}(f)}\bigg[ 
- \E_{(X,\widetilde Y) \sim \calP_{\noisy}} \left[ M_f(X,\widetilde Y) \cdot
\widetilde \indicator \big [M_f(X,\widetilde Y) \ge \widehat \tau^s_\alpha(f) \big ] \right] \bigg] \Bigg | 
\\
= &
\Bigg | - \E_{(X,\widetilde Y) \sim \calP_{\noisy}} M_f(X,\widetilde Y) \cdot \widetilde \indicator[M_f(X,\widetilde Y) \ge \tau_\alpha(f)]
- 
\E_{(X,\widetilde Y) \sim \calP_{\noisy}} \left[ M_f(X,\widetilde Y) \cdot
\widetilde \indicator \big [M_f(X,\widetilde Y) \ge \widehat \tau^s_\alpha(f) \big ] \right] \bigg]\Bigg | 
\\
\leq &
\frac{1}{ 4 \texttt{temp} } \cdot
\bigg | \tau_\alpha(f) - 
\mathop{\E}\limits_{\widehat \tau^s_{\alpha} (f) \sim \calT_{\alpha}(f)}\Big[ \widehat \tau^s_\alpha(f) \Big] \bigg | 
\\
\leq &
\frac{1}{ 4 \texttt{temp} } \cdot \tilde O(1 / \sqrt{s} )
\\
\leq &
\tilde O(1 / \sqrt{s} ),
\end{align*}
where the first inequality is due to Lemma \ref{lemma:CMRM_batch_lipchitz}, the second inequality is due to Proposition \ref{proposition:quantile_gap}, and the last inequality is due to the setting $\texttt{temp} = 1.0$.

Therefore, we have:
\begin{align*}
\Bigg | \widetilde \calL^{\noisy}_{\text{cr}}(f) - 
\mathop{\E}\limits_{\widehat \tau^s_{\alpha} (f) \sim \calT_{\alpha}(f)}\bigg[ 
- \E_{(X,\widetilde Y) \sim \calP_{\noisy}} \left[ M_f(X,\widetilde Y) \cdot
\widetilde \indicator \big [M_f(X,\widetilde Y) \ge \widehat \tau^s_\alpha(f) \big ] \right] \bigg] \Bigg | 
\leq \tilde O(1 / \sqrt{s} ).
\end{align*}

\end{proof}

Now we begin to prove Lemma \ref{lemma:CMRM_batch_lipchitz}.


\begin{proof}
(of Lemma \ref{lemma:CMRM_batch_lipchitz})

Let $Z := M_f(X, \widetilde Y), u := \frac{Z-\tau}{\temp}$.

We can rewrite the interested function as 
$\widetilde \calL_\crf^{\noisy}(f; \tau) = - \E [Z \cdot \sigma(u) ]$,
where the expectation is taken over $(X, \widetilde Y) \sim \calP_{\noisy}$.

First, we develop the differentiation of $\widetilde \calL_\crf^{\noisy}$ w.r.t. $\tau$ by applying the chain rule as follows
\begin{align}\label{eq:L_noisy_derivative_elaboration}
\frac{ \partial \widetilde \calL_\crf^{\noisy} (f; \tau) }{ \partial \tau } 
=
- \E[ Z \cdot \frac{ \partial \sigma(u) }{ \partial \tau } ]
.
\end{align}
Since $u = ( Z - \tau ) / \temp$, we have
$\frac{\partial u}{\partial \tau} = - 1 / \temp$.

The derivative of the Sigmoid function is 
$$
\frac{ \partial \sigma(u) }{ \partial u }
=
\frac{ \exp(-u) }{ ( 1 + \exp(-u))^2 }
.
$$

Thus, by using the chain rule, we have
\begin{align*}
\frac{ \partial \sigma(u) }{ \partial \tau }
=
\frac{ \partial \sigma(u) }{ \partial u }
\cdot \frac{ \partial u }{ \partial \tau }
=
- \frac{1}{\temp} \cdot \frac{ \exp(-u) }{ ( 1 + \exp(-u) )^2 }
.
\end{align*}

Substituting the above equality back into Equation (\ref{eq:L_noisy_derivative_elaboration}), 
we have 
\begin{align}\label{eq:L_noisy_derivative_elaboration2}
\frac{ \partial \widetilde \calL_\crf^{\noisy} }{ \partial \tau } (f; \tau)
=
- \E [
Z \cdot ( - \frac{1}{\temp} \cdot \frac{ \exp(-u) }{ ( 1 + \exp(-u) )^2 } )
]
=
\frac{ 1 }{ \temp } \cdot \E[ Z \frac{ \exp(-u) }{ ( 1 + \exp(-u) )^2 } ]
.
\end{align}

Then, we would like to bound the derivative.
Taking absolute value of (\ref{eq:L_noisy_derivative_elaboration2}) and using $|Z| \leq 1$, we have
\begin{align*}
| \frac{ \partial \widetilde \calL_\crf^{\noisy} (f; \tau) }{ \partial \tau } |
\leq 
\frac{1}{\temp} \cdot \E[ \frac{ \exp(-u) }{ ( 1 + \exp(-u) )^2 } ]
.
\end{align*}

Recall that $h(u) := \frac{ \exp(-u) }{ ( 1 + \exp(-u) )^2 } = \sigma(u) \cdot (1-\sigma(u))$,
which is the standard logistic hat function.
It satisfies 
\begin{align*}
0 \leq h(u) \leq 1/4, \text{ with } h(0) = 1/4.
\end{align*}

Therefore, we have $\E[h(u)] \leq 1/4$.
Hence, we conclude $| \frac{ \partial \widetilde \calL_\crf^{\noisy} (f; \tau) }{ \partial \tau} | \leq 1 / ( 4 \temp )$.

Finally, to determine the Lipschitz parameter, for $\tau_1, \tau_2$, we have
\begin{align*}
| \widetilde \calL_\crf^{\noisy}(f; \tau_1) - \widetilde \calL_\crf^{\noisy}(f; \tau_2) |
\leq 
\sup_\tau | \frac{ \partial \widetilde \calL_\crf^{\noisy}(f; \tau) }{ \partial \tau } | \cdot | \tau_1 - \tau_2 |
\leq 
\frac{1}{4\temp} | \tau_1 - \tau_2 |
.
\end{align*}
Thus, it shows $\widetilde \calL_\crf^{\noisy}(f; \tau)$ is Lipschitz in $\tau$ with constant $L = 1/(4\temp)$.
\end{proof}

\subsubsection{ Proof for Lemma \ref{lemma:CMRM_surrogate_gap} }

\begin{lemmainappendix}
\label{lemma:CMRM_surrogate_gap_appendix}
(Lemma \ref{lemma:CMRM_surrogate_gap} restated.)
Suppose the assumptions in Proposition~\ref{proposition:quantile_gap} hold.
Then, for any $f$, there exists a constant $C > 0$ independent from $f$ and $\temp$ such that:
\begin{align*}
| \calL^{\noisy}_{\text{cr}}(f) - \widetilde \calL^{\noisy}_{\text{cr}}(f) | 
\leq 
C \temp
.
\end{align*}
In particular, with Big-O notation, we have
\begin{align*}
| \calL^{\noisy}_{\text{cr}}(f) - \widetilde \calL^{\noisy}_{\text{cr}}(f) | \leq O(\texttt{temp}).
\end{align*}
\end{lemmainappendix}

\begin{proof}
(of Lemma \ref{lemma:CMRM_surrogate_gap}.)

Before proving Lemma \ref{lemma:CMRM_surrogate_gap}, we first present the following technical lemma:

\begin{lemma}
\label{lemma:tech_proof_for_lemma_surrogate2}
(Upper bounding the gap between hard and soft indicator functions)
Let 
\begin{align*}
\Gamma(\Delta)
:=
\indicator[ \Delta \geq 0 ] - \tilde \indicator[ \Delta \geq 0 ],
\text{ and }
\tilde \indicator[ \Delta \geq 0 ] = 1 / ( 1 + \exp(-\Delta / \temp) )
,
\end{align*}
where $\temp > 0$ is the temperature parameter.
Then, for every $\Delta \in \R$, the following inequality holds:
\begin{align*}
| \Gamma(\Delta) | 
\leq 
\exp( - | \Delta | / \temp )
.
\end{align*}
\end{lemma}

Now we begin to prove Lemma \ref{lemma:CMRM_surrogate_gap}.

Recall the definitions of the noisy conformal margin risk and its smoothed variant:
\begin{align*}
\calL_\crf^{\noisy}(f)
=
- \E_{(X, \widetilde Y) \sim \calP_{\noisy}} 
[ M_f(X, \widetilde Y) \cdot 
\indicator[ M_f(X, \widetilde Y) \geq \tau_\alpha(f) ] ]
,\\
\widetilde \calL_\crf^{\noisy}(f)
=
- \E_{(X, \widetilde Y) \sim \calP_{\noisy}} 
[ M_f(X, \widetilde Y) \cdot 
\widetilde \indicator[ M_f(X, \widetilde Y) \geq \tau_\alpha(f) ] ]
.
\end{align*}
where the soft smoothed indicator function is $\widetilde \indicator[ M \geq \tau ] = \frac{ 1 }{ 1 + \exp( - ( M - \tau ) / \temp ) }$.

For a fixed $f$, define the margin–threshold difference as 
\begin{align*}
\Delta_f(X, \widetilde Y)
:=
M_f(X, \widetilde Y) - \tau_\alpha(f)
,
\end{align*}
and define the difference between the hard and soft smoothed indicators as
\begin{align*}
\Gamma(\Delta_f(X. \widetilde Y))
:=
\indicator[ \Delta_f(X, \widetilde Y) \geq 0 ]
- \widetilde \indicator[ \Delta_f(X, \widetilde Y) \geq 0 ]
.
\end{align*}

By assumption, $M_f(X, \widetilde Y) \in (-1, 1)$.
Since $\tau_\alpha(f)$ is the $\alpha$-quantile of $M_f(X, \widetilde Y)$ under $\calP_\noisy$,
we also have $\tau_\alpha(f) \in (-1, 1)$.
Hence we have
\begin{align*}
\Delta_f(X, \widetilde Y)
\in (-2, 2) \text{ almost surely.}
\end{align*}

In what follows, we use three steps to prove the desired result.

\paragraph{Step 1: Reducing to bounding $\E[|\Gamma(\Delta_f)|]$.}

We first write the difference between the two risks explicitly as follows:
\begin{align*}
\calL_\crf^{\noisy}(f) - \widetilde \calL_\crf^{\noisy}(f)
= &
- \E_\noisy [ M_f(X, \widetilde Y) \cdot \indicator[ M_f(X, \widetilde Y) \geq \tau_\alpha(f) ] ]
+ \E_\noisy [ M_f(X, \widetilde Y) \cdot \widetilde \indicator[ M_f(X, \widetilde Y) \geq \tau_\alpha(f) ] ]
\\
= &
\E_\noisy [ M_f(X, \widetilde Y) \cdot ( \widetilde \indicator[ M_f(X, \widetilde Y) \geq \tau_\alpha(f) ] - \indicator[ M_f(X, \widetilde Y) \geq \tau_\alpha(f) ] ) ]
\\
= &
-E_{\noisy} [ M_f(X, \widetilde Y) \cdot \Gamma (\Delta_f(X, \widetilde Y) ) ]
.
\end{align*}

Taking the absolute values and using $| M_f(X, \widetilde Y)| \leq B$, we have
\begin{align*}
| \calL_\crf^{\noisy}(f) - \widetilde \calL_\crf^{\noisy}(f) |
= &
| \E_{\noisy} [ M_f(X, \widetilde Y) \cdot \Gamma (\Delta_f(X, \widetilde Y) ) ] |
\leq 
\E_{\noisy} [ | M_f(X, \widetilde Y) | \cdot | \Gamma(\Delta_f(X, \widetilde Y ) ) | ]
\\
\leq &
\E_{\noisy} [ B | \Gamma(\Delta_f(X, \widetilde Y)) | ]
.
\end{align*}
Therefore, it suffices to upper bound $\E_{\noisy}[ | \Gamma (\Delta_f) | ]$.

\paragraph{Step 2: Use the density of $\Delta_f$ and Lemma \ref{lemma:tech_proof_for_lemma_surrogate2}.}

Let $G(t)$ and $g(t) = G'(t)$ denote the CDF and the density of $M_f(X, \widetilde Y)$ under $\calP_{\noisy}$, respectively.
By the assumptions in Proposition \ref{proposition:quantile_gap}, $g$ is continuous and bounded, i.e.,
\begin{align*}
g_\infty 
:= 
\sup_{t \in \R} g(t) 
< \infty
\end{align*}

For fixed $f$, $\tau_\alpha(f)$ is a constant, so the random variable $\Delta_f = M_f - \tau_\alpha(f)$ has density which is also continuous and bounded with $\| g_{\Delta_f} \|_\infty \leq g_\infty$

Therefore, we can write 
\begin{align*}
\E_\noisy[ | \Gamma (\Delta_f) | ]
=
\int_{-2}^{2} | \Gamma(\delta) | \cdot g_{\Delta_f}(\delta) d\delta
,
\end{align*}
where we use the fact that $\Delta_f \in (-2, 2)$ almost surely.

By Lemma \ref{lemma:tech_proof_for_lemma_surrogate2}, for all $\delta \in \R$, we have 
\begin{align*}
| \Gamma(\delta) |
\leq 
\exp( - | \delta | / \temp )
.
\end{align*}

Combining this with the bound on $g_{\Delta_f}$, we derive
\begin{align*}
\E_{\noisy} [ | \Gamma(\Delta_f) | ]
=
\int_{-2}^2 | \Gamma(\delta) | g_{\delta_f}(\delta) d\delta
\leq 
g_\infty \cdot \int_{-2}^2 \exp( - | \delta| / \temp ) d\delta
.
\end{align*}

The integral can be computed explicitly:
\begin{align*}
\int_{-2}^2 \exp( - \frac{|\delta|}{\temp} ) d \delta
=
2 \int_{0}^2 \exp(-\delta / \gamma) d\delta
=
2 \temp ( 1 - \exp(-2 / \temp ))
\leq 2 \temp
,
\end{align*}
where we used $1 - \exp(2/\temp) \leq 1$.

Thus, we have $\E_{\noisy}[ | \Gamma(\Delta_f) | ] \leq 2 g_\infty \temp$.

\paragraph{Step 3: Conclude the upper bound.}

Putting everything together, we derive
\begin{align*}
||
\leq 
\E_{\noisy} [ | \Gamma(\Delta_f) | ]
\leq 
2 g_\infty \temp
.
\end{align*}

Hence the difference between the hard formal margin risk and its smoothed counter is $O(\temp)$.

Recall that the $M_f(X,\widetilde Y)$ is the confidence gap between the confidence score of the observed label and the highest confidence among other candidate labels.
Thus, $M_f(X,\widetilde Y) \in (-1, 1)$.
As $\tau_\alpha(f)$ is the $\alpha$-quantile of $M_f(X,\widetilde Y)$, $\tau_\alpha(f)\in (-1, 1)$ and $\Delta_f(X,Y) \in (-2,2)$.

We also recall that $G(t)$ is the CDF of $M_f(X,\widetilde Y)$ under $\mathcal{P}_{\noisy}$, and $g(t) = G'(t)$ is its density, which is assumed to be continuous and bounded.
Denote $g_{\infty} = \sup_t g(t) < +\infty$.
Since $\tau_\alpha(f)$ is a constant for fixed $f$, the density of $\Delta_f$ is $g_{\Delta_f}(t) = g(t+\tau_\alpha(f))$, which satisfies $\| g_{\Delta_f} \|_\infty = g_{\infty}$.

Then, we have:
\begin{align*}
&
| \calL^{\noisy}_{\text{cr}}(f) - \widetilde \calL^{\noisy}_{\text{cr}}(f) |
\\
= &
\Big | \E_{(X,\widetilde Y) \sim \calP_{\noisy}} \big [ M_f(X,\widetilde Y) \cdot \widetilde \indicator[M_f(X,\widetilde Y) \ge \tau_\alpha(f)] \big ] - \E_{(X,\widetilde Y) \sim \calP_{\noisy}} \big [ M_f(X,\widetilde Y) \cdot \indicator[M_f(X,\widetilde Y) \ge \tau_\alpha(f)] \big ] \Big | 
\\
= &
\bigg | \E_{(X,\widetilde Y) \sim \calP_{\noisy}} \Big [ M_f(X,\widetilde Y) \cdot \big (\widetilde \indicator[M_f(X,\widetilde Y) \ge \tau_\alpha(f)] - \indicator[M_f(X,\widetilde Y) \ge \tau_\alpha(f)] \big) \Big ] \bigg | 
\\
\leq &
| \E_{(X,\widetilde Y) \sim \calP_{\noisy}} \big [ M_f(X,\widetilde Y) \big ] | \cdot
\Big | \E_{(X,\widetilde Y) \sim \calP_{\noisy}} \Big [ \big (\widetilde \indicator[M_f(X,\widetilde Y) \ge \tau_\alpha(f)] - \indicator[M_f(X,\widetilde Y) \ge \tau_\alpha(f)] \big ) \Big ] \Big | 
\\
\leq &
\Big | \E_{(X,\widetilde Y) \sim \calP_{\noisy}} \Big [ \big (\widetilde \indicator[M_f(X,\widetilde Y) \ge \tau_\alpha(f)] - \indicator[M_f(X,\widetilde Y) \ge \tau_\alpha(f)] \big ) \Big ] \Big |
\\
\leq &
\E_{(X,\widetilde Y) \sim \calP_{\noisy}} 
\Big | \widetilde \indicator[M_f(X,\widetilde Y) \ge \tau_\alpha(f)] - \indicator[M_f(X,\widetilde Y) \ge \tau_\alpha(f)] \Big |
\\
= &
\E_{(X,\widetilde Y) \sim \calP_{\noisy}} 
\Big | \widetilde \indicator[\Delta_f(X,Y) \ge 0] - \indicator[\Delta_f(X,Y) \ge 0] \Big |
\\
= &
\E [ \Gamma(\Delta_f(X,Y))]
\\
= &
\int^2_{-2} \Gamma(\Delta_f(X,Y)) \partial_{\Delta_f(X,Y)} \Gamma(\Delta_f(X,Y)) d\Delta_f(X,Y)
\\
\leq &
\int^2_{-2} \exp \Big (- \frac{|\Delta_f(X,Y)|}{\texttt{temp}} \Big ) \partial_{\Delta_f(X,Y)} \Gamma(\Delta_f(X,Y)) d\Delta_f(X,Y)
\\
\leq &
g_{\infty} \cdot 
\int^2_{-2} \exp \Big (- \frac{|\Delta_f(X,Y)|}{\texttt{temp}} \Big ) d\Delta_f(X,Y)
\\
= &
g_{\infty} \cdot 2 \texttt{temp} (1 - \exp(-2/\texttt{temp}))
\\
\leq &
O(\texttt{temp}),
\end{align*}
where the first inequality is due to the Hölder's inequality, the second inequality is due to the Jansen's inequality and expectation is convex function, the third inequality is due to $M_f(X,\widetilde Y) \in (-1, 1)$, the fourth inequality is due to Lemma \ref{lemma:tech_proof_for_lemma_surrogate2}, the fifth inequality is due to the upper bound for the gradient of $\Delta_f(X,Y)$ is $g_{\infty}$, and the last inequality is due to $1 - \exp(-2/\texttt{temp}) \leq 1$.

Therefore, we have:
\begin{align*}
| \calL^{\noisy}_{\text{cr}}(f) - \widetilde \calL^{\noisy}_{\text{cr}}(f) | \leq O(\texttt{temp}).
\end{align*}
\end{proof}

Now we begin to prove Lemma \ref{lemma:tech_proof_for_lemma_surrogate2}.

\begin{proof}
(of Lemma \ref{lemma:tech_proof_for_lemma_surrogate2})

We analyze the two cases, i.e., $\Delta \geq 0$ and $\Delta < 0$ separately below.

\textbf{Case 1:} $\Delta \geq 0$

In this case, by using the condition $\Delta \geq 0$, we first reformulate $\Gamma(\Delta)$ as follows
\begin{align*}
\Gamma(\Delta)
=
1 - \frac{ 1 }{ 1 + \exp(-\Delta / \temp) }
=
\frac{ \exp( - \Delta / \temp ) }{ 1 + \exp( - \Delta / \temp ) }
.
\end{align*}

Since the denominator is at least $1$ (due to $\exp(\cdot) > 0$),
we have the following inequalities
\begin{align}\label{eq:bounding_Gamma_case1}
0 
< 
\Gamma(\Delta)
=
| \Gamma(\Delta) |
\leq 
\exp( - \Delta / \temp )
=
\exp( - | \Delta | / \temp )
,
\end{align}
where the last equality is due to $\Delta \geq 0$, the condition for this {\bf case 1}.
Thus, the desired inequality holds for all $\Delta > 0$.

\textbf{Case 2:} $\Delta < 0$

In this case, by using the condition $\Delta < 0$, we reformulate $\Gamma(\Delta)$ as follows
\begin{align*}
\Gamma(\Delta)
=
0 - \frac{ 1 }{ 1 + \exp( - \Delta / \temp ) }
=
- \frac{ 1 }{ 1 + \exp( - \Delta / \temp ) }
< 0
.
\end{align*}

Then, with the condition $\Delta < 0$, we have $| \Delta | = - \Delta$.
To upper bound $|\Gamma(\Delta)|$ in this case, it suffixes to show:
\begin{align}\label{eq:bounding_Gamma_case2}
| \Gamma(\Delta) |
=
- \Gamma(\Delta)
=
\frac{ 1 }{ 1 + \exp( -\Delta / \temp ) }
\leq 
\exp( \Delta/\temp )
=
\exp( - | \Delta | / \temp )
,
\end{align}
where the last inequality is due to $\exp(a) \leq 1 + \exp(a)$ for any $a \in \R$.

Finally, combining (\ref{eq:bounding_Gamma_case1}) and (\ref{eq:bounding_Gamma_case2}) implies that for any $\Delta \in \R$, the following inequality holds:
\begin{align*}
| \Gamma(\Delta) | 
\leq 
\exp( - |\Delta| / \temp )
.
\end{align*}
\end{proof}

\subsubsection{ Proof for Lemma \ref{lemma:CMRM_distribution_gap}}

\begin{lemmainappendix}
\label{lemma:CMRM_population_gap_appendix}
(Lemma \ref{lemma:CMRM_distribution_gap} restated.)
Define $\delta_w$ as the average Wasserstein-1 distance between posteriors under the training and test domains, where $\delta_w = \mathbb{E}_{X \sim \mathcal{P}(X)} \Big[ W_1\big( \P_{\calP_{\noisy}}(\cdot \mid X),\; \P_{\calP}(\cdot \mid X) \big) \Big]$.
Then the following inequality holds:
\[
\big|\calL_{\text{cr}}(f) - \calL^{\noisy}_{\text{cr}}(f) \big| 
\;\leq\;
2\delta_w + \alpha.
\]
\end{lemmainappendix}

\begin{proof}
of Lemma \ref{lemma:CMRM_distribution_gap}.

Before proving Lemma \ref{lemma:CMRM_distribution_gap}, we first show the following technical lemmas:

\begin{lemma}[Range of Probability]
\label{lemma:prob_upper_bound}
Define $\delta_w$ as the average Wasserstein-1 distance between posteriors under the training and test domains, where $\delta_w = \mathbb{E}_{X \sim \mathcal{P}(X)} \Big[ W_1\big( \P_{\calP_{\noisy}}(\cdot \mid X),\; \P_{\calP}(\cdot \mid X) \big) \Big]$. 
Then the probability $\P_{\calP} \big (M_f(X,y) \geq \tau_{\alpha}(f) \big ) \in [1-\alpha - \delta_w, 1-\alpha + \delta_w]$.
\end{lemma}

\begin{lemma}[ Wasserstein-1 distance bound of expected difference \citep{dudley2018real,arjovsky2017wasserstein}]
\label{lemma:kr_duality}
For any bounded measurable function $g: X \rightarrow \R$, we have: 
\begin{align*}
\Big \vert \E_{\calP_{\noisy}} [g] - \E_{\calP} [g]\Big \vert 
\leq 2 \| g\|_{\infty} W_1\big( \P_{\calP_{\noisy}}, \P_{\calP}\big).
\end{align*}
\end{lemma}

The proof of Lemma \ref{lemma:prob_upper_bound} is deferred to the end of this proof. 
Now we begin to prove Lemma \ref{lemma:CMRM_distribution_gap}.

Recall that $M_f(X,\widetilde Y) $ as the gap of confidence scores, where $| M_f(X,y)| \leq 1$ for any $(X, \widetilde Y)$.
We also recall that $\mathcal{L}_{\text{cr}}(f)
= - \E_{X \sim \calP} \left[ M_f(X,Y) \mid M_f(X,Y) \ge \tau_\alpha(f) \right]$.

Define $P = \P_{\calP} \big (M_f(X,Y) \geq \tau_{\alpha}(f) \big )$.
Thus, we have:
\begin{align}
\label{eq:proof_theo2_1}
&
\big| \calL_{\text{cr}}(f) - \calL^{\noisy}_{\text{cr}}(f)  \big| 
\nonumber \\
= &
\big| \E_{\calP} \left[ M_f(X,Y) \mid M_f(X,Y) \ge \tau_\alpha(f) \right] - \E_{\calP_{\noisy}} \left[ M_f(X,\widetilde Y) \cdot \indicator[M_f(X,\widetilde Y) \ge \tau_\alpha(f)] \right] \big| 
\nonumber \\
= &
\Bigg| \frac{\E_{\calP} \Big [ M_f(X,Y)  \indicator \big[ M_f(X,Y) \ge \tau_\alpha(f) \big] \Big ]}{\P_{\calP} \big (M_f(X,Y) \geq \tau_{\alpha}(f) \big )} - \E_{\calP_{\noisy}} \left[ M_f(X,\widetilde Y) \cdot \indicator[M_f(X,\widetilde Y) \ge \tau_\alpha(f)] \right] \Bigg|
\nonumber \\
= &
\bigg| \frac{1}{P}\E_{\calP} \Big [ M_f(X,Y) \indicator \big [ M_f(X,Y) \ge \tau_\alpha(f) \big ]  \Big] - \E_{\calP_{\noisy}} \Big[ M_f(X,\widetilde Y) \indicator \big [ M_f(X,\widetilde Y) \ge \tau_\alpha(f) \big ] \Big] \bigg| 
\nonumber \\
= &
\bigg| \Big (\frac{1}{P} - 1 \Big )\E_{\calP} \Big [ M_f(X,Y) \indicator \big [ M_f(X,Y) \ge \tau_\alpha(f) \big ]  \Big] + \E_{\calP} \Big [ M_f(X, Y) \indicator \big [ M_f(X,Y) \ge \tau_\alpha(f) \big ]  \Big] 
\nonumber \\
& 
- \E_{\calP_{\noisy}} \Big[ M_f(X,\widetilde Y) \indicator \big [ M_f(X,\widetilde Y) \ge \tau_\alpha(f) \big ] \Big] \bigg| 
\nonumber \\
\leq &
\underbrace{\bigg| \Big (\frac{1}{P} - 1 \Big )\E_{\calP} \Big [ M_f(X,Y) \indicator \big [ M_f(X,Y) \ge \tau_\alpha(f) \big ]  \Big]  \bigg|}_{A} + 
\nonumber \\
&
\underbrace{ \bigg|  \E_{\calP} \Big [ M_f(X,Y) \indicator \big [ M_f(X,Y) \ge \tau_\alpha(f) \big ]  \Big] - 
\E_{\calP_{\noisy}} \Big[ M_f(X,\widetilde Y) \indicator \big [ M_f(X,\widetilde Y) \ge \tau_\alpha(f) \big ] \Big] \bigg|}_{B} 
,
\end{align}
where the first inequality is due to the triangle inequality.

Then we analyze the upper bound of $A$ and $B$, respectively.

{\bf Part I: Upper bound of $A$} 

\begin{align}
\label{eq:proof_theo2_2}
&
\bigg| \Big (\frac{1}{P} - 1 \Big )\E_{\calP} \Big [ M_f(X, Y) \indicator \big [ M_f(X, Y) \ge \tau_\alpha(f) \big ]  \Big]  \bigg|
\nonumber \\
\leq &
\bigg| \Big (\frac{1}{P} - 1 \Big ) \bigg| \E_{\calP} \bigg| \Big [ M_f(X, Y) \indicator \big [ M_f(X, Y) \ge \tau_\alpha(f) \big ]  \Big]  \bigg|
\nonumber \\
= &
\bigg| \Big (\frac{1}{P} - 1 \Big ) \bigg| \E_{\calP} \Big| \big [ M_f(X, Y) \big] \Big| \indicator \big [ M_f(X, Y) \ge \tau_\alpha(f) \big ] 
\nonumber \\
\leq &
\bigg| \Big (\frac{1}{P} - 1 \Big ) \bigg| \E_{\calP} \indicator \big [ M_f(X, Y) \ge \tau_\alpha(f) \big ] 
\nonumber \\
= &
\bigg| \Big (\frac{1}{P} - 1 \Big ) \bigg| P
\nonumber \\
= &
\frac{\bigg| 1 - P \bigg| }{P} P
\nonumber \\
\leq &
\alpha+\delta_w, 
\end{align}
where the first inequality is due to the triangle inequality, the second inequality is due to $| M_f(X, Y)| \leq 1$, and the last inequality is due to the Lemma \ref{lemma:prob_upper_bound}.

{\bf Part II: Upper bound of $B$} 

For any fixed $X$, we define $\phi_X(Y) = M_f(X,Y) \indicator \big [ M_f(X, Y) \ge \tau_\alpha(f) \big ] $. 
Then, we try to bound $\| \phi_X(y) \|_{\infty}$ as:
\begin{align*}
\| \phi_X(Y) \|_{\infty}
=
\| M_f(X, Y) \indicator \big [ M_f(X, Y) \ge \tau_\alpha(f) \|_{\infty}
\leq 
\| M_f(X, Y) \|_{\infty}
\leq 
1
, 
\end{align*}
where the first inequality is due to $\indicator[\cdot] \in [0,1]$, and the second inequality is due to $|M_f(X, Y)| \leq 1$. 

Then, we rewrite $B$ as:
\begin{align}
\label{eq:proof_theo2_3}
&
\bigg|  \E_{\calP} \Big [ M_f(X, Y) \indicator \big [ M_f(X, Y) \ge \tau_\alpha(f) \big ]  \Big] - 
\E_{\calP_{\noisy}} \Big[ M_f(X, \widetilde Y) \indicator \big [ M_f(X,\widetilde Y) \ge \tau_\alpha(f) \big ] \Big] \bigg|
\nonumber \\
= &
\Bigg| \E_X \bigg [ \E_{y \sim \P(\cdot|X) } \Big [ M_f(X, Y) \indicator \big [ M_f(X, Y) \ge \tau_\alpha(f) \big ]  \Big] - 
E_{y \sim \P_{\noisy}(\cdot|X) } \Big[ M_f(X,\widetilde Y) \indicator \big [ M_f(X,\widetilde Y) \ge \tau_\alpha(f) \big ] \Big] \bigg ] \Bigg|
\nonumber \\
= &
\Bigg| \E_X \bigg [ \E_{y \sim \P(\cdot|X) } \Big [ \phi_X(Y) \Big] - 
E_{y \sim \P_{\noisy}(\cdot|X) } \Big[ \phi_X(\widetilde Y) \Big] \bigg ] \Bigg|
\nonumber \\
\leq &
\Bigg| \E_X \bigg [ 2 W_1\big( \P_{\calP_{\noisy}}(\cdot \mid X),\; \P_{\calP}(\cdot \mid X) \big)  \bigg ] \Bigg|
\nonumber \\
= &
2 \Bigg| \E_X \bigg [ W_1\big( \P_{\calP_{\noisy}}(\cdot \mid X),\; \P_{\calP}(\cdot \mid X) \big)  \bigg ] \Bigg|
\nonumber \\
= &
2\delta_w,
\end{align}
where the first inequality is due to Lemma \ref{lemma:kr_duality}.

Combining inequalities (\ref{eq:proof_theo2_1}),(\ref{eq:proof_theo2_2}), and (\ref{eq:proof_theo2_3}), we have that:
\begin{align*}
\big| \calL_{\text{cr}}(f) - \calL^{\noisy}_{\text{cr}}(f) \big| 
\leq 
2\delta_w + \alpha.
\end{align*}

\end{proof}

\begin{proof}
(of Lemma \ref{lemma:prob_upper_bound})

Recall that $\mathcal{Y}$ denotes the label space equipped with the $0$–$1$ distance
$d(y_1,y_2) = \mathbb{I}\{y_1 \neq y_2\}$.
For each fixed $X$, define $\Phi_X(Y) = \indicator \big ( M_f(X, Y) \geq \tau_{\alpha}(f) \big ) $.
Under this metric, $\Phi_X(Y)$ is $1$-Lipschitz because 
$|\Phi_X(y_1) - \Phi_X(y_2)| \le d(y_1,y_2)$.

Denote by 
$\mu_X = \P_{\mathcal{P}_{\noisy}}(\cdot \mid X)$ 
and 
$\nu_X = \P_{\mathcal{P}}(\cdot \mid X)$ 
the label posteriors under the noisy and clean domains, respectively.
Then, we have:
\begin{align*}
&
\Big|
\P_{\mathcal{P}_{\noisy}}\!\big(M_f(X,\widetilde Y)\ge\tau_{\alpha}(f)\big)
-
\P_{\mathcal{P}}\!\big(M_f(X,Y)\ge\tau_{\alpha}(f)\big)
\Big|
\\
=&
\Big|
\E_{X \sim \mathcal{P}(X)} 
\Big[
\E_{\widetilde Y \sim \mu_X}\Phi_X(\widetilde Y)
-
\E_{Y \sim \nu_X}\Phi_X(Y)
\Big]
\Big|
\\
\le &
\E_{X \sim \mathcal{P}(X)} 
\Big[
\Big|
\E_{\widetilde Y \sim \mu_X}\Phi_X(\widetilde Y)
-
\E_{Y \sim \nu_X}\Phi_X(Y)
\Big|
\Big]
\\
= &
\E_{X \sim \mathcal{P}(X)} 
\Big[
\Big|
\sum_{y \in \mathcal{Y}} 
\Phi_X(y)
\big(
\mu_X(y) - \nu_X(y)
\big)
\Big|
\Big]
\\
\leq &
\E_{X \sim \calP_X} \Big[ W_1\big( \mu_X,\; \nu_X \big) \Big]
\\
= &
\E_{X \sim \calP_X} \Big[ W_1\big( \P_{\calP_{\noisy}}(\cdot \mid X),\; \P_{\calP}(\cdot \mid X) \big)  \Big]
=
\delta_w
,
\end{align*}
where the first inequality is due to Jensen’s inequality, and the last inequality is due to Kantorovich–Rubinstein duality for Wasserstein-1.

Due to $\P_{\calP_{\noisy}} \big (M_f(X, \widetilde Y) \geq \tau_{\alpha}(f) \big ) = 1-\alpha$, we have: $\Big|\P_{\calP} \big (M_f(X, Y) \geq \tau_{\alpha}(f) \big ) - (1-\alpha) \Big| = \delta_w$.
\end{proof}

\section{Additional Experiments for Multi-class Classification}
\label{ssup:sec:experiment}

\subsection{Additional Experimental Setup Details}
\label{ssup:sec:multi_experiment_setup}

\textbf{Datasets.}
To evaluate model robustness under noisy supervision, we construct asymmetric label noise on both CIFAR-100 and mini-ImageNet, simulating realistic mislabeling patterns.

\textbf{CIFAR-100} consists of $100$ fine-grained classes organized into $20$ coarse-grained superclasses (e.g., aquatic mammals, large carnivores). 
To inject structured label noise, we first build a mapping from each fine label to its corresponding coarse superclass. 
For noise rates $\rho \in \{0\%, 5\%, 10\%, 20\%, 30\%, 40\% \}$, we randomly select $\rho n$ training samples, where $n$ is the dataset size. 
Each selected label is replaced by a randomly chosen different label from the same superclass, ensuring that the corrupted label remains semantically similar to the original. 
The indices of noisy samples are recorded to facilitate evaluation and ablation.

\textbf{mini-ImageNet} is a $100$-class image classification dataset. To inject asymmetric noise, we randomly select a fraction $\rho = 20 \%$ of training samples and replace each label $Y$ with $(Y+1) \mod 100$, introducing a deterministic and minimal perturbation. 

\textbf{FOOD101} is a $101$-class image classification dataset. To inject asymmetric noise, we randomly select a fraction $\rho = 20 \%$ of training samples and replace each label $Y$ with $(Y+1) \mod 100$, introducing a deterministic and minimal perturbation.

This form of circular asymmetric noise preserves the class index structure and ensures the new label is different from the original. The transformation is applied only to training labels, leaving the validation and test sets clean for evaluation.

\textbf{Hyperparameters for training.}
We set datasets, base loss, batch size, training epochs, training parameters (learning rate, learning schedule, momentum, gamma, and weight decay), $\lambda$ and $\alpha$ as hyperparameter choices. 
We search for hyperparameters on batch size $\in \{64, 128, 256, 512\}$, epochs $\in \{50, 100\}$, learning rate ($\eta$) $\in \{0.1, 0.07, 0.05, 0.03, 0.01\}$, learning rate schedule $\in \{[10], [10, 30], [60], [60, 80]\}$, Momentum $=0.9$, weight decay $= 0.0002$, $\gamma = 0.01$, $\lambda = \{0.05, 0.1, 0.15, 0.2, 0.25\}$, and $\alpha = \{0.05, 0.1, 0.15, 0.2, 0.25\}$ to select the best combination of hyperparameters of each methods.
For GCE, we additionally scale $\lambda$ by a factor of $0.1$, resulting in $\lambda \in \{0.005, 0.01, 0.015, 0.02, 0.025\}$.
The hyerparameters employed to get the results presented in the main paper are summarized in Table \ref{tab:finetune_hyper_params_multi}.

\begin{table}[!ht]
\centering
\resizebox{\textwidth}{!}{
\begin{tabular}{|c|c|c|c|c|c|c|c|c|c|c|}
\hline
Data  & Loss & Batch size & Epochs &  $\eta$  & lr schedule & Momentum & $\gamma$ & weight decay & $\lambda$ & $\alpha$ \\ \hline
\multirow{4}{*}{CIFAR-100}    & CE+CMRM    & 128   & 50     & 0.05   & [10]    & 0.9      & 0.01   & 0.0002    & 0.1 & 0.15   \\ \cline{2-11} 
 & Focal+CMRM  & 128   & 50     & 0.05   & [10]    & 0.9      & 0.01   & 0.0002    & 0.15 & 0.1   \\ \cline{2-11} 
& LDAM+CMRM & 128   & 100     & 0.05   & [60, 80]    & 0.9      & 0.01   & 0.0002    & 0.1 & 0.15   \\ \cline{2-11}
& GCE+CMRM & 128   & 50     & 0.05   & [10]    & 0.9      & 0.01   & 0.0002    & 0.005 & 0.05   \\ \hline 
\multirow{4}{*}{mini-ImageNet}    & CE+CMRM    & 512   & 50     & 0.05   & [10]    & 0.9      & 0.01   & 0.0002    & 0.15 & 0.2   \\ \cline{2-11} 
 & Focal+CMRM  & 512   & 50     & 0.05   & [10]    & 0.9      & 0.01   & 0.0002    & 0.15 & 0.15   \\ \cline{2-11} 
& LDAM+CMRM & 512   & 100     & 0.05   & [60, 80]    & 0.9      & 0.01   & 0.0002    & 0.2 & 0.1  \\ \cline{2-11}
& GCE+CMRM & 512   & 50     & 0.05   & [10]    & 0.9      & 0.01   & 0.0002    & 0.0005 & 0.15   \\ \hline
\multirow{4}{*}{FOOD101}    & CE+CMRM    & 512   & 50     & 0.05   & [10]    & 0.9      & 0.01   & 0.0002    & 0.15 & 0.15   \\ \cline{2-11} 
 & Focal+CMRM  & 512   & 50     & 0.05   & [10]    & 0.9      & 0.01   & 0.0002    & 0.15 & 0.1   \\ \cline{2-11} 
& LDAM+CMRM & 512   & 100     & 0.05   & [60, 80]    & 0.9      & 0.01   & 0.0002    & 0.05 & 0.15  \\ \cline{2-11}
& GCE+CMRM & 512   & 50     & 0.05   & [10]    & 0.9      & 0.01   & 0.0002    & 0.0001 & 0.1   \\ \hline
\end{tabular}
}
\caption{\textbf{The details we used to train our models for multi-class classification corrupted by synthetic noise} with noise rate $20\%$. We reported the hyperparameters that give the best accuracy. We employed SGD optimizer for all training unless specified.}
\label{tab:finetune_hyper_params_multi}
\end{table}

\subsection{Additional Experimental Results}
\label{ssup:sec:bianry_experiment_setup}

\begin{table*}[!ht]
  \centering
  \resizebox{\textwidth}{!}{%
  \begin{NiceTabular}{@{}ll*{8}{c}@{}}
    \toprule
    \multirow{2}{*}{\textbf{Noise Rates}} & \multirow{2}{*}{\textbf{Metric}} &
    \multicolumn{2}{c}{\textbf{CE}} &
    \multicolumn{2}{c}{\textbf{Focal}} &
    \multicolumn{2}{c}{\textbf{LDAM}} &
    \multicolumn{2}{c}{\textbf{GCE}} \\
    \cmidrule(lr){3-4}\cmidrule(lr){5-6}\cmidrule(lr){7-8}\cmidrule(lr){9-10}
    \multicolumn{2}{c}{} & Base & +CMRM & Base & +CMRM & Base & +CMRM & Base & +CMRM \\
    \midrule
    \multirow{2}{*}{$0\%$}
      & ACC (\%) $\uparrow$ & 68.82 & \textbf{69.38 ($+0.56$)} & 68.32 & \textbf{69.27 ($+0.95$)} & 62.94 & \textbf{64.15 ($+1.19$)} & \textbf{62.80} & \textbf{62.80 ($+0$)} \\
      & M.APSS $\downarrow$ & \textbf{3.22} & 3.37 ($+4.52\%$) & \textbf{3.04} & 3.26 ($+7.24\%$) & \textbf{10.89} & 11.06 ($+1.56\%$) & \textbf{5.95} & \textbf{5.95 ($+0\%$)} \\
    \midrule
    \multirow{2}{*}{$5\%$}
      & ACC (\%) $\uparrow$ & 67.48 & \textbf{68.02 ($+0.54$)} & 66.76 & \textbf{67.66 ($+0.90$)} & 61.63 & \textbf{63.00 ($+1.37$)} & 63.88 & \textbf{65.29 ($+1.41$)} \\
      & M.APSS $\downarrow$ & \textbf{3.79} & 4.13 ($+8.97\%$) & \textbf{3.66} & 3.82 ($+4.37\%$) & \textbf{13.03} & 13.28 ($+1.92\%$) & \textbf{6.27} & 6.30 ($+0.48\%$) \\
    \midrule
    \multirow{2}{*}{$10\%$}
      & ACC (\%) $\uparrow$ & 66.29 & \textbf{67.32 ($+1.03$)} & 66.52 & \textbf{66.82 ($+0.30$)} & 61.06 & \textbf{62.28 ($+1.22$)} & 62.80 & \textbf{64.73 ($+1.93$)} \\
      & M.APSS $\downarrow$ & \textbf{4.31} & 4.57 ($+6.03\%$) & \textbf{4.28} & 4.53 ($+5.84\%$) & 15.57 & \textbf{15.32 ($-1.61\%$)} & 6.55 & \textbf{5.58 ($-14.81\%$)} \\
    \midrule
    \multirow{2}{*}{$20\%$}
      & ACC (\%) $\uparrow$ & 65.16 & \textbf{66.32 ($+1.16$)} & 64.42 & \textbf{65.39 ($+0.97$)} & 59.63 & \textbf{61.12 ($+1.49$)} & 62.17 & \textbf{63.65 ($+1.48$)} \\
      & M.APSS $\downarrow$ & 6.67 & \textbf{6.52 ($-2.25\%$)} & 6.89 & \textbf{6.61 ($-4.06\%$)} & 17.85 & \textbf{17.67 ($-0.78\%$)} & 7.70 & \textbf{6.28 ($-18.44\%$)} \\
    \midrule
    \multirow{2}{*}{$30\%$}
      & ACC (\%) $\uparrow$ & 64.02 & \textbf{65.22 ($+1.2$)} & 63.46 & \textbf{65.33 ($+1.87$)} & 58.27 & \textbf{59.56 ($+1.29$)} & 61.49 & \textbf{63.23 ($+1.74$)} \\
      & M.APSS $\downarrow$ & 8.81  & \textbf{7.88 ($-10.56\%$)} & 9.12 & \textbf{7.66 ($-16.01\%$)} & 20.44 & \textbf{18.63 ($-8.86\%$)} & 6.77 & \textbf{6.18 ($-8.71\%$)} \\
    \midrule
    \multirow{2}{*}{$40\%$}
      & ACC (\%) $\uparrow$ & 62.63 & \textbf{64.77 ($+2.14$)} & 62.65 & \textbf{64.87 ($+2.22$)} & 58.25 & \textbf{59.27 ($+1.02$)} & 59.34 & \textbf{60.85 ($+1.51$)} \\
      & M.APSS $\downarrow$ & 10.31 & \textbf{8.99 ($-12.80\%$)} & 10.55  & \textbf{9.12 ($-13.55\%$)} & \textbf{19.54} & 20.95 ($+7.22\%$) & 7.76 & \textbf{7.74 ($-0.26\%$)} \\
    \bottomrule
  \end{NiceTabular}%
  }
  \caption{\textbf{Top-1 accuracy (\%) and marginal average prediction set size (M.APSS $\downarrow$) on CIFAR-100 datasets corrupted by synthetic noise with noise rates $\{0\%, 5\%, 10\%, 20\%, 30\%, 40\% \}$.}
  Each Base objective is paired with its +CMRM counterpart; the better value within each pair is in \textbf{bold}.
  Numbers in parentheses indicate the relative change (\%): $+$ denotes accuracy improvement, and $-$ denotes M.APSS reduction compared to the corresponding Base objective. 
  On average across all datasets and objectives, 
  CMRM improves accuracy by $1.34$ and reduces M.APSS by $3.89\%$.}
  \label{tab:results_multi_noise_rates}
\end{table*}

\begin{table*}[t]
\centering
\resizebox{\textwidth}{!}{
\begin{NiceTabular}{@{}lcccccccc@{}}
\toprule
\multirow{2}{*}{\textbf{Metric}} &
\multicolumn{2}{c}{\textbf{CE}} &
\multicolumn{2}{c}{\textbf{Focal}} &
\multicolumn{2}{c}{\textbf{LDAM}} &
\multicolumn{2}{c}{\textbf{GCE}} \\
\cmidrule(lr){2-3}\cmidrule(lr){4-5}\cmidrule(lr){6-7}\cmidrule(lr){8-9}
& Base & +CMRM & Base & +CMRM & Base & +CMRM & Base & +CMRM \\
\midrule
ACC (\%) $\uparrow$
& 64.74 $\pm$ 0.50 & \textbf{65.95 $\pm$ 0.30 (+1.21)}
& 64.33 $\pm$ 0.15 & \textbf{65.26 $\pm$ 0.21 (+0.93)}
& 59.64 $\pm$ 0.34 & \textbf{60.76 $\pm$ 0.26 (+1.12)}
& 62.46 $\pm$ 0.70 & \textbf{63.60 $\pm$ 0.24 (+1.14)} \\
\bottomrule
\end{NiceTabular}
}
\caption{
\textbf{Mean $\pm$ standard deviation (std) of Top-1 accuracy (\%) across multiple random seeds on CIFAR-100 with synthetic label noise (noise rate $20\%$).}
Each base objective is paired with its +CMRM variant.
The better value within each pair is shown in \textbf{bold}.
Numbers in parentheses denote the absolute accuracy improvement over the corresponding base objective.
}
\label{tab:results_multi_seeds}
\end{table*}

\begin{table*}[!ht]
  \centering
  \resizebox{\textwidth}{!}{%
    \begin{NiceTabular}{@{}l|cc cc cc cc cc cc@{}}
      \toprule
      \multirow{2}{*}{\textbf{Method}} & \multicolumn{2}{c}{\textbf{CIFAR-10N (Aggre)}} & \multicolumn{2}{c}{\textbf{CIFAR-10N (Rand1)}} & \multicolumn{2}{c}{\textbf{CIFAR-10N (Rand2)}} & \multicolumn{2}{c}{\textbf{CIFAR-10N (Rand3)}} & \multicolumn{2}{c}{\textbf{CIFAR-10N (Worst)}} & \multicolumn{2}{c}{\textbf{CIFAR-100N}}  \\
      \cmidrule(lr){2-3} \cmidrule(lr){4-5} 
      \cmidrule(lr){6-7} \cmidrule(lr){8-9} 
      \cmidrule(lr){10-11} \cmidrule(lr){12-13} 
      & ACC(\%) & M.APSS  & ACC(\%)  & M.APSS 
      & ACC(\%) & M.APSS  & ACC(\%)  & M.APSS
      & ACC(\%) & M.APSS  & ACC(\%)  & M.APSS\\
      \midrule
      NI-ERM
      & 98.69 & 0.904
      & 98.80 & 0.902
      & 98.65 & \textbf{0.903}
      & 98.67 & 0.904
      & 95.71 & 0.93
      & 83.17 & 1.49 \\
      NI-ERM+CMRM 
      & \textbf{98.81} & \textbf{0.903}
      & \textbf{99.03} & \textbf{0.901}
      & \textbf{98.95} & \textbf{0.903}
      & \textbf{98.88} & \textbf{0.899}
      & \textbf{97.19} & \textbf{0.91}
      & \textbf{83.95} & \textbf{1.29} \\
      & ($+0.12$) & ($-0.11\%$)
      & ($+0.23$) & ($-0.11\%$)
      & ($+0.30$) & ($0\%$)
      & ($+0.21$) & ($-0.55\%$)
      & ($+1.48$) & ($-2.15\%$) 
      & ($+0.78$) & ($-13.42\%$) \\
      \bottomrule
    \end{NiceTabular}%
  }
  \caption{\textbf{Top-1 accuracy (\%) and marginal average prediction set size (M.APSS $\downarrow$) on CIFAR-10N and CIFAR-100N corrupted by human annotation noise.} 
  Numbers in parentheses indicate the relative change: $+$ denotes accuracy improvement and $-\%$ denotes M.APSS reduction. 
  CMRM consistently improves accuracy and reduces uncertainty across CIFAR-N variants, with the largest gains observed on CIFAR-10N and CIFAR-100N. 
  On average across all datasets, 
  CMRM improves accuracy by $0.52$ and reduces M.APSS by $2.72\%$.
  }
  \label{tab:results_multi_cifarn_full}
\end{table*}

\begin{figure*}[!t]
    \centering
    \begin{minipage}[t]{0.24\linewidth}
    \centering
    \textbf{(a)} CE+CMRM
    \includegraphics[width=\linewidth]{figure/loss_curves_CE.png}
    \end{minipage}
    \begin{minipage}[t]{0.24\linewidth}
    \centering
    \textbf{(b)} Focal+CMRM
    \includegraphics[width=\linewidth]{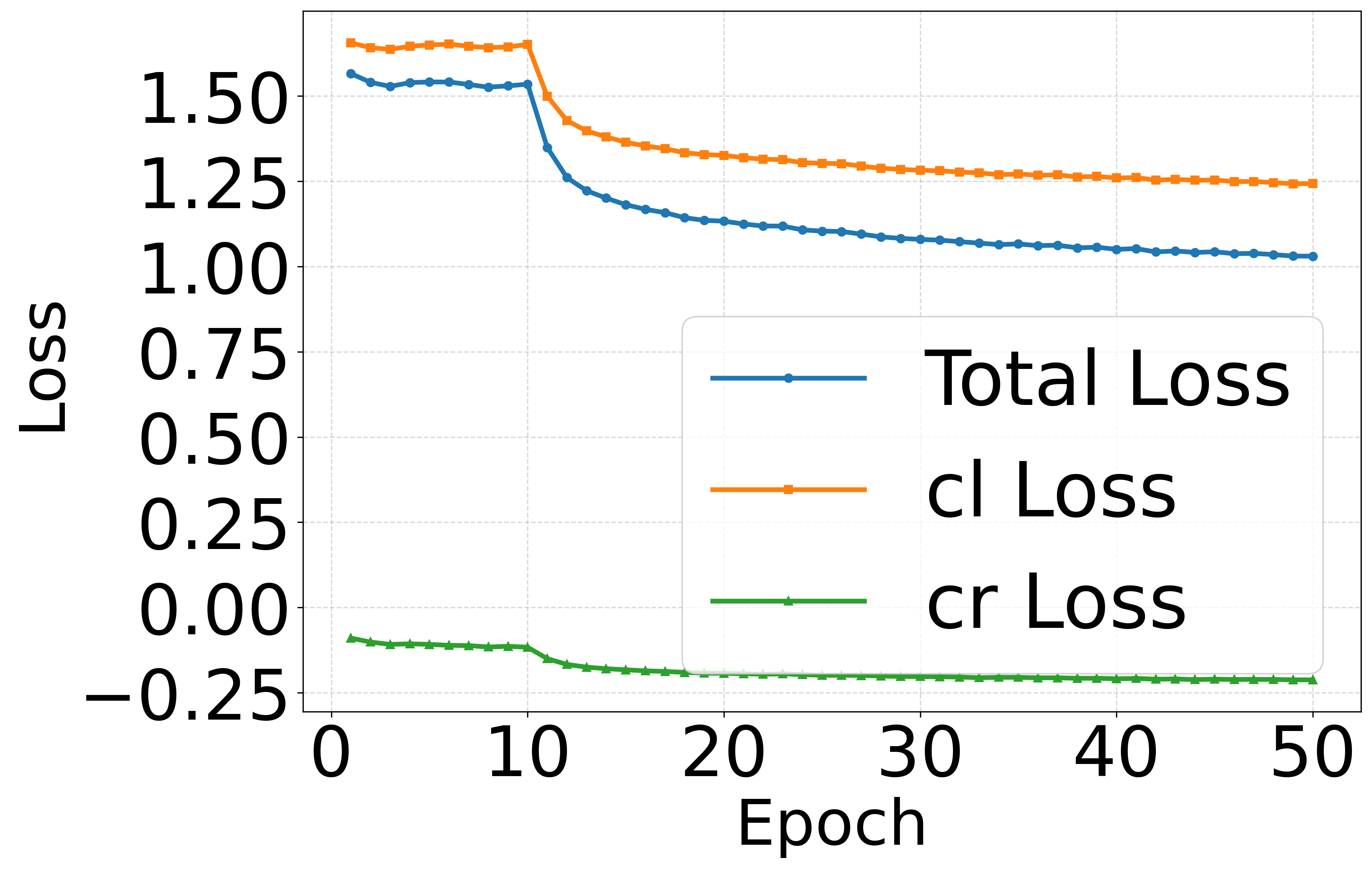}
    \end{minipage}
    \begin{minipage}[t]{0.24\linewidth}
    \centering
    \textbf{(c)} LDAM+CMRM
    \includegraphics[width=\linewidth]{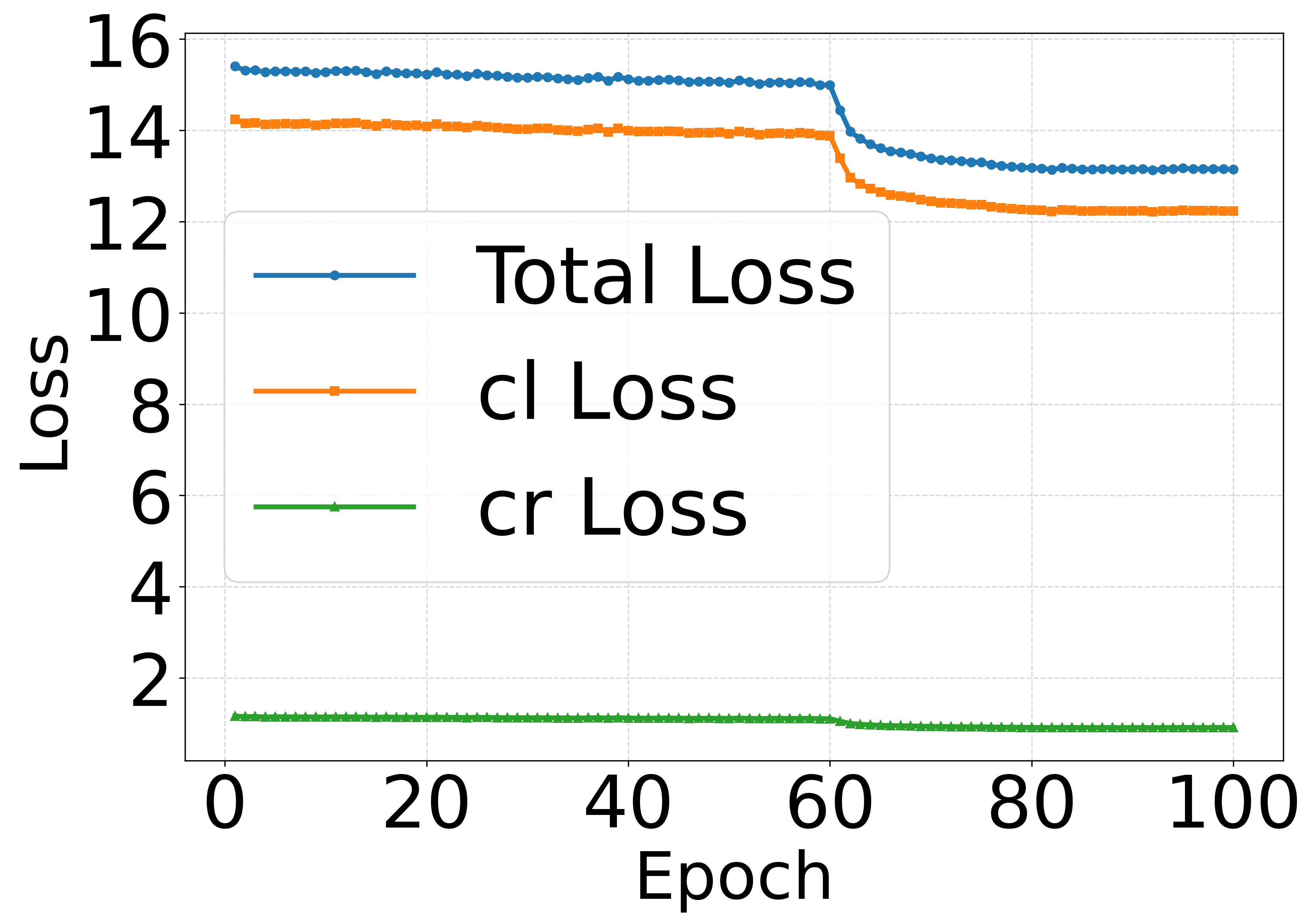}
    \end{minipage}
    \begin{minipage}[t]{0.24\linewidth}
    \centering
    \textbf{(d)} GCE+CMRM
    \includegraphics[width=\linewidth]{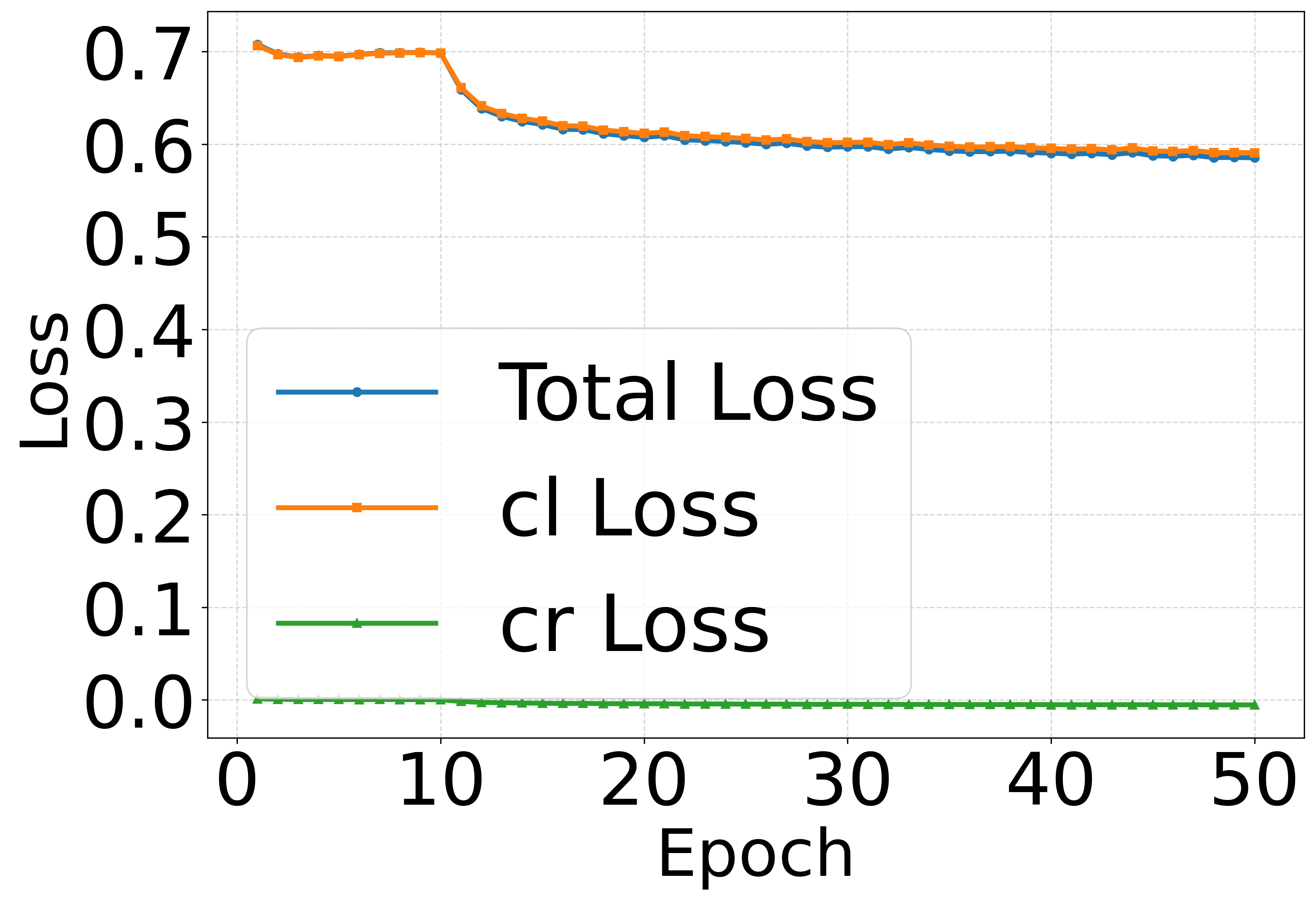}
    \end{minipage}
    \caption{
    \textbf{training dynamics of total loss (Total), classification loss (cl), and CMRM loss (cr) for all base losses} over epochs on CIFAR-100 with $20\%$ synthetical label noise. 
    Subfigure \textbf{(a)} CE+CERM; 
    Subfigure \textbf{(b)} Focal+CERM; 
    Subfigure \textbf{(c)} LDAM+CERM; 
    Subfigure \textbf{(d)} GCE+CERM; 
    CMRM exhibits stable and monotonic convergence alongside standard loss components.
    }
    \label{fig:multi_loss}
\end{figure*}

\begin{figure*}[!t]
    \centering
    \begin{minipage}[t]{0.24\linewidth}
    \centering
    \textbf{(a)} CE+CMRM
    \includegraphics[width=\linewidth]{figure/noise_ratio_curve.png}
    \end{minipage}
    \begin{minipage}[t]{0.24\linewidth}
    \centering
    \textbf{(b)} Focal+CMRM
    \includegraphics[width=\linewidth]{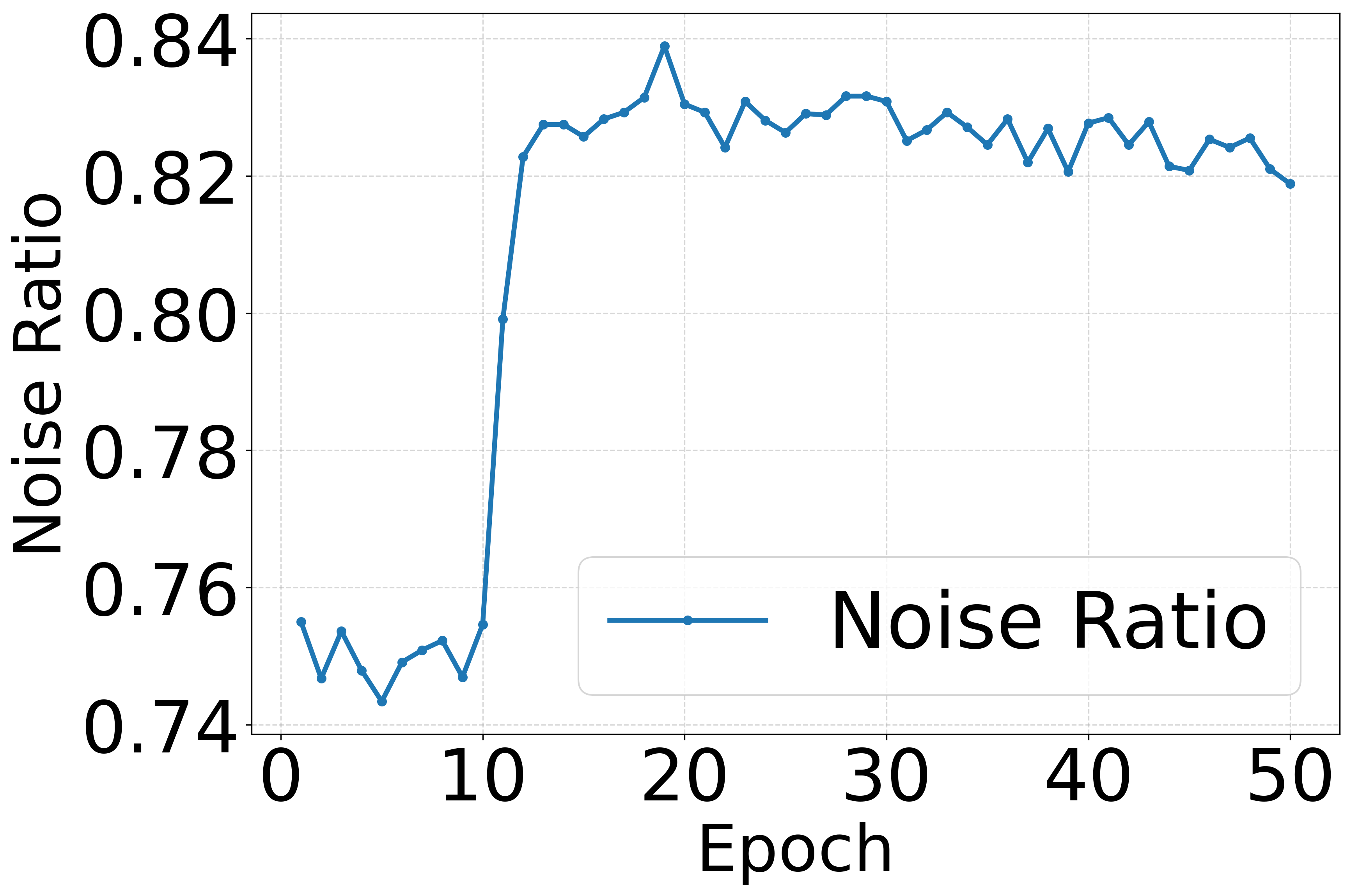}
    \end{minipage}
    \begin{minipage}[t]{0.24\linewidth}
    \centering
    \textbf{(c)} LDAM+CMRM
    \includegraphics[width=\linewidth]{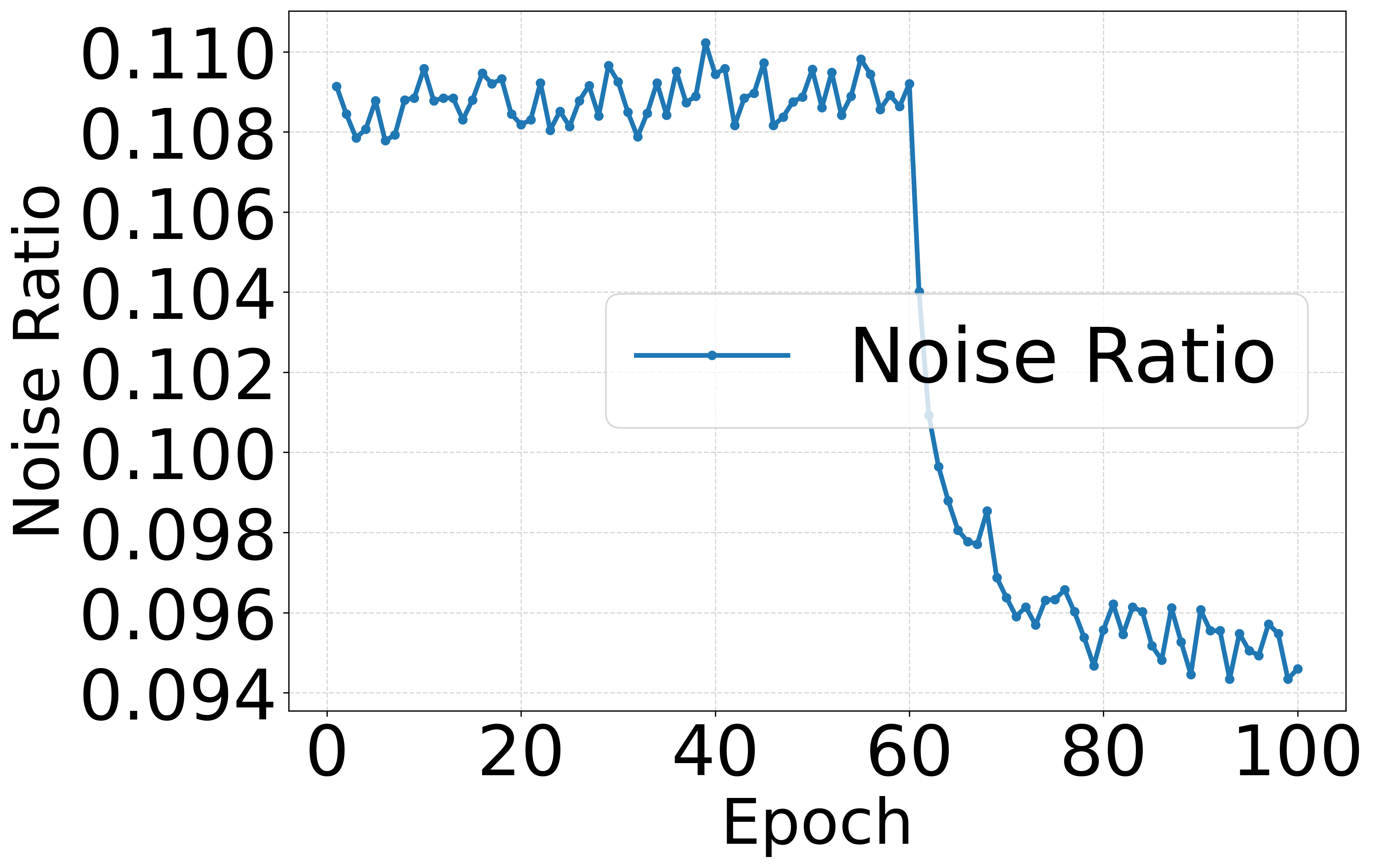}
    \end{minipage}
    \begin{minipage}[t]{0.24\linewidth}
    \centering
    \textbf{(d)} GCE+CMRM 
    \includegraphics[width=\linewidth]{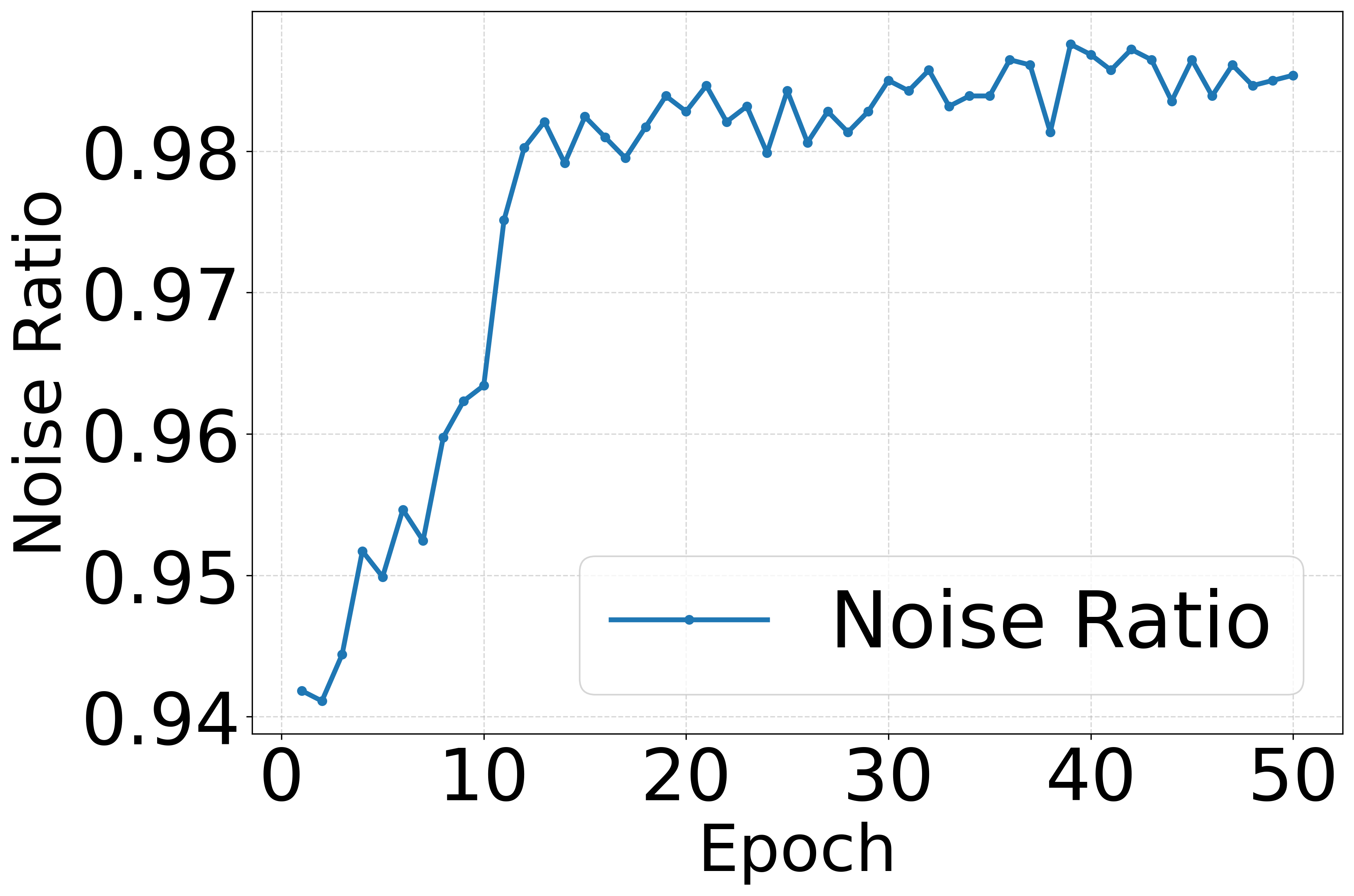}
    \end{minipage}
    \caption{
    \textbf{The ratio of noisy samples among those filtered out by CMRM at each epoch} 
    on CIFAR-100 with $20\%$ synthetical label noise. 
    Subfigure \textbf{(a)} CE+CMRM ($\alpha = 0.15$); 
    Subfigure \textbf{(b)} Focal+CMRM ($\alpha = 0.1$); 
    Subfigure \textbf{(c)} LDAM+CMRM ($\alpha = 0.15$); 
    Subfigure \textbf{(d)} GCE+CMRM ($\alpha = 0.05$); 
    CMRM consistently suppresses noisy examples by excluding low-margin samples during training.
    }
    \label{fig:multi_noise_ratio}
\end{figure*}

\begin{figure*}[!ht]
    \centering
    \begin{minipage}[t]{0.24\linewidth}
    \centering
    \textbf{(a)} CE+CMRM
    \includegraphics[width=\linewidth]{figure/alpha_best_acc_curve.png}
    \end{minipage}
    \begin{minipage}[t]{0.24\linewidth}
    \centering
    \textbf{(b)} Focal+CMRM
    \includegraphics[width=\linewidth]{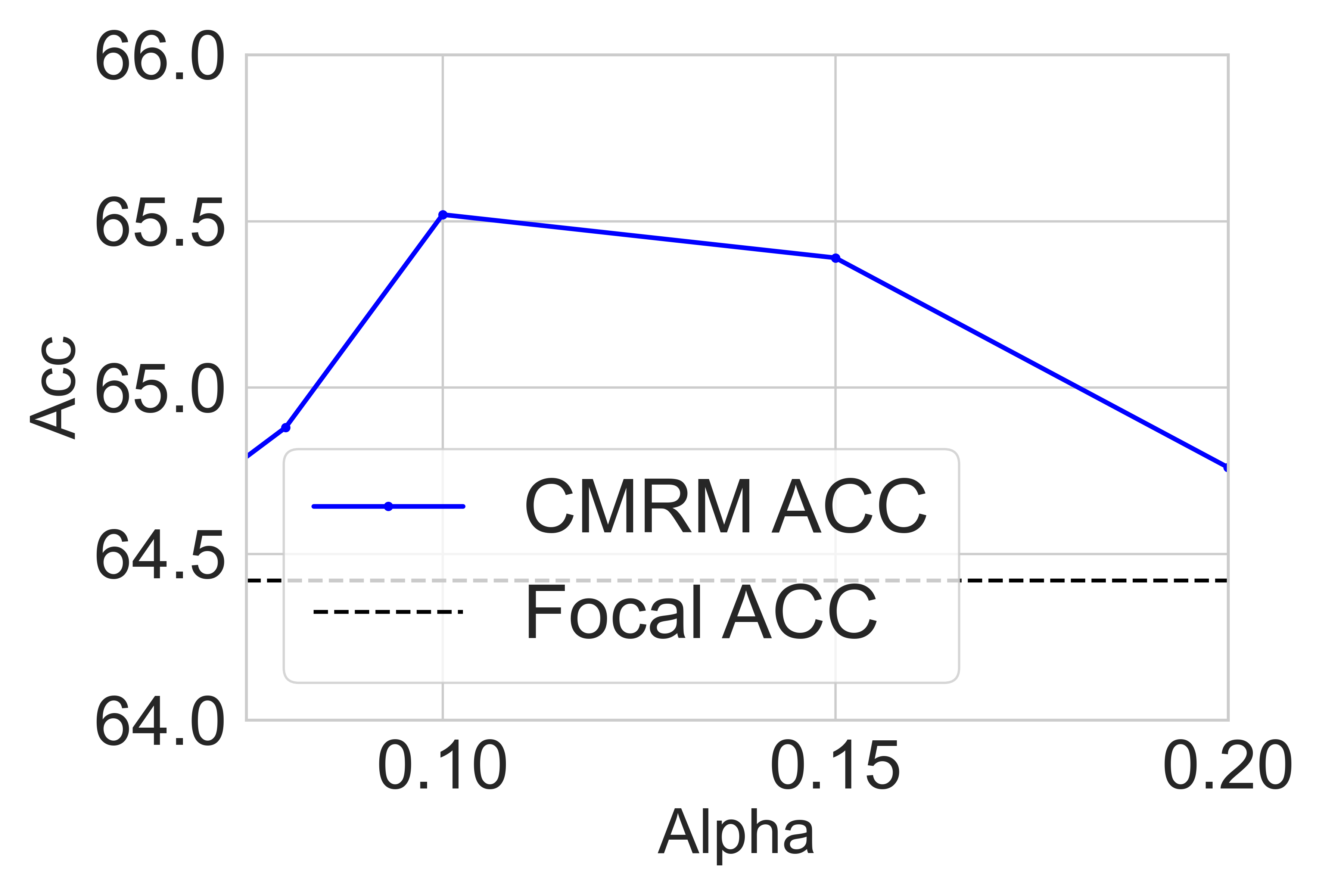}
    \end{minipage}
    \begin{minipage}[t]{0.24\linewidth}
    \centering
    \textbf{(c)} LDAM+CMRM
    \includegraphics[width=\linewidth]{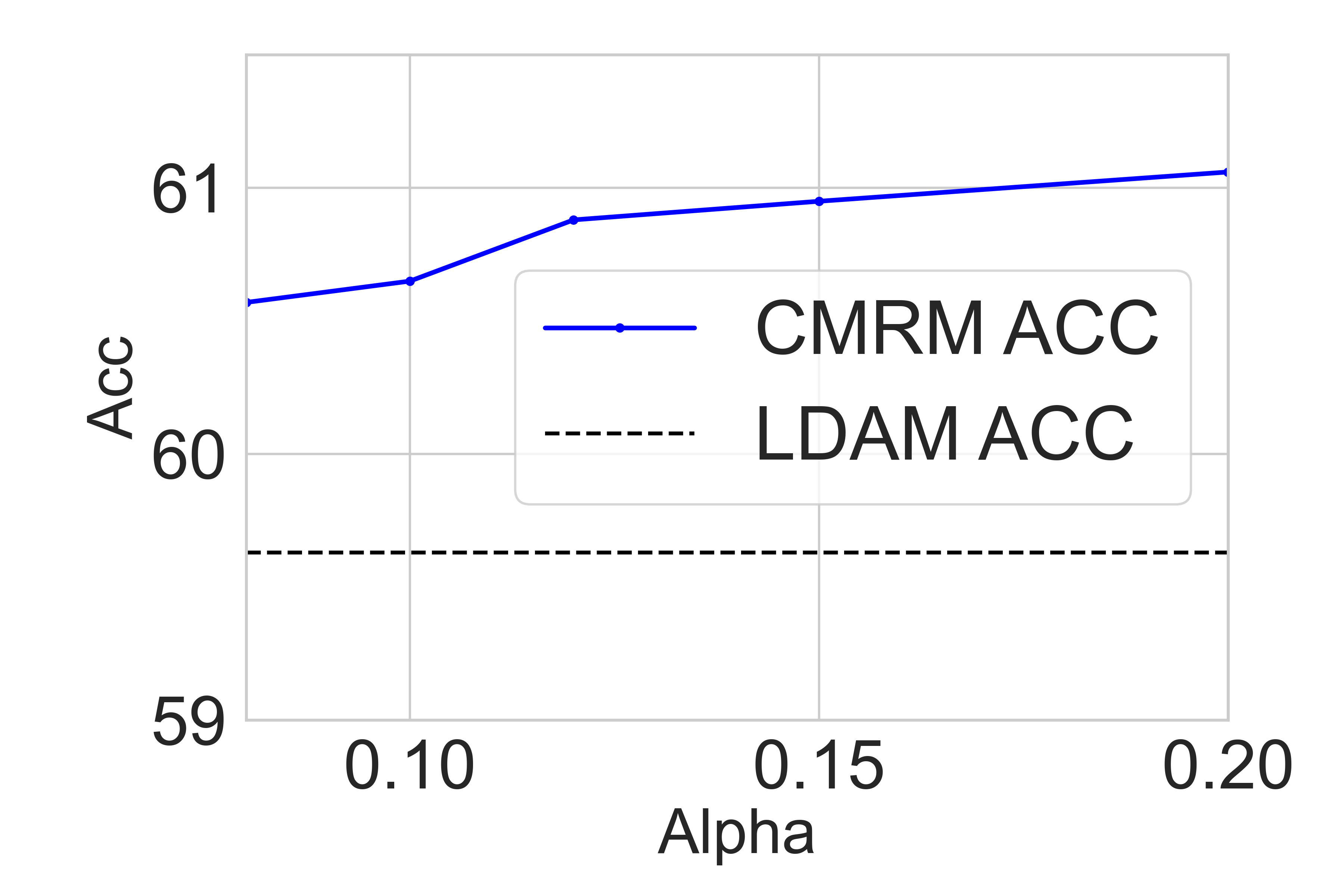}
    \end{minipage}
    \begin{minipage}[t]{0.24\linewidth}
    \centering
    \textbf{(d)} GCE+CMRM
    \includegraphics[width=\linewidth]{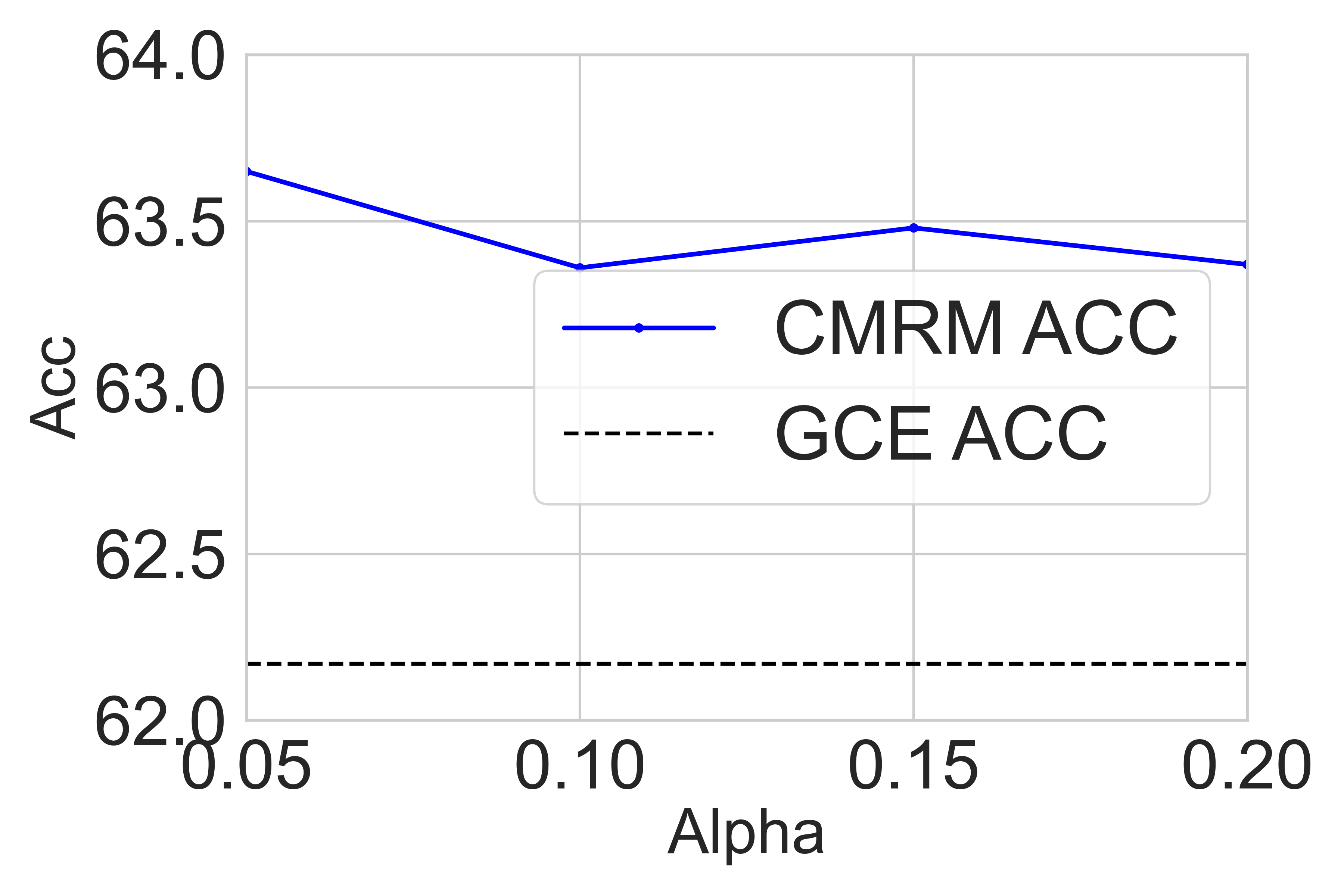}
    \end{minipage}
    \caption{
    \textbf{The sensitivity of $\alpha$ of different base losses} 
    on CIFAR-100 with $20\%$ synthetical label noise. 
    Subfigure \textbf{(a)} CE+CERM; 
    Subfigure \textbf{(b)} Focal+CERM; 
    Subfigure \textbf{(c)} LDAM+CERM; 
    Subfigure \textbf{(d)} GCE+CERM; 
    CMRM maintains higher accuracy than CE across a range of $\alpha$ values, indicating robustness to hyperparameter $\alpha$.
    }
    \label{fig:multi_sensitivity}
\end{figure*}

\begin{figure*}[!ht]
    \centering
    \begin{minipage}[t]{0.24\linewidth}
    \centering
    \textbf{(a)} CE+CMRM
    \includegraphics[width=\linewidth]{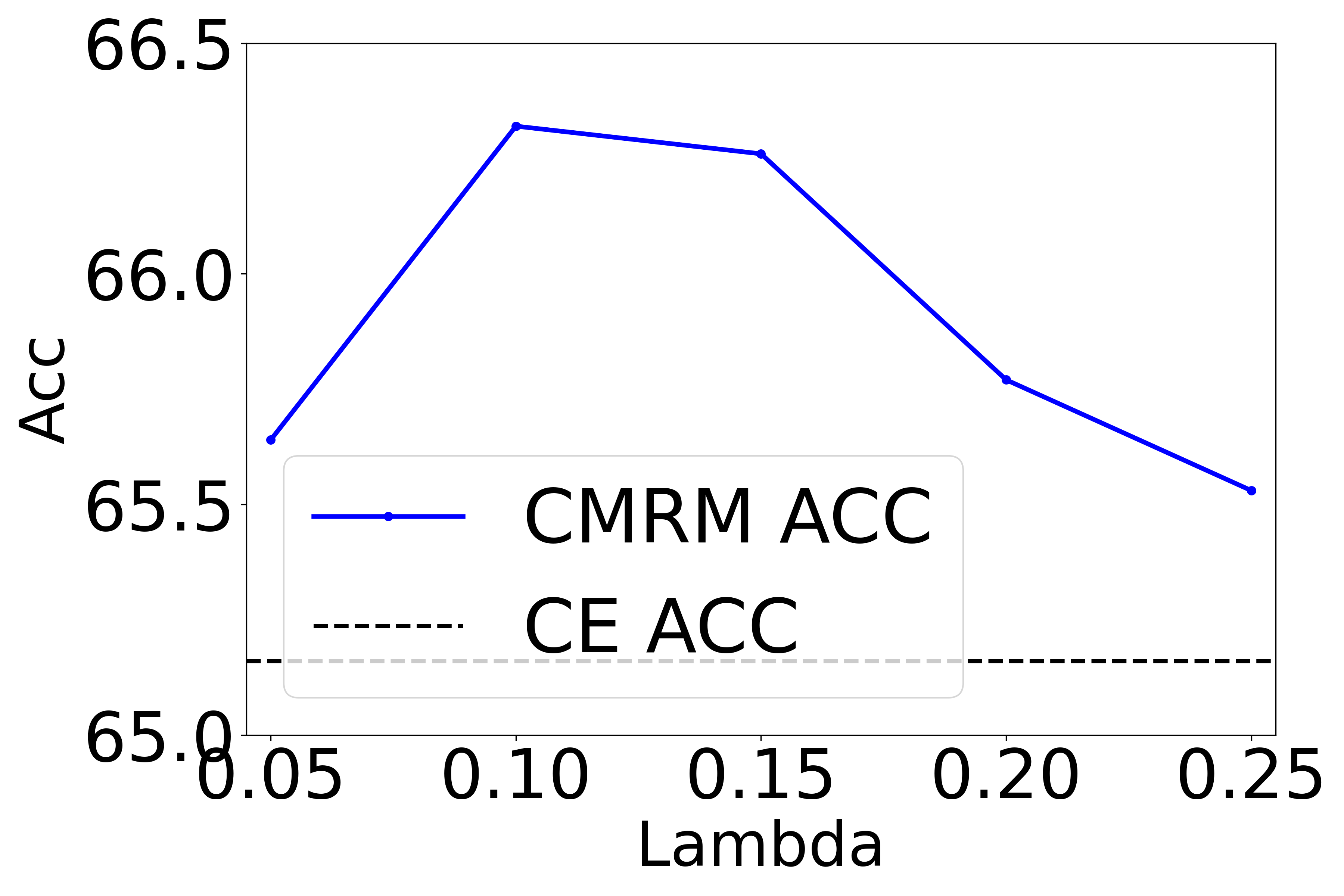}
    \end{minipage}
    \begin{minipage}[t]{0.24\linewidth}
    \centering
    \textbf{(b)} Focal+CMRM
    \includegraphics[width=\linewidth]{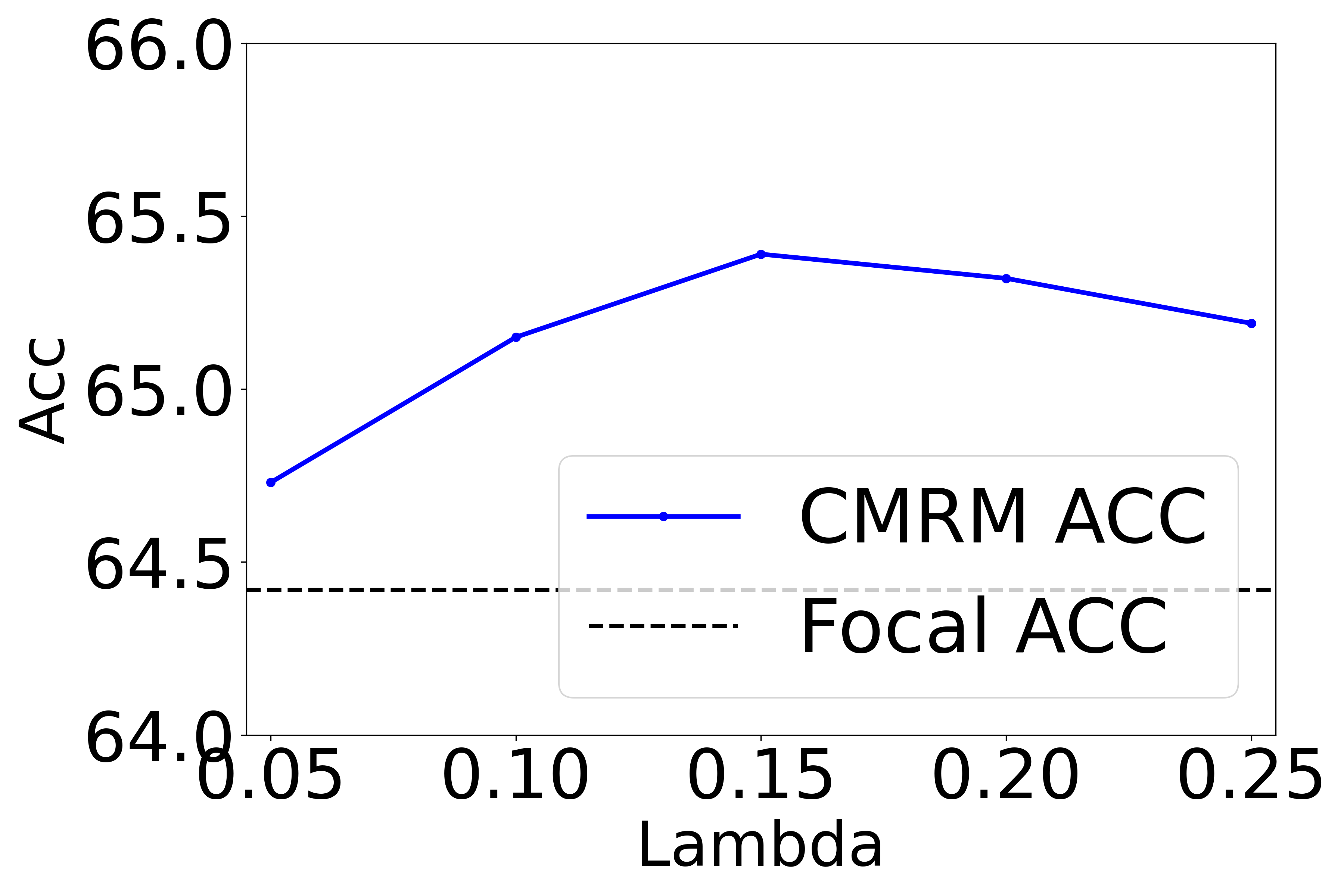}
    \end{minipage}
    \begin{minipage}[t]{0.24\linewidth}
    \centering
    \textbf{(c)} LDAM+CMRM
    \includegraphics[width=\linewidth]{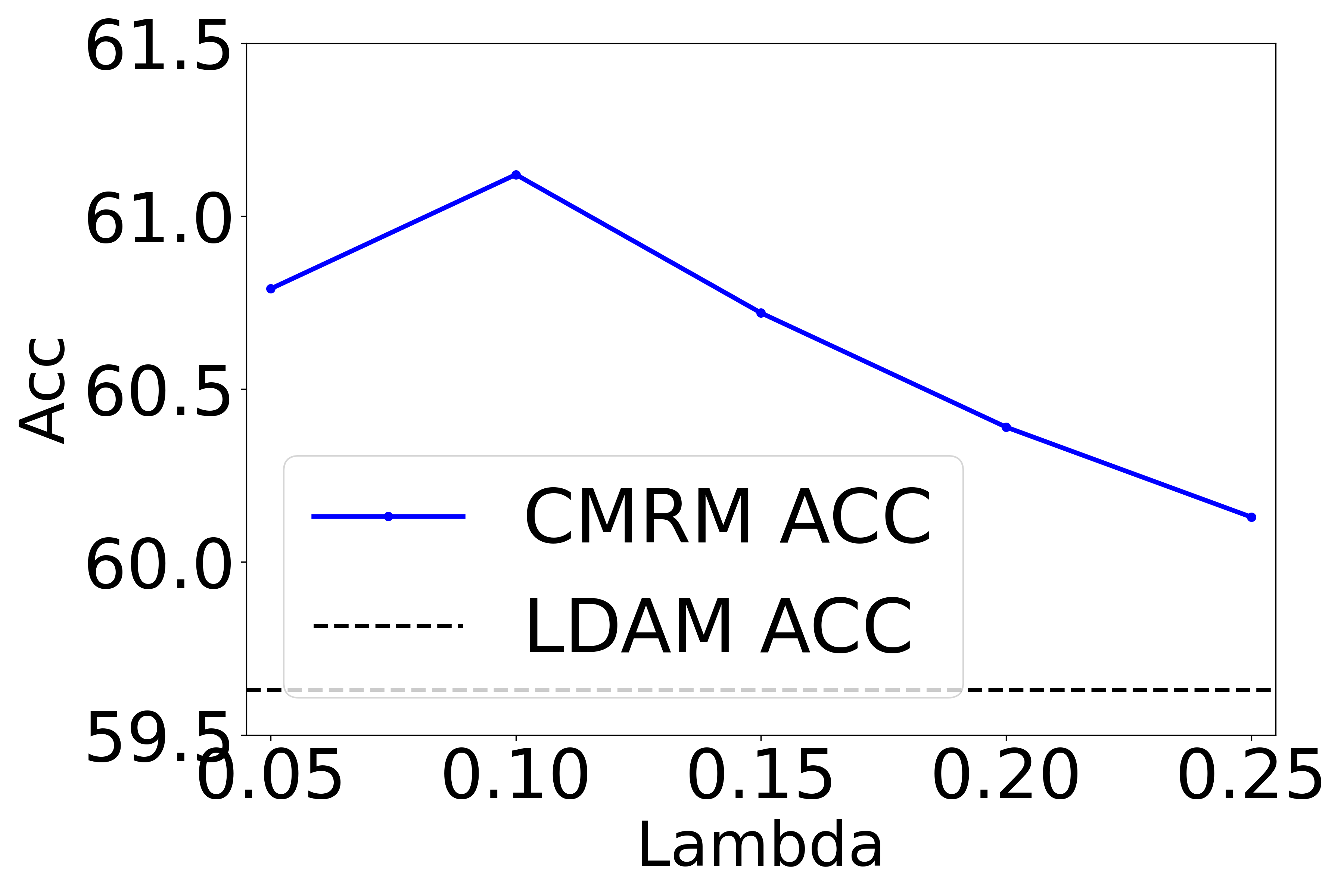}
    \end{minipage}
    \begin{minipage}[t]{0.24\linewidth}
    \centering
    \textbf{(d)} GCE+CMRM
    \includegraphics[width=\linewidth]{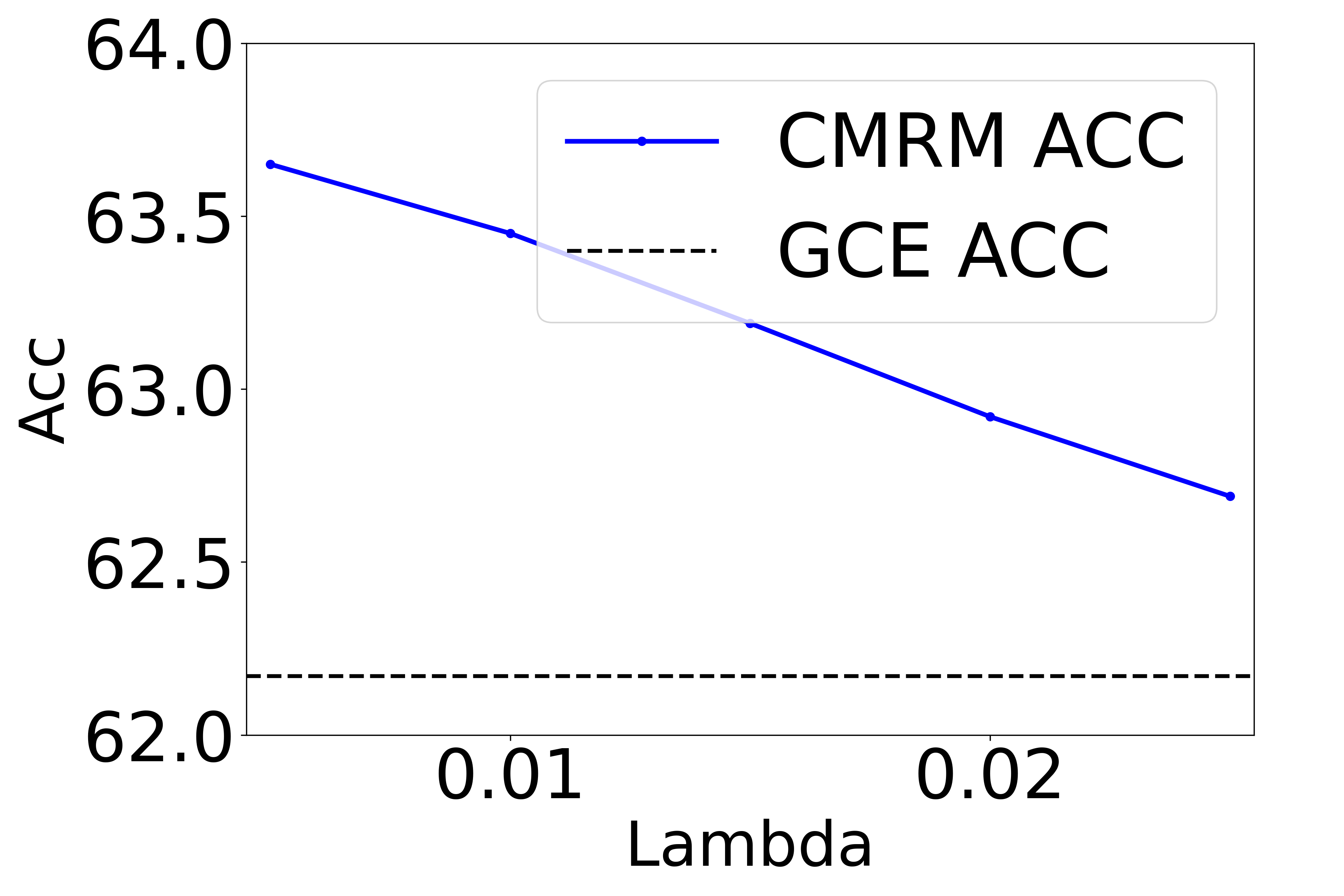}
    \end{minipage}
    \caption{
    \textbf{The sensitivity of $\lambda$ of different base losses} 
    on CIFAR-100 with $20\%$ synthetical label noise. 
    Subfigure \textbf{(a)} CE+CERM; 
    Subfigure \textbf{(b)} Focal+CERM; 
    Subfigure \textbf{(c)} LDAM+CERM; 
    Subfigure \textbf{(d)} GCE+CERM; 
    CMRM maintains higher accuracy than CE across a range of $\lambda$ values, indicating robustness to hyperparameter $\lambda$.
    }
    \label{fig:multi_sensitivity_lambda}
\end{figure*}

\begin{figure*}[!ht]
    \centering
    \begin{minipage}[t]{0.24\linewidth}
    \centering
    \textbf{(a)} CE+CMRM
    \includegraphics[width=\linewidth]{figure/hist_margin_kde.png}
    \end{minipage}
    \begin{minipage}[t]{0.24\linewidth}
    \centering
    \textbf{(b)} Focal+CMRM
    \includegraphics[width=\linewidth]{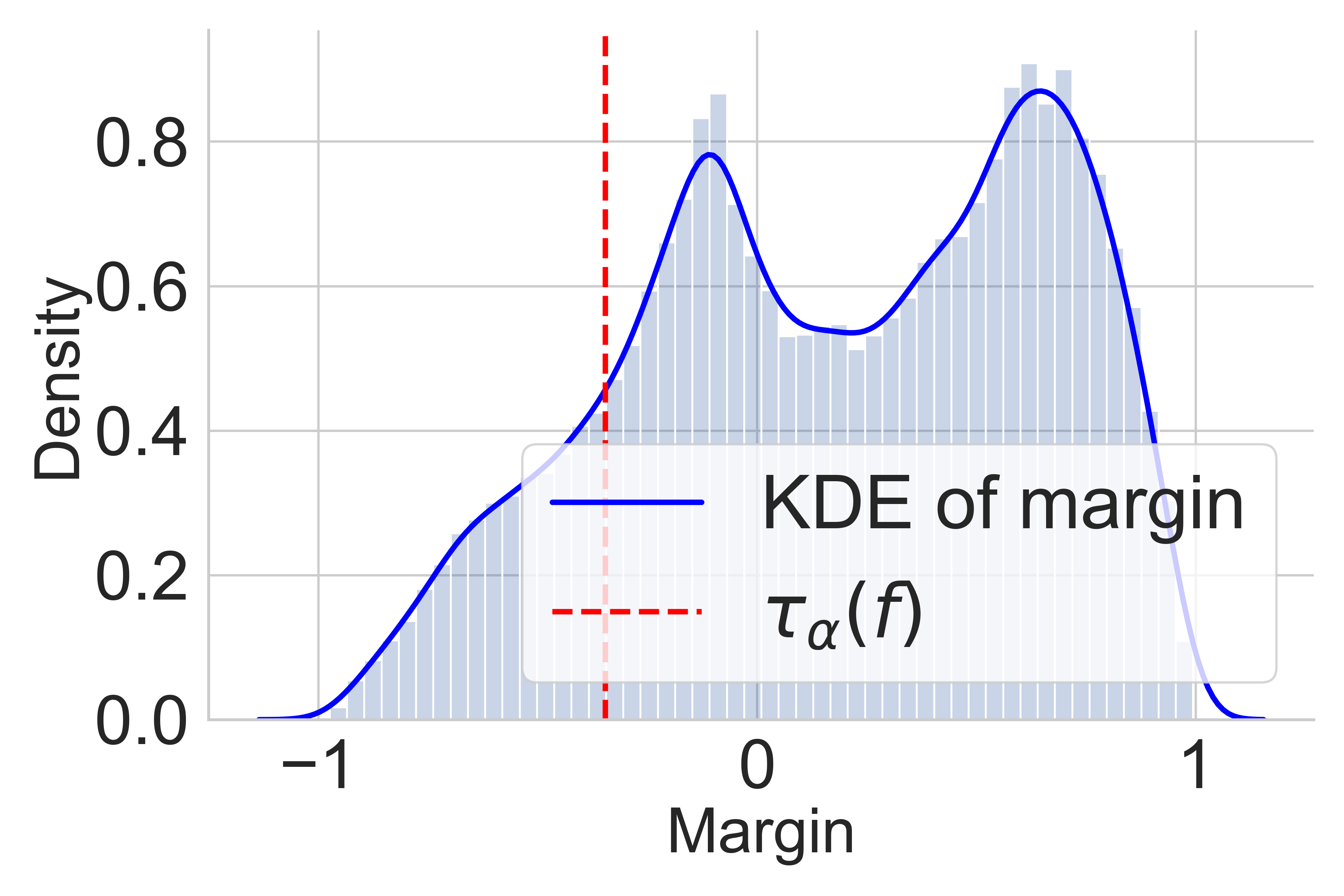}
    \end{minipage}
    \begin{minipage}[t]{0.24\linewidth}
    \centering
    \textbf{(c)} LDAM+CMRM
    \includegraphics[width=\linewidth]{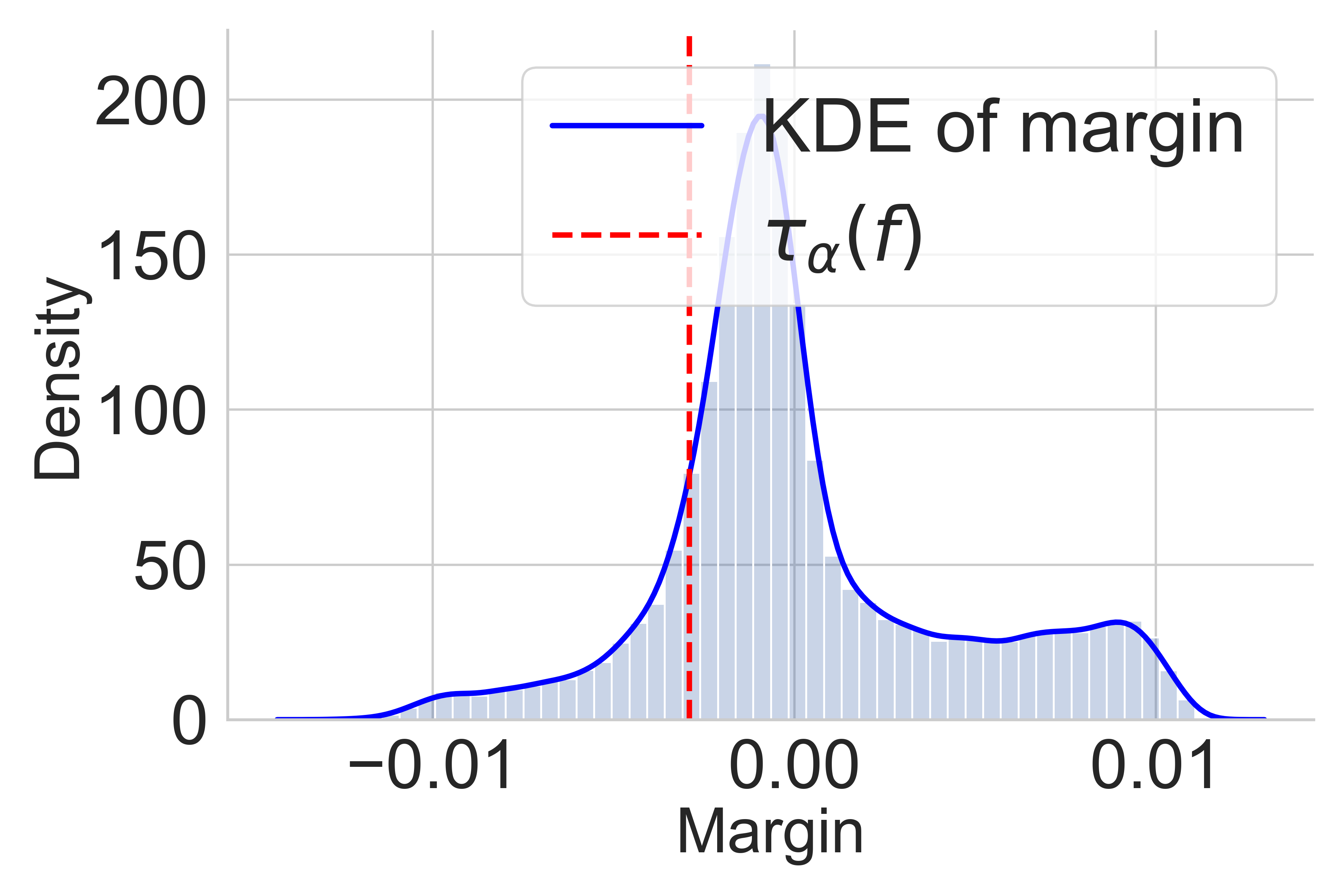}
    \end{minipage}
    \begin{minipage}[t]{0.24\linewidth}
    \centering
    \textbf{(d)} GCE+CMRM
    \includegraphics[width=\linewidth]{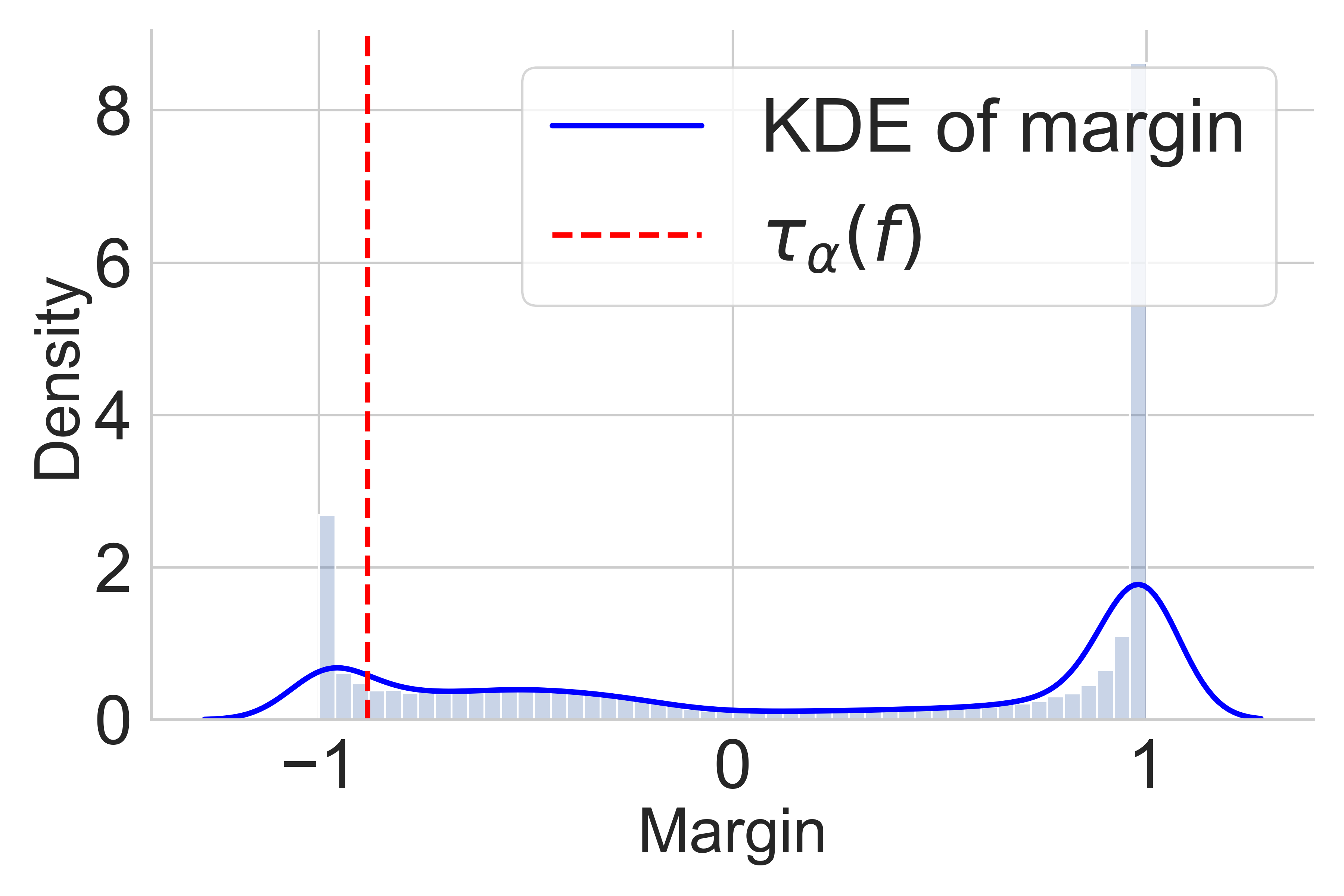}
    \end{minipage}
    \caption{
    \textbf{the kernel density estimate (KDE) of the margin distribution of different base losses} 
    on CIFAR-100 with $20\%$ synthetical label noise. 
    Subfigure \textbf{(a)} CE+CMRM ($\alpha = 0.15$); 
    Subfigure \textbf{(b)} Focal+CMRM ($\alpha = 0.1$); 
    Subfigure \textbf{(c)} LDAM+CMRM ($\alpha = 0.15$); 
    Subfigure \textbf{(d)} GCE+CMRM ($\alpha = 0.05$);
    The vertical dashed line indicating the estimated $\tau_\alpha(f)$. 
    The density curve is smooth and strictly positive around $\tau_\alpha(f)$, supporting the differentiability and positive-density assumption in Proposition~\ref{proposition:quantile_gap}.
    }
    \label{fig:multi_kde}
\end{figure*}

\textbf{Result: CMRM improves accuracy and reduces uncertainty under different types of noise.}

Table~\ref{tab:results_multi_noise_rates} summarizes results on CIFAR-100 with synthetic label noise at rates $\{0\%, 5\%, 10\%, 20\%, 30\%, 40\%\}$, where $0\%$ corresponds to the clean-label setting.
For each objective (CE, Focal, LDAM, and GCE), we compare the Base model with its +CMRM variant.
On average across all objectives and noise levels, CMRM improves accuracy by $1.34$ and reduces M.APSS by $3.89\%$.
The most notable gains occur under moderate to high noise. 
For example, CE and Focal achieve up to +$2.14$ and $2.22$ accuracy improvements, while GCE shows the largest uncertainty reduction (up to $-18.44\%$ in M.APSS).
Even when combined with LDAM, which already promotes margin separation, CMRM consistently provides additional accuracy improvements without increasing predictive uncertainty.

Table~\ref{tab:results_multi_seeds} reports the mean $\pm$ standard deviation of Top-1 accuracy across multiple random seeds on CIFAR-100 with $20\%$ synthetic label noise.
For each objective (CE, Focal, LDAM, and GCE), we compare the Base model with its +CMRM variant.
Across all objectives, CMRM consistently improves accuracy under the multi-seed evaluation.
In particular, CMRM yields accuracy gains of $+1.21$, $+0.93$, $+1.12$, and $+1.14$ for CE, Focal, LDAM, and GCE, respectively.
These results demonstrate that the performance improvements brought by CMRM remain stable across different random seeds.

Table \ref{tab:results_multi_cifarn_full} summarizes results on CIFAR-10N and CIFAR-100N with human-annotated label noise. 
Across all datasets, CMRM consistently improves both accuracy and uncertainty compared to NI-ERM.
On average, CMRM increases accuracy by $0.52$ and reduces M.APSS by $2.72\%$, indicating more confident and reliable predictions under noisy supervision.
The improvements are particularly pronounced on the most challenging settings, i.e., CIFAR-10N (Worst) and CIFAR-100N, where accuracy rises by up to $1.48$ and predictive uncertainty decreases by as much as $13.42\%$.
These consistent gains demonstrate that CMRM effectively enhances both robustness and calibration under real-world label noise.

\textbf{Result: CMRM loss convergences.}
Figure~\ref{fig:multi_loss} shows the training dynamics of the classification loss and the CMRM regularization loss. 
Both components decrease steadily and stabilize as training progresses, indicating smooth joint optimization. 
The CMRM term integrates with standard objectives and does not introduce instability or slowing down of convergence, demonstrating that CMRM can be efficiently optimized.

\textbf{Result: CMRM filters out noisy samples during training.}
Figure~\ref{fig:multi_noise_ratio} shows the fraction of noisy samples among those excluded by CMRM at each epoch. 
This proportion rapidly increases during the early training phase and stabilizes above $78\%$, indicating that CMRM consistently identifies and filters out mislabeled examples via its margin-based thresholding mechanism.

\textbf{Result: CMRM is robust to the choices of hyperparameter $\alpha$ and $\lambda$.}
Figure~\ref{fig:multi_sensitivity} and \ref{fig:multi_sensitivity_lambda} examine the sensitivity of CMRM to the hyperparameter $\alpha$ and $\lambda$, respectively.
Across a range of $\alpha$ and $\lambda$ values, CMRM consistently achieves higher accuracy than CE, indicating that its performance is robust to the choice of $\alpha$ and $\lambda$, and does not rely on careful hyperparameter tuning.

\textbf{Result: Assumptions in Proposition~\ref{proposition:quantile_gap} are empirically valid.}
Figure~\ref{fig:multi_kde} presents the kernel density estimate (KDE) of the margin distribution, with the vertical dashed line indicating the estimated $\tau_{\alpha}(f)$.
The density curve is smooth and strictly positive in the neighborhood of $\tau_{\alpha}(f)$, supporting the differentiability and positive-density assumption in Proposition~\ref{proposition:quantile_gap}.

\section{Additional Experiments for Binary Classification}
\label{ssup:sec:experiment_binary}

\subsection{Additional Experimental Setup Details}
\label{supp:subsec:binary_setup}

\textbf{Datasets.}
We evaluate our methods on four standard binary classification benchmarks, three of which are sourced from the UCI Machine Learning Repository:
\begin{itemize}
    \item \textbf{EMAIL}: This dataset contains features extracted from emails and aims to classify whether an email is spam or not. We treat spam emails as the positive class (label 1) and non-spam emails as the negative class (label 0). Following prior work, we invert the original labels in the dataset to conform to this definition.
    \item \textbf{ADULT}: This dataset involves predicting whether an individual's income exceeds $50K$ based on demographic and employment features. We define the positive class as individuals earning $\leq 50K$ (label 1) and the negative class as those earning $>50K$ (label 0), effectively focusing on identifying lower-income individuals.
    \item \textbf{CREDIT}: This dataset contains information on credit card clients and whether they defaulted on their payment in the following month. We define the positive class as non-defaulting clients (label 1) and the negative class as clients who defaulted (label 0), inverting the original label to emphasize reliable borrowers.
\end{itemize}

\begin{table*}[!t]
\centering
\setlength{\tabcolsep}{4pt}
\resizebox{\textwidth}{!}{%
\begin{NiceTabular}{@{}l l *{8}{c}@{}}
\toprule
\textbf{Dataset} & \textbf{Method} & \multicolumn{8}{c}{\textbf{Measure}} \\
\cmidrule(lr){3-10}
& & AUROC ($\uparrow$) & AUPRC ($\uparrow$) & FNR ($\downarrow$) & FPR ($\downarrow$)
  & ACC ($\uparrow$) & M.APSS ($\downarrow$) & PC APSS ($\downarrow$) & NC APSS ($\downarrow$) \\
\midrule
\multirow{8}{*}{Adult}
& LR              & 0.784 & 0.885 & \textbf{0.073} & 0.571 & 0.802 & 1.223 & 1.154 & 1.432 \\
& LR + CMRM        & \textbf{0.852} & \textbf{0.925} & 0.082 & \textbf{0.422} & \textbf{0.833} & \textbf{1.209} & \textbf{1.109} & \textbf{1.308} \\
& Focal           & 0.809 & 0.890 & 0.136 & 0.388 & 0.801 & 1.257 & 1.224 & 1.356 \\
& Focal + CMRM     & \textbf{0.872} & \textbf{0.942} & \textbf{0.128} & \textbf{0.324} & \textbf{0.823} & \textbf{1.221} & \textbf{1.148} & \textbf{1.295} \\
& SVM             & 0.808 & 0.925 & \textbf{0.029} & 0.807 & 0.776 & 1.276 & 1.370 & 1.512 \\
& SVM + CMRM       & \textbf{0.847} & \textbf{0.937} & 0.048 & \textbf{0.585} & \textbf{0.817} & \textbf{1.199} & \textbf{1.322} & \textbf{1.343} \\
& GCE             & 0.819 & 0.904 & \textbf{0.119} & 0.424 & \textbf{0.804} & 1.286 & \textbf{1.176} & 1.396 \\
& GCE + CMRM      & \textbf{0.846} & \textbf{0.928} & 0.172 & \textbf{0.286} & 0.800 & \textbf{1.273} & 1.207 & \textbf{1.340} \\
\midrule
\multirow{8}{*}{Email}
& LR              & 0.831 & 0.869 & \textbf{0.246} & 0.200 & 0.773 & 1.405 & 1.415 & 1.392 \\
& LR + CMRM        & \textbf{0.875} & \textbf{0.910} & 0.259 & \textbf{0.129} & \textbf{0.793} & \textbf{1.335} & \textbf{1.247} & \textbf{1.291} \\
& Focal           & 0.833 & 0.858 & \textbf{0.228} & 0.208 & 0.780 & 1.404 & 1.449 & 1.345 \\
& Focal + CMRM     & \textbf{0.907} & \textbf{0.916} & 0.246 & \textbf{0.080} & \textbf{0.822} & \textbf{1.235} & \textbf{1.152} & \textbf{1.202} \\
& SVM             & 0.952 & 0.964 & \textbf{0.043} & 0.219 & 0.885 & 1.029 & 1.003 & 1.001 \\
& SVM + CMRM       & \textbf{0.954} & \textbf{0.967} & 0.060 & \textbf{0.125} & \textbf{0.913} & \textbf{0.975} & \textbf{0.996} & \textbf{0.993} \\
& GCE             & 0.931 & 0.925 & \textbf{0.069} & 0.101 & 0.918 & \textbf{0.976} & \textbf{0.977} & \textbf{0.975} \\
& GCE + CMRM      & \textbf{0.938} & \textbf{0.933} & 0.074 & \textbf{0.082} & \textbf{0.922} & 0.984 & 0.982 & 0.986 \\
\midrule
\multirow{8}{*}{Credit}
& LR              & 0.690 & 0.866 & 0.134 & 0.600 & 0.764 & 1.468 & 1.414 & 1.522 \\
& LR + CMRM        & \textbf{0.714} & \textbf{0.877} & \textbf{0.121} & \textbf{0.565} & \textbf{0.782} & \textbf{1.400} & \textbf{1.365} & \textbf{1.436} \\
& Focal           & 0.673 & 0.862 & \textbf{0.169} & 0.588 & 0.739 & \textbf{1.507} & \textbf{1.467} & 1.547 \\
& Focal + CMRM     & \textbf{0.677} & \textbf{0.861} & 0.176 & \textbf{0.547} & \textbf{0.743} & 1.527 & 1.543 & \textbf{1.511} \\
& SVM             & 0.688 & 0.856 & \textbf{0.079} & 0.618 & 0.803 & \textbf{1.387} & \textbf{1.335} & 1.438 \\
& SVM + CMRM       & \textbf{0.711} & \textbf{0.869} & 0.084 & \textbf{0.553} & \textbf{0.813} & 1.421 & 1.437 & \textbf{1.405} \\
& GCE             & 0.671 & 0.861 & \textbf{0.134} & 0.845 & 0.751 & 1.446 & 1.366 & 1.526 \\
& GCE + CMRM      & \textbf{0.701} & \textbf{0.875} & \textbf{0.115} & \textbf{0.627} & \textbf{0.773} & \textbf{1.417} & 1.339 & \textbf{1.495} \\
\bottomrule
\end{NiceTabular}
}%
\caption{\textbf{Performance comparison on binary classification.}
We evaluate three datasets (Adult, Email, and Credit) under a $20\%$ label noise setting.
Models are assessed across accuracy (ACC), ranking (AUROC, AUPRC), calibration (FNR, FPR), and uncertainty metrics: marginal average prediction set size (M.\,APSS), positive-class prediction set size (PC APSS), and negative-class prediction set size (NC APSS).
$\uparrow$ means higher is better; $\downarrow$ lower is better.
Best results for each base model (LR, Focal, SVM, GCE) are in \textbf{bold}.
Our method consistently improves uncertainty estimation (M.APSS, PC APSS, NC APSS), ranking metrics (AUROC, AUPRC), and accuracy (ACC, FPR), while slightly increasing FNR.
}
\label{tab:binary_appendix}
\end{table*}

\textbf{Hyperparameters for training.}
We set datasets, base losses, base models, learning rate, $\lambda^+$, $\lambda^-$, $\alpha^+$, and $\alpha^-$ as hyperparameter choices. 
We search for hyperparameters on learning rate ($\eta$) $\in \{0.001, 0.005, 0.01\}$, $\lambda^+ \in \{0, 0.05, 0.1, 0.15, 0.2, 0.3, 0.4, 0.5, 0.6, 0.7, 0.8, 0.9, 1.0, 1.1, 1.2, 1.3, 1.4, 1.5\}$, $\lambda^- \in \{0, 0.05, 0.1, 0.2, 0.3, 0.4, 0.5, 0.6, 0.7, 0.8, 0.9, 1.0, 1.1, 1.2, 1.3, 1.4, 1.5\}$, $\alpha^+ \in \{0.1, 0.2, 0.3, 0.4, 0.5\}$, and $\alpha^- \in \{0.1, 0.2, 0.3, 0.4, 0.5\}$ to select the best combination of hyperparameters of each methods.
The hyperparameters employed to get the results presented in the main paper are summarized in Table \ref{tab:finetune_hyper_params}.

\begin{table}[!ht]
\centering
\resizebox{0.85\textwidth}{!}{
\begin{tabular}{|c|c|c|c|c|c|c|c|c|}
\hline
Data & loss & Architecture & Epochs & $\eta$ & $\lambda^+$ & $\lambda^-$ & $\alpha^+$ & $\alpha^-$ \\ \hline
\multirow{4}{*}{Adult} & LR+CMRM & MLP & 150 & 0.01 & 0.4 & 0.5 & 0.3& 0.1 \\ \cline{2-9} 
& Focal+CMRM & MLP & 150 & 0.01 & 0.1 & 0.4 & 0.3& 0.5 \\ \cline{2-9} 
& SVM+CMRM & SVM & 150 & 0.01 & 0.15 & 1.5 & 0.2 & 0.5 \\ \cline{2-9} 
& GCE+CMRM & MLP & 150 & 0.001 & 0.0 & 0.2 & 0.1 & 0.5 \\ \hline 
\multirow{4}{*}{Email} & LR+CMRM & MLP & 150 & 0.01 & 0.7 & 0.2 & 0.2 & 0.2 \\ \cline{2-9} 
& Focal+CMRM & MLP & 150 & 0.01 & 0.4 & 0.6 & 0.1& 0.1 \\ \cline{2-9} 
& SVM+CMRM & SVM & 150 &  0.01 & 0.3 & 1.0 & 0.3 & 0.3  \\ \cline{2-9} 
& GCE+CMRM & MLP & 150 & 0.01 & 0.9 & 0.7 & 0.4 & 0.4 \\ \hline 
\multirow{4}{*}{Credit} & LR+CMRM & MLP & 150 & 0.005 & 0.0 & 0.3 & 0.1 & 0.1 \\ \cline{2-9} 
& Focal+CMRM & MLP & 150 & 0.001 & 0.05 & 0.3 & 0.3 & 0.1 \\ \cline{2-9} 
& SVM+CMRM & SVM & 150 & 0.001 & 0.0 & 1.0 & 0.1 & 0.5 \\ \cline{2-9}
& GCE+CMRM & MLP & 150 & 0.0005 & 0.0 & 0.3 & 0.1 & 0.1 \\ \hline 
\end{tabular}
}
\caption{\textbf{The details we used to train our models for binary classification}. We reported the hyperparameters that give the best combination os metrics.}
\label{tab:finetune_hyper_params}
\end{table}

\subsection{Additional Experimental Results}
\label{supp:subsec:binary_results}

\begin{figure}[!t]
    \centering
    \begin{minipage}[t]{0.24\linewidth}
    \centering
    \textbf{(a)} LR + CMRM 
    \includegraphics[width=\linewidth]{figure/email_MLP_0.2_0.8_0.7_0.5_0.5_0.2_0.2_0.01_loss_plot.png}
    \\
    \includegraphics[width=\linewidth]{figure/email_MLP_0.2_0.8_0.7_0.5_0.5_0.2_0.2_0.01_thresholds_plot.png}
    \end{minipage}
    \begin{minipage}[t]{0.24\linewidth}
    \centering
    \textbf{(b)} Focal+CMRM
    \includegraphics[width=\linewidth]{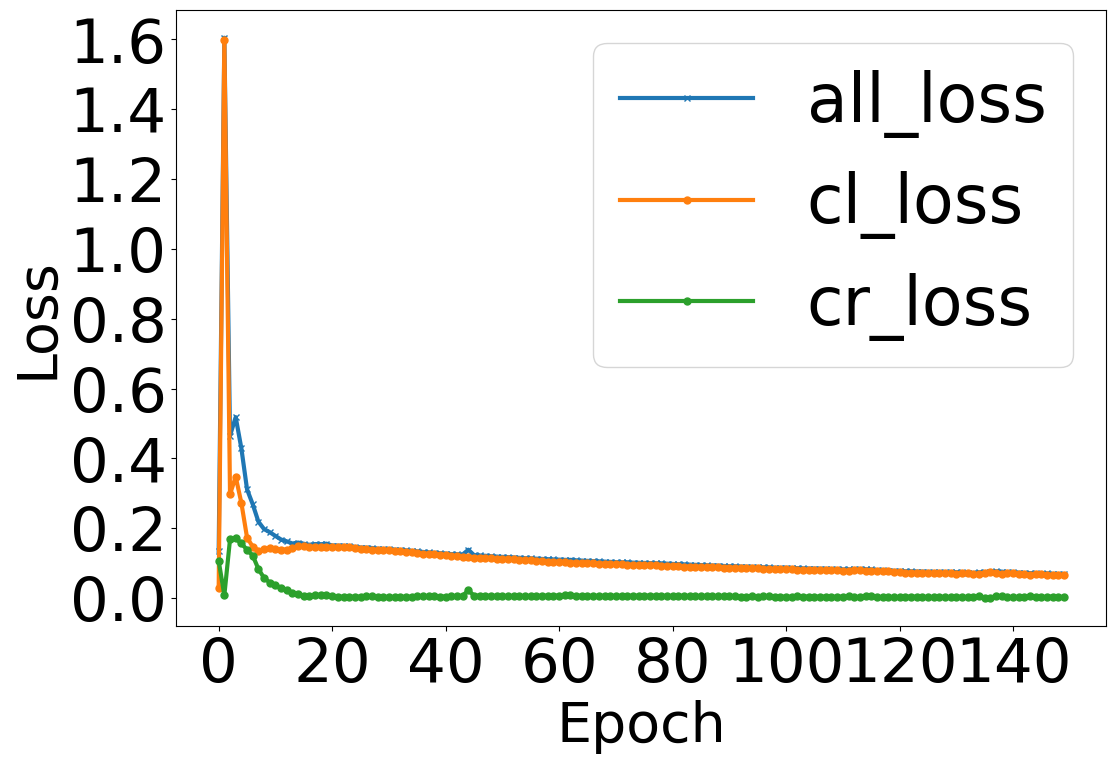}
    \\
    \includegraphics[width=\linewidth]{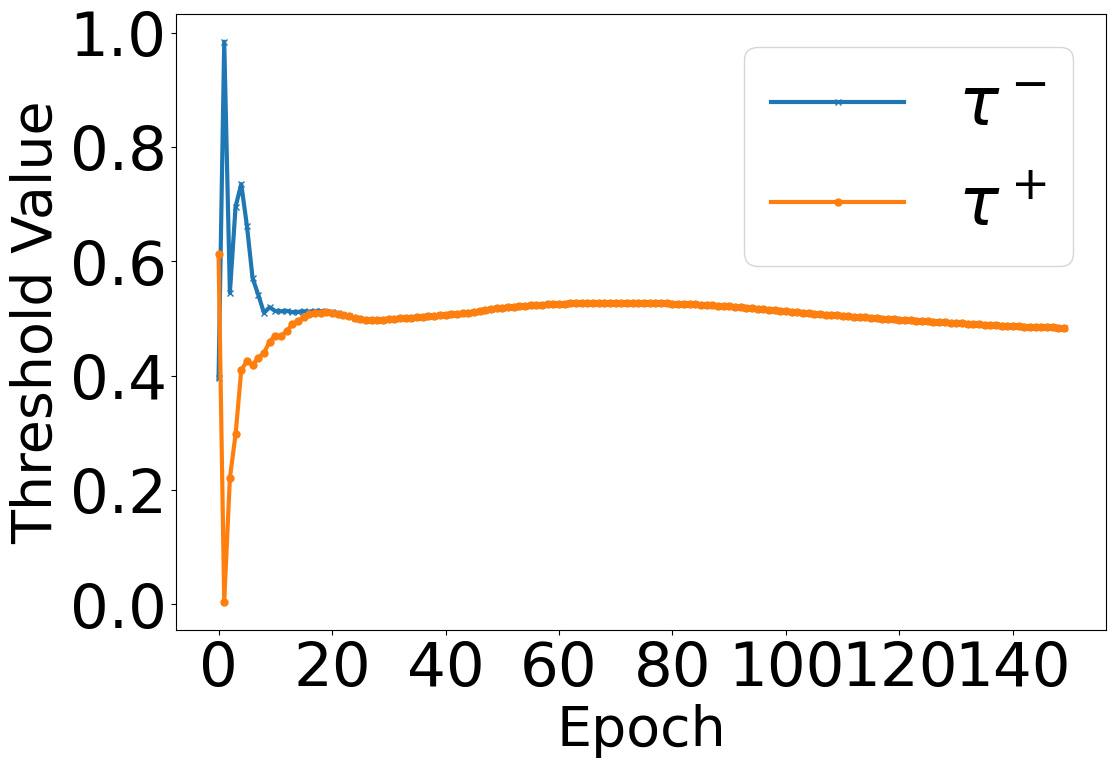}
    \end{minipage}
    \begin{minipage}[t]{0.24\linewidth}
    \centering
    \textbf{(c)} SVM+CMRM
    \includegraphics[width=\linewidth]{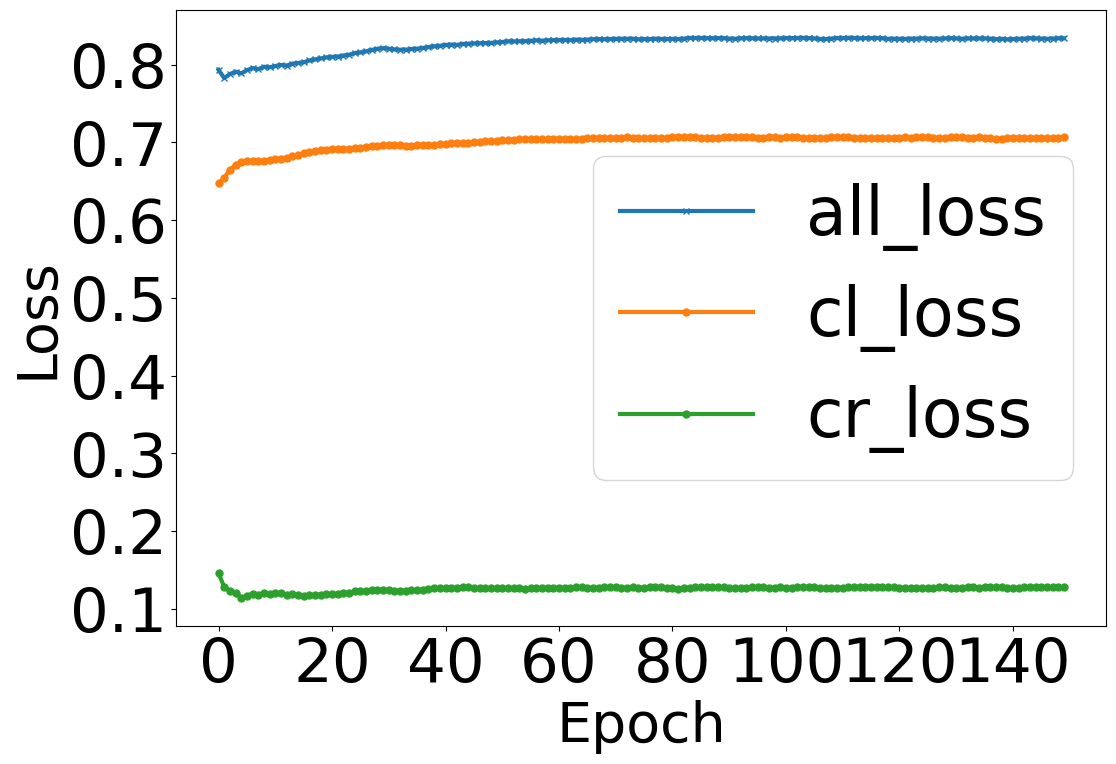}
    \\
    \includegraphics[width=\linewidth]{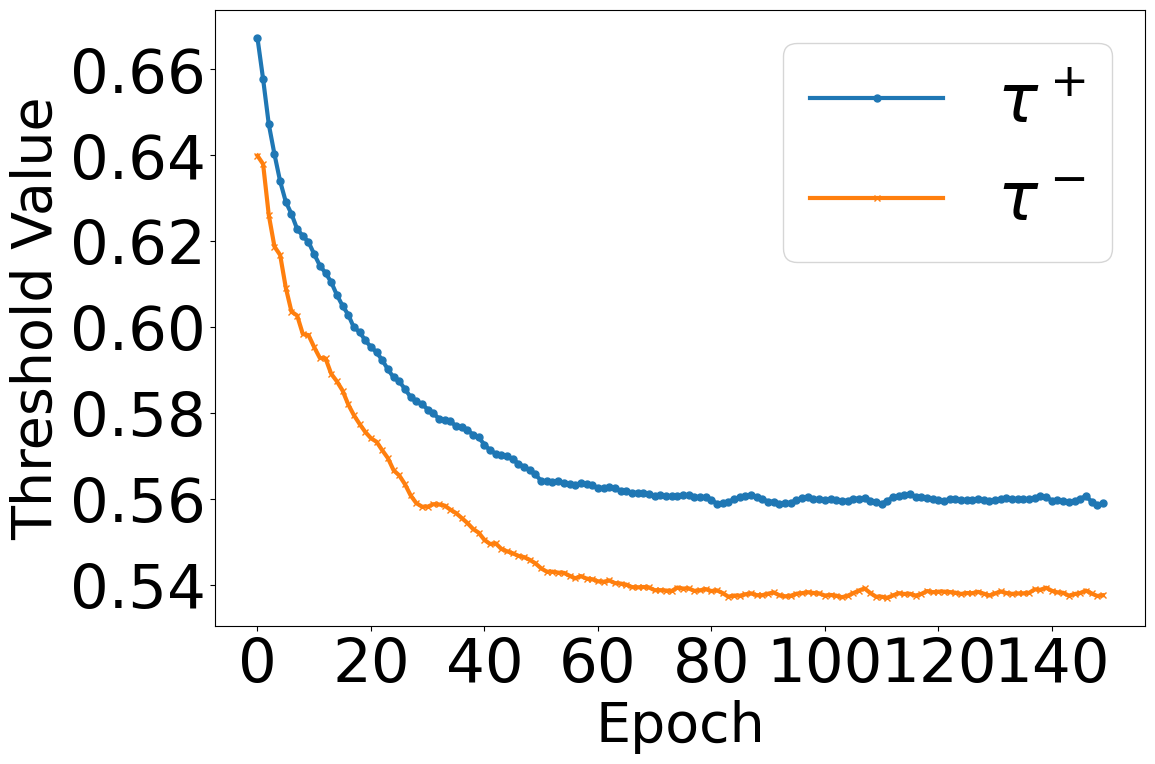}
    \end{minipage}
    \begin{minipage}[t]{0.24\linewidth}
    \centering
    \textbf{(d)} GCE+CMRM
    \includegraphics[width=\linewidth]{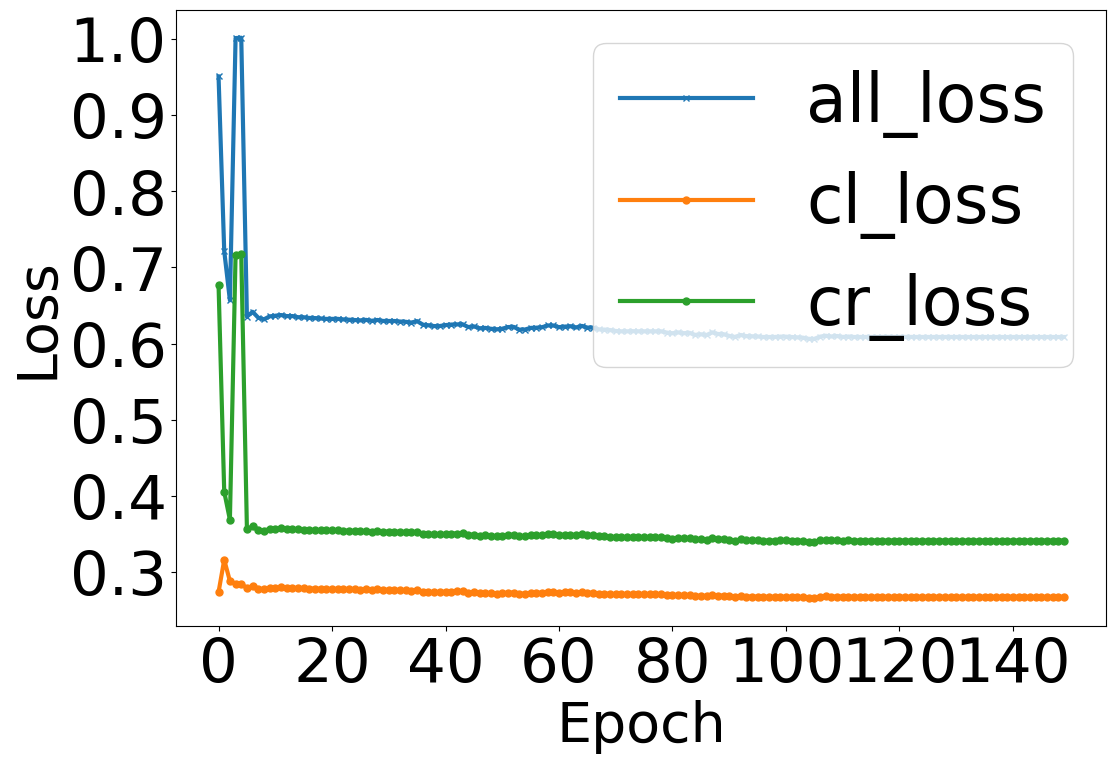}
    \\
    \includegraphics[width=\linewidth]{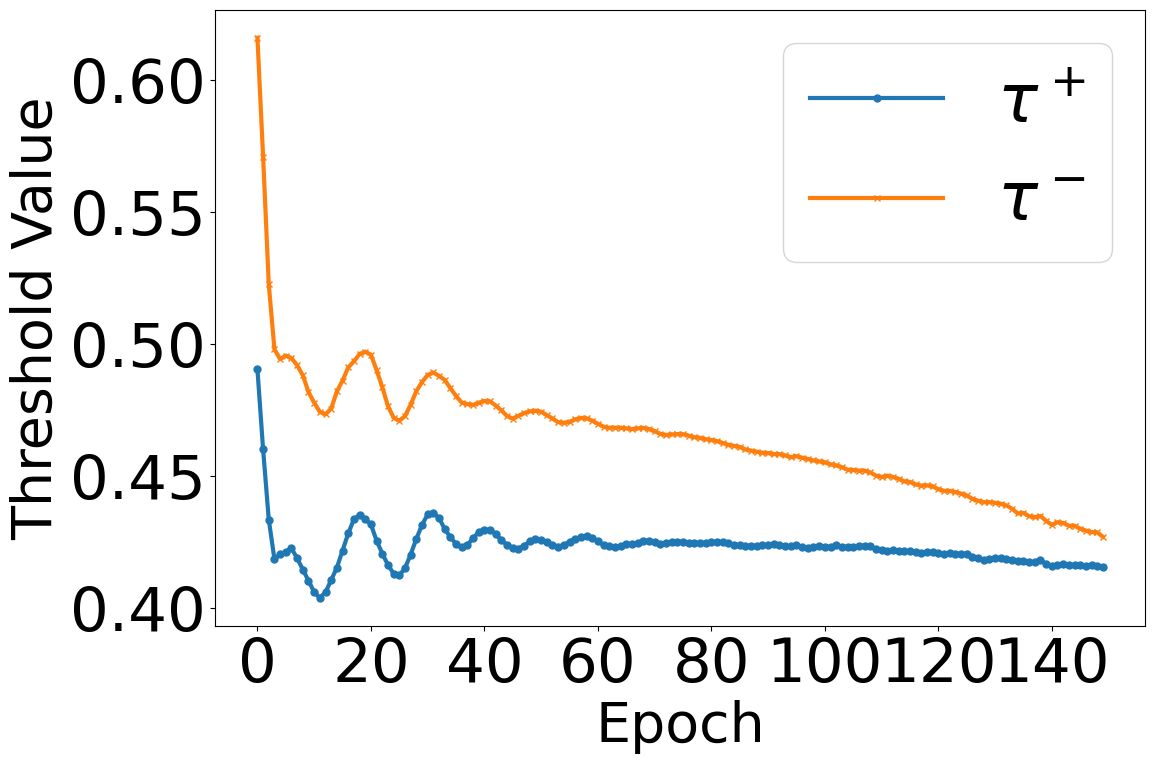}
    \end{minipage}
    \caption{
    \textbf{
    Dynamics of losses (Top Row) and two thresholds ($\tau^+$ and $\tau^-$, Bottom Row) for binary classification of all base loss (LR, Focal, SVM, GCE) with CMRM
    } on Email datasets.
    \textbf{Top row} shows the training dynamics of total loss (all loss), classification loss (cl loss), and CMRM regularization loss (CMRM loss) over epochs. 
    CMRM exhibits stable and monotonic convergence alongside standard loss components.
    \textbf{Bottom row} shows $\tau^+$ (negative class threshold) and $\tau^-$ (positive class threshold) of LR+CMRM during training.
    The separation between the thresholds increases, indicating that CMRM actively maximizes the margin between favorable and unfavorable classes.
    }
    \label{fig:binary_justification_full_email}
\end{figure}

\textbf{Result: CMRM improves robustness for binary classification.}
Table~\ref{tab:binary_appendix} reports results on the Adult, Email, and Credit datasets under a $20\%$ label noise setting.
Across all base models (LR, Focal, SVM, GCE), CMRM consistently improves ranking (AUROC, AUPRC) and classification (ACC, FPR) performance, while maintaining comparable FNR.
It also reduces predictive uncertainty, as shown by lower M.APSS, PC APSS, and NC APSS values.
These results demonstrate that CMRM enhances model robustness and uncertainty estimation under noisy supervision without requiring any noise priors.

\textbf{CMRM Loss convergences.}
To verify the stability and trainability of our method, we monitor the learning dynamics of CMRM during optimization. As shown in the top row of Figure \ref{fig:binary_justification_full_email}, both the classification loss and CMRM loss steadily decrease and stabilize, indicating that the joint objective converges smoothly. The CMRM regularization term integrates seamlessly with standard classification training and does not introduce optimization instability or slow down convergence. This confirms that CMRM can be efficiently optimized using standard gradient-based methods and is compatible with commonly used loss functions.

\textbf{Margin $\tau^- - \tau^+$ of CMRM grows over Time.}
The bottom row of Figure \ref{fig:binary_justification_full_email} illustrates the evolution of the class-conditional thresholds $\tau^+$ and $\tau^-$ over training. These thresholds define the CMRM margin region, with $\tau^-$ indicating the lower bound for confident positives and $\tau^+$ the upper bound for confident negatives. As training progresses, we observe that $\tau^-$  increases while $\tau^+$ decreases, leading to an expanding margin between the two thresholds. This demonstrates that CMRM successfully enforces separation between favorable and unfavorable classes in high-confidence regions, which is critical for robustness under posterior shift.

\end{document}